\documentclass[preprint,review,authoryear,3p]{elsarticle}
\usepackage{amssymb}
\usepackage{natbib}
\usepackage{longtable}
\usepackage{caption2}
\usepackage{graphicx}
\usepackage{url}
\usepackage[colorlinks,citecolor=blue,linkcolor=blue,urlcolor=black,hyperindex]{hyperref}
\usepackage{booktabs}
\usepackage{enumerate}
\usepackage{setspace}

\usepackage{bm}
\usepackage{threeparttable}
\usepackage{float}

\usepackage{geometry}
\usepackage{multirow}
\usepackage[OT2,T1]{fontenc}

\usepackage{amsmath}
\usepackage{amsthm}
\renewcommand{\thefigure}{\arabic{figure}}
\usepackage[switch]{lineno}

\newtheorem{proposition}{Proposition}
\def\ack{\section*{Acknowledgements}%
  \addtocontents{toc}{\protect\vspace{6pt}}%
  \addcontentsline{toc}{section}{Acknowledgements}%
}

\theoremstyle{definition}
\newtheorem{definition}{Definition}

\newtheorem{algorithm}{Algorithm}
\renewcommand{\thealgorithm}{\Roman{algorithm}}

\newtheorem{remark}{Remark}

\allowdisplaybreaks

\newcommand{\bea}{\begin{eqnarray}}
\newcommand{\eea}{\end{eqnarray}}
\newcommand{\ben}{\begin{equation}}
\newcommand{\een}{\end{equation}}

\renewcommand{\figurename}{Fig.}
\newcommand{\RNum}[1]{\uppercase\expandafter{\romannumeral #1\relax}}

\begin{document}
\begin{frontmatter}
\title{{\color{black}Lexicographic optimization-based approaches to learning a representative model for multi-criteria sorting with non-monotonic criteria}}

\author[dut1]{Zhen Zhang}
\ead{zhen.zhang@dlut.edu.cn}

\author[dut1]{Zhuolin Li\corref{cor1}}
\ead{lizhuolin@mail.dlut.edu.cn}
\cortext[cor1]{Corresponding author.}

\author[dufe]{Wenyu Yu\corref{cor2}}
\ead{dr.yuwy@gmail.com}
\cortext[cor2]{Corresponding author.}

\address[dut1]{School of Economics and Management, Dalian University of Technology, Dalian 116024, P. R. China}
\address[dufe]{School of Data Science and Artificial Intelligence, Dongbei University of Finance and Economics, Dalian 116025, P. R. China}

\begin{abstract}
Deriving a representative model using value function-based methods from the perspective of preference disaggregation has emerged as a prominent and growing topic in multi-criteria sorting (MCS) problems. A noteworthy observation is that many existing approaches to learning a representative model for MCS problems traditionally assume the monotonicity of criteria, which may not always align with the complexities found in real-world MCS scenarios. Consequently,  this paper proposes some approaches to learning a representative model for MCS problems with non-monotonic criteria through the integration of the threshold-based value-driven sorting procedure. To do so, we first define some transformation functions to map the marginal values and category thresholds into a UTA-like functional space.  Subsequently, we construct constraint sets to model non-monotonic criteria in MCS problems and develop optimization models to check and rectify the inconsistency of the decision maker's assignment example preference information. By simultaneously considering the complexity and discriminative power of the models, two distinct lexicographic optimization-based approaches are developed to derive a representative model for MCS problems with non-monotonic criteria. Eventually, we offer an illustrative example and conduct comprehensive simulation experiments to elaborate the feasibility and validity of the proposed approaches.
\end{abstract}
\end{frontmatter}

\section{Introduction}\label{sec:1}
Nowadays, multi-criteria sorting (MCS) problems, which involve the assignment of a set of alternatives to several predefined ordered categories in terms of multiple criteria, have garnered growing interest and find widespread application across diverse fields. These fields include inventory management \citep{Chen08caor,Liu16omega}, policy assessments \citep{Dias18omega}, consumer preference analysis \citep{Guo20omega}, ecological risk assessments \citep{Qin21ins}, supplier evaluation \citep{Pelissari22eswa} and green building rating \citep{Zhang22anor}. Within the literature, various MCS approaches have been proposed, encompassing value function-based approaches \citep{Doumpo01ds,Greco10ejor,Kadzinski15ejor,Kadzinski17caor,Liu20ejor}, outranking-based approaches \citep{Almeida10ejor,Kadzinski16ins,Olteanu22caor,Kadzinski21asoc}, distance-based approaches \citep{Chen07smca,de20caie}, and rule-based approaches \citep{Dembczynski09ejor,Kadzinski16ins1}.

Among the various mainstream MCS approaches, value function-based approaches stand out for their computational simplicity and intuitive interpretability \citep{Liu19ejor,Li23tcss}. These approaches rely on preference information provided by decision makers. Typically, decision makers have two ways to furnish their preferences: the direct and indirect methods \citep{Tomczyk19caor,Belahcene234or,Cinelli22ejor,Belahcene23caor}. The direct method entails decision makers providing specific parameter information required for value function-based MCS approaches, such as the shape of marginal value functions. In contrast, the indirect method involves decision makers offering holistic judgment information, such as assignment example preference information. Related parameters are then inferred using preference disaggregation techniques \citep{Jacquet01ejor,Liu20joc}. The indirect method is gaining popularity among value function-based MCS approaches due to its reduced cognitive demands on decision makers \citep{Corrente13ml}. Notably, UTilit\'{e}sAdditivesDIScriminantes (UTADIS) and its variants have become classic and prevalent approaches in this category \citep{Zopounidis01ejor,Doumpo02springer,Sobrie18ejor}. In UTADIS and its various variants, the holistic judgement information provided by a decision maker is typically translated into constraints, forming an extensive feasible domain and thus yielding multiple compatible parameters \citep{Ru23ejor,Arcidiacono23omega}. Consequently, the selection of a representative model, defined by a set of representative parameters, has emerged as a prominent research area for MCS problems. For instance, \cite{Greco11caor} introduced a robust ordinal regression method to select a representative value function for MCS problems. \cite{Kadzinski13dss} presented a stochastic ordinal regression method to select a representative value function and yield robust sorting results in MCS problems. \cite{Doumpos14ejor} developed an optimization model to select a representative sorting model by maximizing the minimal difference between the global values of reference alternatives and the corresponding category thresholds.

Many of the aforementioned approaches for learning a representative sorting model operate under the assumption of criteria monotonicity. However, in real-life scenarios, numerous MCS problems involve non-monotonic criteria \citep{Rezaei18eswa,Bagherzadeh22ejor}. For instance, in the financial domain, when classifying firms based on their financial status, the total liabilities to total assets is often considered as a non-monotonic criterion. Specifically, as the total liabilities to total assets decreases, the firm's ability to utilize external funds worsens, and as it increases, the financial risk of the firm rises. Consequently, the total liabilities to total assets must be maintained within a specific range. Similarly, in medical diagnosis, blood glucose levels also exhibit non-monotonic property; if the blood glucose level falls outside the normal range (either too high or too low), a patient is diagnosed with hyperglycemia or hypoglycemia. Consequently, modeling blood glucose levels necessitates the use of non-monotonic value functions \citep{Belahcene234or1}.

To address non-monotonic criteria for multi-criteria decision making problems, various approaches have been developed. For instance, \cite{Despotis95ama} introduced a linear programming method for learning non-monotonic value functions with a quadratic shape, requiring the decision maker to specify the evaluation value at which monotonicity changes in advance. Based on the UTASTAR algorithm, \cite{Kliegr09pl} proposed a new method to allow any shape of value functions to accommodate non-monotonic criteria by introducing numerous binary variables. \cite{Eckhardt12pl} utilized a local preference transformation method to handle non-monotonic criteria in the UTilit\'{e} Additive (UTA) method. \cite{Doumpos12orsp}  devised an evolutionary optimization-based approach for constructing non-monotonic value function models, considering different types of value functions with non-monotonic criteria.
\cite{Ghaderi15prl} proposed a new preference disaggregation approach that allows non-monotonic additive models, and applied it to analyze the impact of brand colour on brand image.
\cite{Ghaderi17ejor} constructed a linear fractional programming model to infer non-monotonic additive value functions based on indirect pairwise comparisons. \cite{Corrente23eswa} proposed a robust TOPSIS method for multi-criteria decision making problems with hierarchical and non-monotonic criteria. Although significant progress has been made in addressing non-monotonic criteria, most of these approaches are designed for multi-criteria ranking problems and are not directly applicable to MCS problems. Recently, \cite{Guo19eswa} introduced a progressive method for selecting representative marginal value functions for MCS problems with non-monotonic criteria. \cite{Kadzinski20ijar} developed a mixed-integer linear programming model to derive a representative instance of the sorting model by considering different types of monotonic and non-monotonic marginal value functions. Furthermore, \cite{Kadzinski21kbs} introduced a novel approach to address non-monotonic criteria by modeling the marginal value function for each criterion as the sum of non-decreasing and non-increasing components.
While significant advancements have been achieved, two primary challenges persist. Firstly, some methods are constrained  to learning marginal utility functions with specific shapes, limiting their widespread applicability.  Secondly, even though some methods can handle marginal utility functions of any shape, they frequently necessitate the introduction of binary variables, consequently increasing the computational complexity during model solving. Alternatively, they may not completely adhere to specific normalization constraints, thereby compromising the interpretability of the models.

To address the previously mentioned challenges, this paper introduces lexicographic optimization-based approaches to learning a representative model for MCS problems with non-monotonic criteria. The contributions of this paper are outlined as follows:

(1) We define some transformation functions to map the marginal values and category thresholds in the original functional space into UTA-like functional space. These functions ensure that the resulting marginal value functions and category thresholds are in UTA-like standard forms, preserving the sorting result for alternatives after the transformation.

(2) We construct constraint sets tailored to model the non-monotonic criteria in MCS problems. On this basis, we proceed to develop models for checking and addressing inconsistency in the decision maker's assignment example preference information. Subsequently, we introduce two lexicographic optimization-based approaches to learning a representative sorting model, taking into account both model complexity and discriminative power. Additionally, we analyze the robustness of the proposed approaches.

(3) We present an example to illustrate the implementation process of the proposed approaches. Additionally, we conduct extensive simulation experiments to compare the proposed approaches with some non-monotonic criteria modeling methods. Furthermore, we analyze how the performance of these approaches is influenced by factors such as the number of alternatives, criteria, categories, and the proportion of reference alternatives among all alternatives.

The subsequent sections of this paper are organized in the following way. Section \ref{sec:2} provides a review of the threshold-based value-driven sorting procedure. In Section \ref{sec:3}, we formulate the MCS problem with non-monotonic criteria. Following this, Section \ref{sec:4} offers a detailed exposition of the proposed approaches. In Section \ref{sec:5}, we present an illustrative example and some discussions. Subsequently, we provide extensive experiment analysis to validate the effectiveness of the proposed approach in Section \ref{sec:6}. Finally, Section \ref{sec:7} concludes this paper.

\section{Threshold-based value-driven sorting procedure}\label{sec:2}
In this section, we review the threshold-based value-driven sorting procedure \citep{Greco10ejor}, which serves as the foundation for the subsequent research discussed in this paper.

Let $A=\{a_1, a_2, \ldots, a_n\}$ represent a set of alternatives that need to be assigned to several predefined ordered categories $C=\{C_1, \ldots, C_q\}$ based on a family of criteria $G=\{g_1, g_2, \ldots, g_m\}$. In this context, $a_i$ denotes the $i$-th alternative, $i \in N=\{1,2,\ldots,n\}$, $g_j$ refers to the $j$-th criterion, $j \in M=\{1,2,\ldots,m\}$, and $C_h$ represents the $h$-th category, such that $C_{h}\succ C_{h-1}$ for $h \in Q=\{1,\ldots,q\}$. The performance level of the alternative $a_i$ with respect to the criterion $g_j$ is denoted as $g_j(a_i)$, $i \in N$ and $j \in M$. By defining $x_{ij}=g_j(a_i)$, the decision matrix encompassing all alternatives can be represented as $X=(x_{ij})_{n\times m}$. Furthermore, we introduce a category threshold vector $b=(b_0,b_1,\ldots,b_q)^{\rm T}$ to delineate the categories, with $b_{h-1}$ and $b_{h}$ being the lower and upper thresholds of the category $C_h$, respectively, where $b_{h}>b_{h-1}$ for $h \in Q$. For convenience, we set $b_0=0$ and $b_q=1+\varepsilon$, where $\varepsilon$ is a very small positive number.

In order to represent the decision maker's preferences and obtain the global value of each alternative $a_i$, an additive value function $V(a_i)$ is usually used, i.e.,
\begin{equation}\label{eq:global_v}
V(a_i)=\sum\limits_{j=1}^mv_j(x_{ij}), i\in N,
\end{equation}
where $v_j(\cdot)$ is the marginal value function on the criterion $g_j$, $j\in M$.

Let $[\beta_j^-,\beta_j^+]$ be the performance range of the criteria $g_j$, where $\beta_j^-$ and $\beta_j^+$ are the minimum and maximum performance levels for the criterion $g_j$, respectively, $j\in M$. To ensure that $V(a_i)\in [0,1]$, the below normalization condition should be satisfied for an additive value function $v_j(\cdot)$:
\begin{equation}\label{eq:s_v}
\begin{aligned}
&v_j(\beta_j^-)=0,j\in M\\
&\sum\limits_{j=1}^mv_j(\beta_j^+)=1.
\end{aligned}
\end{equation}

In this context, we call an additive value function is in UTA-like standard form if Eqs. \eqref{eq:global_v}~-~\eqref{eq:s_v} hold. In the threshold-based value-driven sorting procedure, a piecewise linear marginal value function, denoted by $v_j(\cdot)$ is typically employed to represent the decision maker's preferences concerning the criterion $g_j$. Specifically, the performance range of each criterion $g_j$ is divided into $s_j$ equal subintervals, denoted as $[\beta_j^1,\beta_j^2]$,\ldots,$[\beta_j^l,\beta_j^{l+1}]$, \ldots, $[\beta_j^{s_j},\beta_j^{s_j+1}]$, where $\beta_j^1=\beta_j^-$, $\beta_j^{s_j+1}=\beta_j^+$, $\beta_j^l=\beta_j^1+\frac{l-1}{s_j}(\beta_j^{s_j+1}-\beta_j^1)$ and each $\beta_j^l$ is termed a breakpoint, $l=1,\ldots,s_j+1$.

For an alternative $a_i$, given an performance level $x_{ij}$ for the criterion $g_j$, its marginal value can be determined through linear interpolation, i.e.,
\begin{equation}\label{eq:marginal_v0}
v_j(x_{ij})=v_j(\beta^l_j)+\frac{x_{ij}-\beta^l_j}{\beta^{l+1}_j-\beta^l_j}(v_j(\beta^{l+1}_j)-v_j(\beta^l_j)), \text{if}\ x_{ij}\in[\beta^l_j,\beta^{l+1}_j].
\end{equation}

Following this, the sorting result (i.e., the assigned category) for each alternative $a_i$ can be determined by comparing its global value $V(a_i)$ with the category thresholds. To elaborate, if $b_{h-1}\le V(a_i)< b_h$, then the alternative $a_i$ is assigned to the category $C_h$, $i\in N$, $h\in Q$.

In the fundamental framework of the threshold-based value-driven sorting procedure, it is typically assumed that the criteria exhibit a monotonically increasing behavior. Nonetheless, practical MCS problems often involve non-monotonic criteria. To address this challenge, a novel approach is introduced in the subsequent sections.

\section{Problem formulation}\label{sec:3}

We consider the following MCS problem with non-monotonic criteria, which seeks to assign a set of alternatives $\{a_1,a_2,\ldots,a_n\}$ into predefined ordered categories $\{C_1,C_2\ldots,C_q\}$, based on a family of criteria $\{g_1,g_2,\ldots,g_m\}$. It is noteworthy that this MCS problem accommodates non-monotonic criteria, but no prior information is available regarding these non-monotonic criteria.

To facilitate the assignment of alternatives, the decision maker will offer some assignment example preference information, typically consisting of holistic judgments for reference alternatives and their corresponding assigned categories. Let $A^R \subset A$ represent the set of reference alternatives, and the set of assignment example preference information can be succinctly denoted as $\{a_{i}\rightarrow C_{B_{i}}|a_{i}\in A^R\}$. Additionally, we define $A^N = A\backslash A^R$ as the set of non-reference alternatives.

This paper is dedicated to introducing lexicographic optimization-based approaches to learning a representative model for MCS problems with non-monotonic criteria from the perspective of preference disaggregation and then obtain the sorting result for non-reference alternatives. The proposed approaches consist of the following parts:

(1) Modeling the non-monotonic criteria using transformation functions

In this part, we employ a transformation function designed to convert non-monotonic marginal value functions into a UTA-like standard form \citep{Ghaderi17ejor}. Accordingly, we also introduce another transformation function that maps the category thresholds within the non-monotonic functional space into a UTA-like functional space. These transformations serve a dual purpose: simplifying the subsequent modeling process and ensuring that the sorting result for alternatives remain unchanged after the transformation. Afterwards, utilizing the proposed transformation functions, we proceed to construct several constraint sets tailored to model non-monotonic criteria for MCS problems. These constraint sets serve as the foundation for the subsequent development of our models (see subsections \ref{sec:4.1} and \ref{sec:4.2}).

(2) Consistency check and preference adjustment

To check whether the assignment example preference information provided by the decision maker is consistent, we develop a consistency check model that takes into consideration the non-monotonic criteria.  In instances where the assignment example preference information is identified as inconsistent, we formulate a minimum adjustment optimization model to rectify this inconsistency (see subsection \ref{sec:4.3}).

(3) Learning a representative sorting model

Once the inconsistency of the decision maker's assignment example preference information is eliminated, we proceed with developing lexicographic optimization-based approaches to learning a representative model for MCS problems with non-monotonic criteria. By solving the models, we can determine the representative marginal value functions and category thresholds effectively (see subsection \ref{sec:4.4}).

(4) Obtaining the sorting result for non-reference alternatives

Using the representative sorting model, we are equipped to assign non-reference alternatives to predefined ordered categories based on the threshold-based value-driven sorting procedure.

\section{The proposed approaches}\label{sec:4}
In this section, the proposed approaches are detailed. We commence by introducing the transformation functions that will be used throughout the remainder of the paper. Subsequently, we provide the non-monotonic criteria modeling method for MCS problems. Following this, we develop some consistency check and preference adjustment models. Additionally, we introduce the lexicographic optimization-based approaches and the robustness analysis method based on the proposed approaches.

\subsection{Transformation functions}\label{sec:4.1}
Since the marginal value functions may be non-monotonic in MCS problems, the performance levels associated with the maximum and minimum marginal values taken for each criterion are not known beforehand. Consequently, the marginal value functions may not adhere to a UTA-like standard form. To address this challenge and  represent the marginal value functions in a UTA-like standard form for MCS problems with non-monotonic criteria, some transformation functions are introduced below. These functions are designed to guarantee that the sorting result for alternatives remain unchanged even after transformation.

Let $v_j^S(\cdot)$ and $V^S(\cdot)$ denote the marginal value functions for the criterion $g_j$ and the overall marginal value function in the UTA-like functional space, respectively, then we have
\begin{equation}\label{eq:s_v1}
\begin{aligned}
&V^S(\cdot)=\sum\limits_{j=1}^mv_j^S(\cdot) \\
&v_j^S(g_j^-)=0,j\in M\\
&\sum\limits_{j=1}^mv_j^S(g_j^+)=1.\\
\end{aligned}
\end{equation}
where $g_j^-$ and $g_j^+$ are the performance levels corresponding to the minimum and maximum marginal values of the criterion $g_j$, $j\in M$, respectively, and the superscript $S$ denotes that the functions are in UTA-like standard forms.

Moreover, let $b^S=(b_0^S,b_1^S,\ldots,b_q^S)^{\rm T}$ be a category threshold vector in the UTA-like functional space, then the following conditions should satisfy:
\begin{equation}\label{eq:s_b}
\begin{aligned}
&b_0^S = 0,  b_q^S = 1+ \varepsilon^S\\
&b_{h}^S - b_{h-1}^S \ge \varepsilon^S, h \in Q.
\end{aligned}
\end{equation}
where $\varepsilon^S$ is an arbitrarily small positive number.

Analogously to \cite{Ghaderi17ejor} wherein a mapping function is proposed to address non-monotonic marginal value functions in multi-criteria ranking problems, this paper defines some transformation functions that convert marginal value functions in the non-monotonic functional space into that in a UTA-like standard form for MCS problems.
\begin{definition}\label{def:1}
Let $v_j(\cdot)$ denote the marginal value function of the criterion $g_j$, $j\in M$, then the transformation function that converts the marginal value $v_j(x_{ij})$ into the UTA-like functional space is defined as
\begin{equation}\label{eq:fv}
f_v(x_{ij})=\frac{v_j(x_{ij})-v_j(g_j^-)}{\sum\limits_{j=1}^m(v_j(g_j^+)-v_j(g_j^-))}, i\in N,j\in M.
\end{equation}
\end{definition}

According to Definition \ref{def:1}, if $x_{ij}=g_j^-$, then we have $f_v(x_{ij})=0$; if $x_{ij}=g_j^+$, then we have $\sum\limits_{j=1}^mf_v(x_{ij})=\sum\limits_{j=1}^m \frac{v_j(g_j^+)-v_j(g_j^-)}{\sum\limits_{j=1}^m(v_j(g_j^+)-v_j(g_j^-))}=1$. Therefore, the transformation function $f_v(\cdot)$ ensures that marginal value functions after transformation adhere to the UTA-like standard form.

In a similar manner, we define the function that maps the category thresholds in the original functional space into the UTA-like functional space as follows.
\begin{definition}\label{def:2}
Let $b=(b_0,b_1,\ldots,b_q)^{\rm T}$ be the category threshold vector in the original functional space, then the function that maps a category threshold $b_h$ into the UTA-like functional space is defined as
\begin{equation}\label{eq:fb}
f_b(b_h)=\frac{b_h-\sum\nolimits_{j=1}^m v_j(g_j^-)}{\sum\limits_{j=1}^m(v_j(g_j^+)-v_j(g_j^-))}, h=0,1,\ldots,q.
\end{equation}
\end{definition}

According to Definition \ref{def:2}, it is evident that when $b_h=\sum\nolimits_{j=1}^m v_j(g_j^-)$, we obtain $f_b(b_h)=0$. Conversely, when $b_h=\sum\nolimits_{j=1}^m v_j(g_j^+)$, we have $f_b(b_h)=1$. Furthermore, in accordance with the threshold-based value-driven sorting procedure, we have $b_h-b_{h-1}\ge \varepsilon$, where $\varepsilon$ is an arbitrarily small positive number. As a result, it follows that $f_b(b_h)-f_b(b_{h-1})=\frac{b_h-b_{h-1}}{\sum\limits_{j=1}^m(v_j(g_j^+)-v_j(g_j^-))}\ge \frac{\varepsilon}{\sum\limits_{j=1}^m(v_j(g_j^+)-v_j(g_j^-))}$. Let $\varepsilon^S=\frac{\varepsilon}{\sum\limits_{j=1}^m(v_j(g_j^+)-v_j(g_j^-))}$, since $f_b(b_h)-f_b(b_{h-1})\ge \varepsilon^S$, we  can conclude that $b_h-b_{h-1}\ge \varepsilon$. To summarize, it can be concluded that the transformation function $f_b(\cdot)$ ensures that the transformed category thresholds satisfy the basic conditions outlined in Eq. \eqref{eq:s_b}.

\begin{proposition}\label{proposition:1}
By employing the transformation functions $f_v(\cdot)$ and $f_b(\cdot)$, the sorting result for alternatives remain unchanged after transformation.
\end{proposition}

\begin{proof}
Proposition \ref{proposition:1} means that if the alternative $a_i$ is assigned to the category $C_h$, i.e., $b_{h-1}\le V(a_i) < b_{h}$, then we have $b_{h-1}^S\le V^S(a_i) < b_h^S$.
We first prove that $b_{h-1}\le V(a_i)$ is equivalent to $b_{h-1}^S\le V^S(a_i)$. In terms of Definitions \ref{def:1}~-~\ref{def:2},  it is apparent that $V^S(a_i)-b_{h-1}^S=\frac{\sum\nolimits_{j=1}^m v_j(x_{ij})-\sum\nolimits_{j=1}^m v_j(g_j^-)}{\sum\nolimits_{j=1}^m(v_j(g_j^+)-v_j(g_j^-))}- \frac{b_{h-1}-\sum\nolimits_{j=1}^m v_j(g_j^-)}{\sum\nolimits_{j=1}^m(v_j(g_j^+)-v_j(g_j^-))}=\frac{V(a_i)-b_{h-1}}{\sum\nolimits_{j=1}^m(v_j(g_j^+)-v_j(g_j^-))}$. Consequently,  we obtain $V(a_i)-b_{h-1}\ge 0 \iff V^S(a_i)-b_{h-1}^S\ge 0$. Analogously, it follows that $V(a_i) - b_{h}<0 \iff V^S(a_i)-b_h^S<0$. Therefore, we have that $b_{h-1}\le V(a_i) < b_{h} \iff b_{h-1}^S\le V^S(a_i) < b_h^S$.
This completes the proof of Proposition \ref{proposition:1}.
\end{proof}

The transformation functions introduced in this subsection enable us to convert marginal values and category thresholds from the original functional space into the UTA-like functional space. This transformation simplifies the modeling process for MCS problems with non-monotonic criteria.

\subsection{Modeling the non-monotonic criteria in MCS problems}\label{sec:4.2}

In this subsection, we model the non-monotonic criteria in MCS problems by constructing some constraint sets based on the transformation functions proposed in Section \ref{sec:4.1}.

In this paper, it is assumed that the decision maker provides assignment examples as his/her preference information, and the threshold-based value-driven procedure is employed for assigning alternatives. For a reference alternative $a_{i}\in A^R$ with the category assignment $C_{B_{i}}$, the following constraint set can be utilized to model the decision maker's assignment example preference information:
\begin{equation}\label{eq:base0}
E^{A^{R^0}}\begin{cases}
\begin{aligned}
&V^S(a_{i})\ge b_{B_{i}-1}^S,\forall a_{i}\in A^R\\
&V^S(a_{i})\le b_{B_{i}}^S - \varepsilon^S,\forall a_{i}\in A^R.\\
\end{aligned}
\end{cases}
\end{equation}

Furthermore, some fundamental conditions of the threshold-based value-driven sorting procedure should be satisfied, i.e.,
\begin{small}
\begin{equation}
E^{Sort^0}\begin{cases}
\begin{aligned}
&V^S(a_{i})=\sum\limits_{j=1}^mv_j^S(x_{{i}j}),\forall a_{i}\in A^R \\
&v_j^S(g_j^-)=0,j\in M\\
&\sum\limits_{j=1}^mv_j^S(g_j^+)=1\\
&v_j^S(x_{ij})=v_j^S(\beta^l_j)+\frac{x_{ij}-\beta^l_j}{\beta^{l+1}_j-\beta^l_j}(v_j^S(\beta^{l+1}_j)-v_j^S(\beta^l_j)),\\ &\hspace{5em}\text{if}\ x_{ij}\in[\beta^l_j,\beta^{l+1}_j], \forall a_{i}\in A^R, j\in M \\
&b_0^S = 0,  b_q^S = 1+ \varepsilon^S\\
&b_{h}^S - b_{h-1}^S \ge \varepsilon^S, h \in Q.
\end{aligned}
\end{cases}
\end{equation}
\end{small}

\begin{proposition}\label{proposition:2}
In terms of the transformation functions $f_v(\cdot)$ and $f_b(\cdot)$ proposed in Section \ref{sec:4.1}, the constraint set $E^{A^{R^0}}$ and $E^{Sort^0}$ can be equivalently converted into the following constraint sets $E^{A^{R^1}}$ and $E^{Sort^1}$, respectively, i.e.,
\begin{equation}
E^{A^{R^1}}\begin{cases}
\begin{aligned}
&V(a_{i})\ge b_{B_{i}-1},\forall a_{i}\in A^R\\
&V(a_{i})\le b_{B_{i}} - \varepsilon,\forall a_{i}\in A^R.\\
\end{aligned}
\end{cases}
\end{equation}

\begin{equation}
E^{Sort^1}\begin{cases}
\begin{aligned}
&V(a_{i})=\sum\limits_{j=1}^mv_j(x_{ij}),\forall a_{i}\in A^R \\
&v_j(x_{ij})=v_j(\beta^l_j)+\frac{x_{ij}-\beta^l_j}{\beta^{l+1}_j-\beta^l_j}(v_j(\beta^{l+1}_j)-v_j(\beta^l_j)),\\ &\hspace{5em}\text{if}\ x_{ij}\in[\beta^l_j,\beta^{l+1}_j], \forall a_{i}\in A^R, j\in M \\
&b_0 = \sum\nolimits_{j=1}^m v_j(g_j^-),  b_q = \sum\nolimits_{j=1}^m v_j(g_j^+)+ \varepsilon\\
&b_{h} - b_{h-1} \ge \varepsilon, h \in Q.
\end{aligned}
\end{cases}
\end{equation}

\end{proposition}
\begin{proof}
As per Proposition \ref{proposition:1}, we have that $b_{h-1}\le V(a_i) < b_{h} \iff b_{h-1}^S\le V^S(a_i) < b_h^S$. Consequently, the constraint set $E^{A^{R^0}}$ can be converted into $E^{A^{R^1}}$. Furthermore, By substituting $v^S(\cdot)$ with $f_v(\cdot)$, the fourth constraint in $E^{Sort^0}$ can be equivalently replaced by the second constraint in $E^{Sort^1}$. This also leads to $V(a_{i})=\sum\nolimits_{j=1}^mv_j(x_{ij})$. Additionally, based on the transformation function $f_b(\cdot)$, the last two constraints in $E^{Sort^0}$ can be transformed into the last two constraints in $E^{Sort^1}$.
This completes the proof of Proposition \ref{proposition:2}.
\end{proof}

The constraint set $E^{Sort^1}$ includes some variables $v_j(g_j^-)$, $v_j(g_j^+)$ that are utilized to denote $b_0$ and $b_q$. However, accounting for potential non-monotonic criteria makes it challenging to directly represent $b_0$ and $b_q$ without introducing additional binary variables. This could increase the model's complexity. As a result, when formulating the representative sorting model, we omit the constraints $b_0 = \sum\nolimits_{j=1}^m v_j(g_j^-),  b_q = \sum\nolimits_{j=1}^m v_j(g_j^+)+ \varepsilon$ from $E^{Sort^1}$. It is important to note that such omissions do not impact the sorting result since $b_0$ and $b_q$ are derived after obtaining the values of $v_j(\beta_j^l)$. Consequently, we can express the constraint sets $E^{A^{R^1}}$ and $E^{Sort^1}$ as follows:
\begin{equation}
E^{A^{R}}\begin{cases}
\begin{aligned}
& b_{B_{i}-1}\le V(a_{i})\le  b_{B_{i}} - \varepsilon, \text{ if } B_{i}\in\{2,\ldots,q-1\}, \forall a_{i}\in A^R\\
& V(a_{i}) \le b_1 - \varepsilon, \text{ if } B_{i}=1, \forall a_{i}\in A^R\\
& V(a_{i}) \ge b_{q-1}, \text{ if } B_{i}=q, \forall a_{i}\in A^R.\\
\end{aligned}
\end{cases}
\end{equation}

\begin{equation}
E^{Sort}\begin{cases}
\begin{aligned}
&V(a_{i})=\sum\limits_{j=1}^mv_j(x_{ij}),\forall a_{i}\in A^R \\
&v_j(x_{ij})=v_j(\beta^l_j)+\frac{x_{ij}-\beta^l_j}{\beta^{l+1}_j-\beta^l_j}(v_j(\beta^{l+1}_j)-v_j(\beta^l_j)),\\ &\hspace{5em}\text{if}\ x_{ij}\in[\beta^l_j,\beta^{l+1}_j], \forall a_{i}\in A^R, j\in M \\
&b_{h} - b_{h-1} \ge \varepsilon, h=2,\ldots,q-1.
\end{aligned}
\end{cases}
\end{equation}

Furthermore, when modeling the non-monotonic criteria in MCS problems, it is typically desirable to bound the solution space of the marginal value functions \citep{Ghaderi17ejor}. To achieve this, we utilize the following constraints to restrict the marginal values of breakpoints for each criterion, i.e.,
\begin{equation}\label{eq:bound}
E^{Bound}\left\{
\begin{aligned}
& v_j(\beta_j^l)\ge 0, l=1,\ldots,s_j+1, j\in M\\
& v_j(\beta_j^l)\le 1, l=1,\ldots,s_j+1, j\in M.
\end{aligned}
\right.
\end{equation}

\begin{remark}
The constraint set $E^{Bound}$ guarantees that the value of $v_j(\beta_j^l)$ are non-negative and does not exceed 1. In other words, it enforces that 0 is the lower bound for $v_j(\beta_j^l)$, and 1 is the upper bound.  It's worth noting that the selection of these lower and upper bounds is somewhat arbitrary, and different values can be chosen to constrain $v_j(\beta_j^l)$ if needed.
\end{remark}

In the rest of this paper, we will employ the constraint sets $E^{A^{R}}$, $E^{Sort}$, and $E^{Bound}$ to model the decision maker's assignment example preference information for MCS problems with non-monotonic criteria.

\subsection{Consistency check and preference adjustment}\label{sec:4.3}

\subsubsection{Consistency check}\label{sec:4.3.1}
In this paper, our preference model posits the use of an additive value function to reconstruct the decision maker's assignment example preference information. However, in certain instances, the assumed preference model may fall short of fully restoring the decision maker's assignment example preference information. This limitation primarily arises due to the preference model's restrictive assumptions, such as preference independence and additivity \citep{Ghaderi21omega}. To identify a sorting model that is compatible with the assumed preference model, it becomes crucial to check the consistency of the assignment example preference information provided by the decision maker.

To this end, inspired by the UTADIS method \citep{Devaud80UTADIS}, the following optimization model is developed to check whether the assignment example preference information is consistent:
\begin{equation}\tag{M-1}\label{m:con_check}
\begin{aligned}
&\min \ \sum\limits_{a_{i}\in A^R} \delta_{i}^+ + \delta_{i}^- \\
&\begin{aligned}
\rm{s.t.} \ & b_{B_{i}-1} - \delta_{i}^+ \le V(a_{i})\le  b_{B_{i}} + \delta_{i}^- - \varepsilon, \text{ if } B_{i}\in\{2,\ldots,q-1\}, \forall a_{i}\in A^R\\
& V(a_{i})\le  b_1 + \delta_{i}^- - \varepsilon, \text{ if } B_{i}=1, \forall a_{i}\in A^R\\
& V(a_{i})\ge b_{q-1} - \delta_{i}^+, \text{ if } B_{i}=q, \forall a_{i}\in A^R\\
&\delta_{i}^+\ge 0, \delta_{i}^- \ge 0, \forall a_{i}\in A^R\\
&E^{Bound},\ E^{Sort}.
\end{aligned}
\end{aligned}
\end{equation}

In the model \eqref{m:con_check}, auxiliary variables $\delta_{i}^+$ and $\delta_{i}^-$, $\forall a_{i}\in A^R$ are introduced to judge the consistency of the assignment example preference information. If the optimal objective function value of the model \eqref{m:con_check} is 0, the assignment example preference information is deemed consistent. Conversely, if the value is non-zero, it indicates inconsistency, requiring adjustment through the minimum adjustment optimization model proposed in Section \ref{sec:4.3.2}.

\subsubsection{Preference adjustment}\label{sec:4.3.2}
To assist the decision maker in adjusting inconsistent assignment example preference information,  a minimum adjustment optimization model is proposed. For convenience, let the adjusted assignment example preference information be denoted as $\{a_{i}\rightarrow C_{\overline{B}_{i}}|a_{i}\in A^R\}$.

To preserve the initial assignment example preference information provided by the decision maker to the greatest extent,
the objective of the optimization model is to minimize the sum of steps in moving between categories for all reference alternatives within the set of assignment example preference information, i.e.,
\begin{equation}\label{eq:min}
\min \sum\limits_{a_{i}\in A^R}|\overline{B}_{i}-B_{i}|.
\end{equation}

Moreover, to ensure the consistency of the adjusted assignment example preference information, the following constraints must be satisfied:
\begin{equation}\label{eq:EAR}
E^{A^{R'}}\left\{
\begin{aligned}
& b_{\overline{B}_{i}-1} \le V(a_{i})\le b_{\overline{B}_{i}} - \varepsilon, \text{ if } \overline{B_{i}}\in\{2,\ldots,q-1\}, \forall a_{i}\in A^R\\
& V(a_{i}) \le  b_1 - \varepsilon, \text{ if } \overline{B_{i}}=1, \forall a_{i}\in A^R\\
& V(a_{i}) \ge b_{q-1}, \text{ if } \overline{B_{i}}=q, \forall a_{i}\in A^R.\\
\end{aligned}\right.
\end{equation}

\begin{proposition}\label{proposition:3}
Let $t_{h{i}}$ be a binary variable, $\forall h\in Q$, $ a_{i}\in A^R$, then $E^{A^{R'}}$ can be equivalently transformed into the following constraint set:
\begin{equation}\label{eq:T_EAR}
E^{\overline{A^{R}}}\left\{
\begin{aligned}
&V(a_{i})\ge b_{h-1} + M(t_{h{i}} - 1), h=2,\ldots,q, \forall a_{i}\in A^R\\
&V(a_{i})\le  b_{h} - \varepsilon + M(1 - t_{h{i}}), h=1,\ldots,q-1, \forall a_{i}\in A^R\\
&\sum\limits_{h=1}^q t_{h{i}} = 1, \forall a_{i}\in A^R\\
&\overline{B}_{i} = \sum\limits_{h=1}^q h\cdot t_{h{i}}, \forall a_{i}\in A^R.
\end{aligned}\right.
\end{equation}
\end{proposition}
\begin{proof}
The first two constraints of $E^{\overline{A^{R}}}$ ensure that if $a_{i}\rightarrow C_{B_{i}}$, then $t_{h{i}}=1$; otherwise, $t_{hi}=0$, and the third constraint of $E^{\overline{A^{R}}}$ guarantees that each reference alternative can only be assigned to one category. Furthermore, the last constraint is used to calculate the adjusted category to which the reference alternative is assigned. This completes the proof of Proposition \ref{proposition:3}.
\end{proof}

Taking into account the considerations mentioned above, we formulate the minimum adjustment optimization model as:
\begin{equation}\tag{M-2}\label{m:mini}
\begin{aligned}
&\min \sum\limits_{a_{i}\in A^R}|\overline{B}_{i}-B_{i}|\\
&\begin{aligned}
\rm{s.t.} \ &E^{\overline{A^R}}, \ E^{Bound},\ E^{Sort}.
\end{aligned}
\end{aligned}
\end{equation}

By solving the model \eqref{m:mini}, the consistent assignment example preference information can be obtained. Let $\overline{B}_{i}^*$ be the optimal solution to the model \eqref{m:mini}, then the adjusted assignment example preference information is denoted by $\{a_{i}\rightarrow C_{\overline{B}_{i}^*}|a_{i}\in A^R\}$. Without loss of generality, we will continue to denote the assignment example preference information by $\{a_{i}\rightarrow C_{{B}_{i}}|a_{i}\in A^R\}$ for the remainder of this section.

\subsection{The lexicographic optimization-based approaches}\label{sec:4.4}
In this section, we develop some lexicographic optimization-based approaches to derive a representative sorting model for MCS problems with non-monotonic criteria.

When determining a representative sorting model, it becomes crucial to take into account both model complexity and discriminative power \citep{Ghaderi17ejor}. The model complexity can be captured by analyzing changes in the slope of the marginal value function. In particular, a substantial change in the slope indicates higher complexity, while a small change suggests lower complexity.
Additionally, the model's discriminative power can be examined by the value of the parameter $\varepsilon$, i.e., a larger value of the parameter $\varepsilon$ corresponds to better discriminative power of the model.

To do so, we introduce a variable $\gamma_{lj}$ for the $j$-th criterion's marginal value function to represent the change in slope between two consecutive subintervals, namely $[\beta_j^{l-1},\beta_j^l]$ and $[\beta_j^l,\beta_j^{l+1}]$, where
\begin{equation}
\gamma_{lj} = \left|\frac{v_j(\beta_j^l)-v_j(\beta_j^{l-1})}{\beta_j^l-\beta_j^{l-1}} - \frac{v_j(\beta_j^{l+1})-v_j(\beta_j^l)}{\beta_j^{l+1}-\beta_j^l}\right|, l=2,\ldots,s_j, j\in M.
\end{equation}

To obtain a representative sorting model that strikes a balance between low complexity and strong discriminative power, the objective is to minimize the sum of all $\gamma_{lj}$ while maximizing the parameter $\varepsilon$.  This results in the formulation of a multi-objective optimization model as follows:
\begin{equation}\tag{M-3}\label{m:GP}
\begin{aligned}
&\min\ \sum\limits_{j=1}^m\sum\limits_{l=2}^{s_j}\gamma_{lj}\\
&\max\  \varepsilon\\
&\begin{aligned}
\rm{s.t.} &\ E^{A^{R}}, \ E^{Bound},\ E^{Sort}\\
&\gamma_{lj} = \left|\frac{v_j(\beta_j^l)-v_j(\beta_j^{l-1})}{\beta_j^l-\beta_j^{l-1}} - \frac{v_j(\beta_j^{l+1})-v_j(\beta_j^l)}{\beta_j^{l+1}-\beta_j^l}\right|, l=2,\ldots,s_j, j\in M.
\end{aligned}
\end{aligned}
\end{equation}

To address the multi-objective optimization model \eqref{m:GP}, we employ lexicographic optimization, a technique that converts a multi-objective optimization model into several single-objective optimization models based on the priority assigned to each objective. To implement this, we introduce two approaches, namely Approach 1 and Approach 2, to derive a representative sorting model, taking into account different objective priorities.

\textbf{Approach 1}

In Approach 1, we prioritize model complexity over model discriminative power. Under this assumption, we first aim to optimize the objective $\min \sum\limits_{j=1}^m\sum\limits_{l=2}^{s_j}\gamma_{lj}$, and then we optimize the second objective $\max \varepsilon$ while ensuring that the first objective reaches its optimal value. Consequently, the optimization model for the first stage in Approach 1 is constructed as
\begin{equation}\tag{M-4}\label{m:method1_1_0}
\begin{aligned}
&\min \sum\limits_{j=1}^m\sum\limits_{l=2}^{s_j}\gamma_{lj}\\
&\begin{aligned}
\rm{s.t.} &\ E^{{A^R}}, \ E^{Bound},\ E^{Sort}\\
&\gamma_{lj} = \left|\frac{v_j(\beta_j^l)-v_j(\beta_j^{l-1})}{\beta_j^l-\beta_j^{l-1}} - \frac{v_j(\beta_j^{l+1})-v_j(\beta_j^l)}{\beta_j^{l+1}-\beta_j^l}\right|, l=2,\ldots,s_j, j\in M.
\end{aligned}
\end{aligned}
\end{equation}

\begin{proposition}
Let us define the following constraint set $E^{Slope}$,
\begin{equation}\label{eq:gamma}
E^{Slope} \left\{
\begin{aligned}
&\gamma_{lj} \ge \frac{v_j(\beta_j^l)-v_j(\beta_j^{l-1})}{\beta_j^l-\beta_j^{l-1}} - \frac{v_j(\beta_j^{l+1})-v_j(\beta_j^l)}{\beta_j^{l+1}-\beta_j^l}, l=2,\ldots,s_j, j\in M \\
&\gamma_{lj} \ge - \frac{v_j(\beta_j^l)-v_j(\beta_j^{l-1})}{\beta_j^l-\beta_j^{l-1}} + \frac{v_j(\beta_j^{l+1})-v_j(\beta_j^l)}{\beta_j^{l+1}-\beta_j^l}, l=2,\ldots,s_j, j\in M.
\end{aligned}\right.
\end{equation}
then the model \eqref{m:method1_1_0} can be replaced by the following model:
\begin{equation}\tag{M-5}\label{m:method1_1}
\begin{aligned}
&\min \sum\limits_{j=1}^m\sum\limits_{l=2}^{s_j}\gamma_{lj}\\
&\begin{aligned}
\rm{s.t.} &\ E^{A^R}, \ E^{Bound},\ E^{Sort}, \ E^{Slope}.
\end{aligned}
\end{aligned}
\end{equation}
\end{proposition}

\begin{proof}
The constraint set $E^{Slope}$ ensures that $\gamma_{lj} \ge \left|\frac{v_j(\beta_j^l)-v_j(\beta_j^{l-1})}{\beta_j^l-\beta_j^{l-1}} - \frac{v_j(\beta_j^{l+1})-v_j(\beta_j^l)}{\beta_j^{l+1}-\beta_j^l}\right|, l=2,\ldots,s_j, j\in M$.
As the objective of the model \eqref{m:method1_1} is to minimize $\sum\limits_{j=1}^m\sum\limits_{l=2}^{s_j}\gamma_{lj}$, any solutions that satisfy $\gamma_{lj} > \left|\frac{v_j(\beta_j^l)-v_j(\beta_j^{l-1})}{\beta_j^l-\beta_j^{l-1}} - \frac{v_j(\beta_j^{l+1})-v_j(\beta_j^l)}{\beta_j^{l+1}-\beta_j^l}\right|$ are not optimal solutions to the model \eqref{m:method1_1}.

Consequently, $\gamma_{lj} = \left|\frac{v_j(\beta_j^l)-v_j(\beta_j^{l-1})}{\beta_j^l-\beta_j^{l-1}} - \frac{v_j(\beta_j^{l+1})-v_j(\beta_j^l)}{\beta_j^{l+1}-\beta_j^l}\right|$, and the model  \eqref{m:method1_1_0} can be transformed into the model \eqref{m:method1_1}.
\end{proof}

Let $\gamma^*$ be the optimal objective function value of the model \eqref{m:method1_1}. Then, the optimization model for the second stage in Approach 1 can be developed as
\begin{equation}\tag{M-6}\label{m:method1_2}
\begin{aligned}
&\max\ \varepsilon \\
&\begin{aligned}
\rm{s.t.} &\ E^{{A^R}}, \ E^{Bound},\ E^{Sort} \\
& \gamma^* = \sum\limits_{j=1}^m\sum\limits_{l=2}^{s_j}\gamma_{lj} \\
& \gamma_{lj} = \left|\frac{v_j(\beta_j^l)-v_j(\beta_j^{l-1})}{\beta_j^l-\beta_j^{l-1}} - \frac{v_j(\beta_j^{l+1})-v_j(\beta_j^l)}{\beta_j^{l+1}-\beta_j^l}\right|, l=2,\ldots,s_j, j\in M.
\end{aligned}
\end{aligned}
\end{equation}

\textbf{Approach 2}

It is assumed that the model discriminative power has a higher priority than the  model complexity in Approach 2. Therefore, the optimization model in the first stage of Approach 2 is to maximize the parameter $\varepsilon$, and we have
\begin{equation}\tag{M-7}\label{m:method2_1}
\begin{aligned}
&\max \varepsilon \\
&\begin{aligned}
\rm{s.t.} &\ E^{{A^R}}, \ E^{Bound},\ E^{Sort}\\
& \varepsilon >0.
\end{aligned}
\end{aligned}
\end{equation}

Let $\varepsilon^*$ be the optimal objective function value obtained by solving the model \eqref{m:method2_1}, then the optimization model in the second stage is constructed as
\begin{equation}\tag{M-8}\label{m:method2_2}
\begin{aligned}
&\min \sum\limits_{j=1}^m\sum\limits_{l=2}^{s_j}\gamma_{lj}\\
&\begin{aligned}
\rm{s.t.} &\ E^{A^R}, \ E^{Bound},\ E^{Sort}, \ E^{Slope} \\
& \varepsilon=\varepsilon^*.
\end{aligned}
\end{aligned}
\end{equation}

By solving the optimization models in Approach 1 or Approach 2, the representative model can be derived for MCS with non-monotonic criteria. To be specific, the marginal value functions of all criteria and the category thresholds are determined. Note that the variables $b_0$ and $b_q$ have been removed from the constraint set when establishing the above optimization models. Let $v_j(\beta_j^l)^*$ be the optimal solution for $v_j(\beta_j^l)$ obtained in the second stage optimization model, $l=1,\ldots,s_j+1, j\in M$, then $b_0$ and $b_q$ can be obtained by $b_0 = \sum\nolimits_{j=1}^m \min v_j(\beta_j^l)^*,  b_q = \sum\nolimits_{j=1}^m \max v_j(\beta_j^l)^*+ \varepsilon$, respectively.

Furthermore, based on the representative sorting model, the sorting result for the non-reference alternatives can be obtained by applying the threshold-based value-driven sorting procedure.

\begin{remark}
In practice, the choice between the two proposed approaches to derive a representative sorting model depends on the decision maker's priorities. If the decision maker considers model complexity to be more important than model discriminative power, Approach 1 is suitable. Conversely, if the decision maker values model discriminative power more, Approach 2 is the preferred option.
\end{remark}

\subsection{The proposed algorithm}\label{sec:4.5}
In this section, we propose Algorithm \ref{alg:1} to summarize the proposed approach. In addition, a flow chart of Algorithm \ref{alg:1} is depicted in Fig. \ref{fig:flowchart}.
\begin{figure}[htbp]
\centering
\includegraphics[scale=0.55]{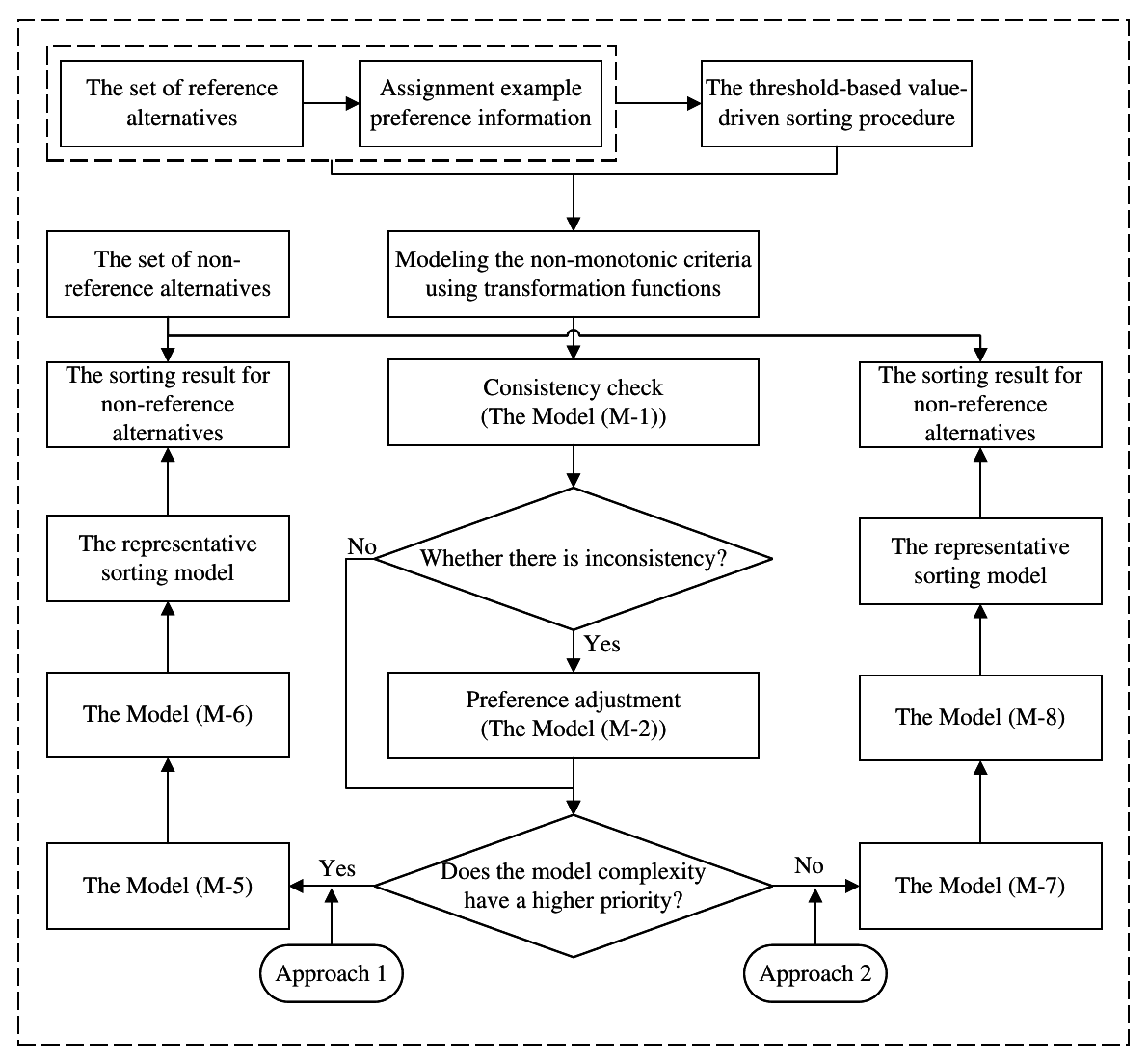}
\caption{The flow chart of the proposed algorithm}
\label{fig:flowchart}
\end{figure}

\begin{algorithm}\label{alg:1}
~\\
{\bf Input:} The decision matrix $X=(x_{ij})_{n\times m}$, the assignment example preference information $\{a_{i}\rightarrow C_{{B}_{i}}|a_{i}\in A^R\}$, the number of subintervals for each criterion $s_j$, $j\in M$.\\
{\bf Output:} The representative sorting model and sorting result for non-reference alternatives.
\begin{enumerate}[\bf Step 1:]
  \item Based on the decision matrix $X=(x_{ij})_{n\times m}$ and the specified number of subintervals for each criterion $s_j$, identify the breakpoints for each criterion, represented as $\beta_j^l$, $l=1,\ldots,s_j+1$, $j\in M$.
  \item Employ the model \eqref{m:con_check} to check whether the assignment example preference information provided by the decision maker is consistent. Proceed to Step 4 if the optimal objective function value of the model \eqref{m:con_check} equals 0. Otherwise, continue to the following step.
  \item Solve the model \eqref{m:mini} to obtain the adjusted assignment example preference information for each reference alternative, denoted by $\{a_{i}\rightarrow C_{{B}_{i}^*}|a_{i}\in A^R\}$, and then let $B_{i}=B_{i}^*$, $\forall a_i\in A^R$.
  \item Determine the priority between model complexity and discriminative power. If model complexity takes precedence, solve the models \eqref{m:method1_1} and \eqref{m:method1_2} to derive the marginal value functions and category thresholds, and then proceed to Step 5. If, however, model discriminative power is deemed more important, solve the models \eqref{m:method2_1} and \eqref{m:method2_2} to obtain the marginal value functions and category thresholds. Then move on to Step 5.
 \item In light of the representative sorting model, use the threshold-based value-driven sorting procedure to determine the sorting result for non-reference alternatives.
 \item Output the representative sorting model and the sorting result for non-reference alternatives.
\end{enumerate}
\end{algorithm}

\subsection{Robustness analysis}\label{sec:4.6}
The proposed lexicographic optimization-based approaches offer a solution for deriving a representative sorting model. However, different approaches for obtaining representative sorting models may lead to varying sorting result for non-reference alternatives. To validate the robustness of the proposed non-monotonic preference modeling approach, this subsection introduces a robustness analysis method based on the concept of ``possible assignments'', which refers to the set of categories to which a given alternative can be assigned by at least one sorting model that aligns with the assignment example preference information provided by the decision maker \citep{Greco10ejor,Liu15ejor,Kadzinski20ijar}.

To facilitate the analysis, let $C^P_i$ be the set of possible assignments for a non-reference alternative $a_{i}\in A^N$. For each non-reference alternative $a_{i}\in A^N$ and each category $C_h$, $h\in Q$, we add $a_i\rightarrow C_h$ to the assignment example preference information provided by the decision maker, and then solve the following optimization model:
\begin{equation}\tag{M-9}\label{m:robust}
\begin{aligned}
&\max \varepsilon \\
&\begin{aligned}
\rm{s.t.} &\ E^{{A^R}}, \ E^{Bound},\ E^{Sort}.
\end{aligned}
\end{aligned}
\end{equation}

Let $\varepsilon^*_h$ denote the optimal objective function value when $a_i\rightarrow C_h$ is included in the assignment example preference information. If $\varepsilon^*_h>0$, then the non-reference alternative $a_i$ can be possibly assigned to the category $C_h$. Conversely, if $\varepsilon^*_h\le 0$, no compatible sorting model exists that can assign the non-reference alternative $a_i$ to the category $C_h$. For each non-reference alternative $a_i$, the model \eqref{m:robust} is solved $q$ times for all $q$ categories. On this basis, the set of possible assignments for the non-reference alternative $a_i$ is determined as $C^P_i=\{C_h| \varepsilon^*_h>0, h\in Q\}$, $\forall a_{i}\in A^N$. Let $|C^P_i|$ be the number of elements in $C^P_i$, if $|C^P_i|=1$, then the non-reference alternative $a_i$ is consistently assigned to a single category $C_h$ by all compatible sorting models, indicating the method's maximum robustness. Otherwise, the non-reference alternative $a_i$ can be assigned to multiple categories.

To further evaluate the robustness of the approaches, we employ the average possible assignment ($APA$) metric, defined as the average number of possible assignments for all non-reference alternatives $a_i\in A^N$. To ensure comparability of results, the $APA$ metric is normalized as follows \citep{Kadzinski21ejor}:
\begin{equation}\label{eq:APA}
APA=1 - \frac{1}{|A^N|}\sum\limits_{a_i\in A^N}\frac{|C^P_i| - 1}{q - 1}.
\end{equation}
where $|A^N|$ and $|C^P_i|$ denote the number of elements in the sets $A^N$ and $C^P_i$, respectively. The value of $APA$ ranges from 0 to 1. An $APA$ value of 1 indicates that all non-reference alternatives are assigned to a single category, reflecting maximum robustness. Conversely, an APA value of 0 suggests that all non-reference alternatives can possibly be assigned to any category.

By applying Eq. \eqref{eq:APA}, we can assess the robustness of the proposed approaches, where a higher $APA$ value indicates greater robustness.

\section{Illustrative example}\label{sec:5}
In this section, we begin by implementing the proposed approaches using a simple illustrative example. Subsequently, we provide further discussions based on the illustrative example.

\subsection{Implementation of the proposed approaches}\label{sec:5.1}
In this subsection, we will apply the proposed approach to a real-world example using the data from Section 4 of \cite{Guo19eswa}. In particular, the majority of the data predominantly stems from \cite{Despotis95ama} and \cite{Ghaderi17ejor}. In this example, a decision maker wants to assign twenty firms (denoted by $A=\{a_1,a_2,\ldots,a_{20}\}$) into four categories based on their financial states: firms in the worst financial state ($C_1$), firms in a lower-intermediate financial state ($C_2$),  firms in an upper-intermediate financial state ($C_3$), and firms in the best financial state ($C_4$). The following three criteria are considered: the cash to total assets ($g_1$), the long term debt and stockholder's equity to fixed assets ($g_2$), and the total liabilities to total assets ($g_3$). The decision matrix is presented as below:
\begin{footnotesize}
\[\setlength{\arraycolsep}{1.5pt}
X= \left(
  \begin{array}{cccccccccccccccccccc}
  3.8 & 5.84 & 0.04 & 4.89 & 0.57 & 16.7 & 3.16 & 25.42 & 17.99 & 3.98 & 0.76 & 24.16 & 2.53 & 35.06 & 0.72 & 24 & 8.86 & 10.58 & 16.35 & 1.7\\
  2.4 & 1.96 & 1.14 & 2.92 & 1.72 & 2.32 & 4.1 & 3.35 & 1.34 & 3.26 & 2.74 & 2.83 & 2.54 & 9.56 & 0.97 & 2.5 & 29.06 & 4.03 & 3.6 & 5.92\\
  60.7 & 63.7 & 64.26 & 55.04 & 64.7 & 53.29 & 23.9 & 59.03 & 73.84 & 84.95 & 84.44 & 70.51 & 81.05 & 61.08 & 99.67 & 99.92 & 47.4 & 89.64 & 56.55 & 85.83\\
  \end{array}\right)^{\rm T}.\]
\end{footnotesize}

Suppose that the decision maker provides assignment example preference information as follows: $\{a_{2}\rightarrow C_4, a_{3}\rightarrow C_2, a_{9}\rightarrow C_3, a_{10}\rightarrow C_2, a_{12}\rightarrow C_1, a_{17}\rightarrow C_3\}$. In what follows, we implement Algorithm \ref{alg:1} to determine the sorting result for non-reference alternatives.

Without loss of generality, we assume that the numbers of subintervals for all criteria are set to 4, i.e., $s_j=4$, $j=1,2,3$. On this basis, we can calculate the breakpoints for each criterion. Let us take the criterion $g_1$ as an example. As $\beta_1^-=0.04$ and $\beta_1^+=35.06$, we obtain the breakpoints as follows: $\beta_1^1=0.04$, $\beta_1^2=8.795$, $\beta_1^3=17.55$, $\beta_1^4=26.305$, and $\beta_1^5=35.06$.

Following this, we apply the model \eqref{m:con_check} to check the consistency of the provided assignment example preference information. Upon solving the model, we find that the optimal objective function value is 0, signifying that the decision maker's assignment example preference information is consistent and requires no adjustments. Consequently, we can proceed to utilize the lexicographic optimization-based approaches to derive a representative sorting model.

\textbf{The implementation of Approach 1}

We first assume that the model complexity has a higher priority. By solving the model \eqref{m:method1_1}, we have that $\gamma^*=0.000683$. Subsequently, by solving the model \eqref{m:method1_2}, the value of the parameter $\varepsilon$ is determined as $\varepsilon^*=0.001$. Accordingly, we obtain the marginal value functions and category thresholds. Denoting $h_{jl}=v_j(\beta_j^l)$, $h=1,2,3$, $l=1,2,\ldots,5$, the marginal values for each breakpoint are expressed as:
\[H=(h_{jl})_{3\times 5}= \left(
  \begin{array}{ccccc}
  0.004 & 0.0071 & 0.0057 & 0.0029 & 0\\
  1 & 0.9992 & 0.9985 & 0.9977 & 0.9969\\
  1 & 0.9989 & 0.9978 & 0.9967 & 0.9957
  \end{array}\right).\]

Examining the marginal values at each breakpoint, we observe that $g_1$ displays non-monotonic behavior, in contrast to the monotonic trends observed in $g_2$ and $g_3$. This demonstrates the capability of the proposed approach to effectively capture the non-monotonic nature of criteria.

Furthermore, the category threshold vector is obtained as $b=(1.9926, 2.0017, 2.0027, 2.0037, 2.0081)^{\rm T}$. Applying the threshold-based value-driven sorting procedure, each alternative $a_i$ is assigned to a specific category denoted by $f_i\in C$, where $i=1,2,\ldots,20$. The global value and the sorting result for all firms are shown in Table \ref{table:global_v}.

\begin{table}[htbp]
\centering
\setlength{\abovecaptionskip}{0pt}
\setlength{\belowcaptionskip}{10pt}
\caption{The global value and sorting result for all firms obtained by Approach 1}
\label{table:global_v}
\begin{tabular}{cccccccccccc}
\toprule
Firm & $V(a_i)$ & $f_i$  & Firm & $V(a_i)$ & $f_i$ & Firm & $V(a_i)$ & $f_i$ & Firm & $V(a_i)$ & $f_i$    \\ \midrule
$a_1$ & 2.0031 & $C_3$ & $a_6$ & 2.0041 & $C_4$ & $a_{11}$ & 2.0007 & $C_1$ & $a_{16}$ & 1.9992 & $C_1$ \\
$a_2$ & 2.0037 & $C_4$ & $a_7$ & 2.0048 & $C_4$ & $a_{12}$ & 2.0007 & $C_1$ & $a_{17}$ & 2.0027 & $C_3$ \\
$a_3$ & 2.0017 & $C_2$ & $a_8$ & 2.0009 & $C_1$ & $a_{13}$ & 2.0015 & $C_1$ & $a_{18}$ & 2.0028 & $C_3$ \\
$a_4$ & 2.0038 & $C_4$ & $a_9$ & 2.0027 & $C_3$ & $a_{14}$ & 1.9970 & $C_1$ & $a_{19}$ & 2.0038 & $C_4$ \\
$a_5$ & 2.0018 & $C_2$ & $a_{10}$ & 2.0017 & $C_2$ & $a_{15}$ & 2 & $C_1$ & $a_{20}$ & 2.0006 & $C_1$ \\
\bottomrule
\end{tabular}
\end{table}

To further validate the effectiveness of the transformation functions proposed in Section \ref{sec:4.1}, we also present the transformed marginal value functions in Fig. \ref{fig:numericla_v1}. Evidently, the transformed marginal value functions adhere to UTA-like standard forms, and the relative importance of the three criteria $g_1$, $g_2$ and $g_3$ (i.e., criteria weights) is obtained as 0.4930, 0.2114 and 0.2956, respectively. Moreover, the category threshold vector after transformation is obtained as $b^S=(0, 0.6275, 0.6965, 0.7655, 1.069)^{\rm T}$. Table \ref{table:T_V} includes the global value and sorting result for all firms based on the transformed marginal value functions.
\begin{figure}[htbp]
\centering
\includegraphics[scale=0.4]{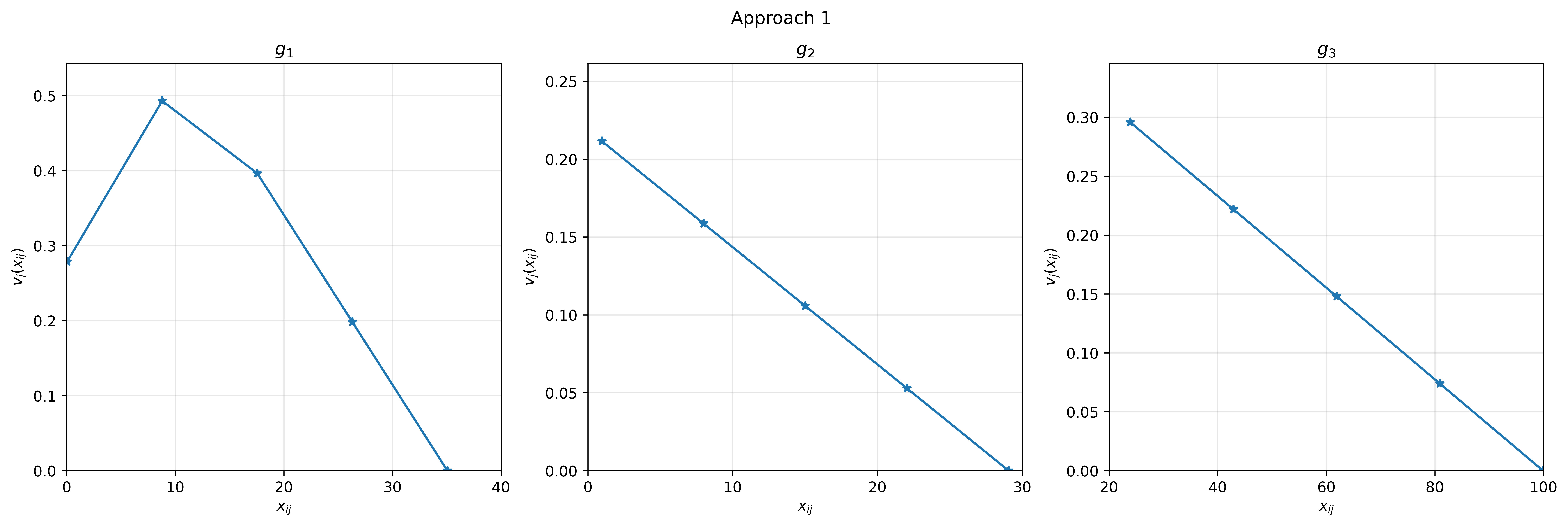}
\caption{The transformed marginal value functions obtained by Approach 1}
\label{fig:numericla_v1}
\end{figure}

\begin{table}[htbp]
\centering
\setlength{\abovecaptionskip}{0pt}
\setlength{\belowcaptionskip}{10pt}
\caption{The transformed global value and sorting result for all firms obtained by Approach 1}
\label{table:T_V}
\begin{tabular}{cccccccccccc}
\toprule
Firm & $V^S(a_i)$ & $f_i$  & Firm & $V^S(a_i)$ & $f_i$ & Firm & $V^S(a_i)$ & $f_i$ & Firm & $V^S(a_i)$ & $f_i$    \\ \midrule
$a_1$ & 0.7239 & $C_3$ & $a_6$ & 0.7884 & $C_4$ & $a_{11}$ & 0.5546 & $C_1$ & $a_{16}$ & 0.4503 & $C_1$ \\
$a_2$ & 0.7655 & $C_4$ & $a_7$ & 0.8385 & $C_4$ & $a_{12}$ & 0.5585 & $C_1$ & $a_{17}$ & 0.6965 & $C_3$ \\
$a_3$ & 0.6275 & $C_2$ & $a_8$ & 0.5707 & $C_1$ & $a_{13}$ & 0.6126 & $C_1$ & $a_{18}$ & 0.7016 & $C_3$ \\
$a_4$ & 0.7687 & $C_4$ & $a_9$ & 0.6965 & $C_3$ & $a_{14}$ & 0.2978 & $C_1$ & $a_{19}$ & 0.7699 & $C_4$ \\
$a_5$ & 0.6344 & $C_2$ & $a_{10}$ & 0.6275 & $C_2$ & $a_{15}$ & 0.5077 & $C_1$ & $a_{20}$ & 0.5483 & $C_1$ \\
\bottomrule
\end{tabular}
\end{table}

From Tables \ref{table:global_v} and \ref{table:T_V}, it is evident that the sorting result for all alternatives are the same before and after transformation.

\textbf{The implementation of Approach 2}

In this scenario, assuming that model discriminative power takes higher priority than model complexity, Approach 2 is used to derive the representative sorting model. Under this assumption, solving the model \eqref{m:method2_1} yields the maximum value of the parameter $\varepsilon$ as $\varepsilon^*=0.2675$. Subsequently, by solving the model \eqref{m:method2_2}, we obtain the optimal objective function value as $\gamma^*=0.458$, along with the marginal values of each breakpoint and the category thresholds. Let $h_{jl}=v_j(\beta_j^l)$, $h=1,2,3$, $l=1,2,\ldots,5$, then the marginal values of each breakpoint is shown as
\[H= (h_{jl})_{3\times 5}=\left(
  \begin{array}{ccccc}
  0.0336 & 1 & 0.5562 & 0 & 0\\
  1 & 0 & 0 & 0 & 0\\
  1 & 0.6972 & 0.3941 & 0 & 1
  \end{array}\right).\]

Moreover, the category threshold vector is obtained as $b=(0, 1.3547, 1.6222, 1.8898, 3.2675)^{\rm T}$, and the global value and sorting result for all firms are presented in Table \ref{table:global_v2}.
\begin{table}[htbp]
\centering
\setlength{\abovecaptionskip}{0pt}
\setlength{\belowcaptionskip}{10pt}
\caption{The global value and sorting result obtained by Approach 2}
\label{table:global_v2}
\begin{tabular}{cccccccccccc}
\toprule
Firm & $V(a_i)$ & $f_i$  & Firm & $V(a_i)$ & $f_i$ & Firm & $V(a_i)$ & $f_i$ & Firm & $V(a_i)$ & $f_i$    \\ \midrule
$a_1$ & 1.6584 & $C_3$ & $a_6$ & 1.9386 & $C_4$ & $a_{11}$ & 1.0465 & $C_1$ & $a_{16}$ & 1.9286 & $C_4$ \\
$a_2$ & 1.8898 & $C_4$ & $a_7$ & 1.9323 & $C_4$ & $a_{12}$ & 1.0872 & $C_1$ & $a_{17}$ & 1.6222 & $C_3$ \\
$a_3$ & 1.3547 & $C_2$ & $a_8$ & 1.1573 & $C_1$ & $a_{13}$ & 1.092 & $C_1$ & $a_{18}$ & 1.9329 & $C_4$ \\
$a_4$ & 1.7949 & $C_3$ & $a_9$ & 1.6222 & $C_3$ & $a_{14}$ & 0.4073 & $C_1$ & $a_{19}$ & 1.7221 & $C_3$ \\
$a_5$ & 1.3215 & $C_1$ & $a_{10}$ & 1.3547 & $C_2$ & $a_{15}$ & 2.0955 & $C_4$ & $a_{20}$ & 0.7705 & $C_1$ \\
\bottomrule
\end{tabular}
\end{table}

Analogously, we display the transformed marginal value functions in Fig. \ref{fig:numericla_v2}. In addition, the relative importance of the three criteria $g_1$, $g_2$ and $g_3$ in Approach 2 is obtained as 0.3333, 0.3333 and 0.3334, respectively, and the category threshold vector after transformation is determined as $b^S=(0, 0.4516, 0.5407, 0.6299,\\$ $ 1.0892)^{\rm T}$. The transformed global value and sorting result for all firms are shown in Table \ref{table:T_V2}.
\begin{figure}[htbp]
\centering
\includegraphics[scale=0.4]{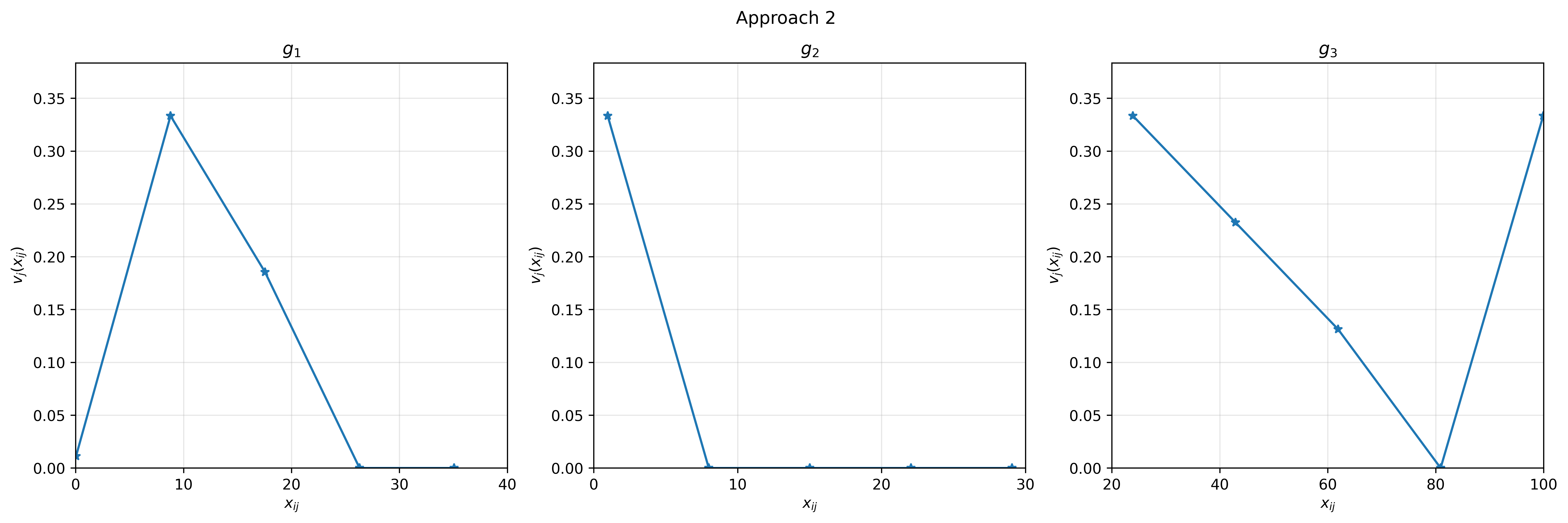}
\caption{The transformed marginal value functions obtained by Approach 2}
\label{fig:numericla_v2}
\end{figure}

\begin{table}[htbp]
\centering
\setlength{\abovecaptionskip}{0pt}
\setlength{\belowcaptionskip}{10pt}
\caption{The transformed global value and sorting result for all firms obtained by Approach 2}
\label{table:T_V2}
\begin{tabular}{cccccccccccc}
\toprule
Firm & $V^S(a_i)$ & $f^S_i$  & Firm & $V^S(a_i)$ & $f^S_i$ & Firm & $V^S(a_i)$ & $f^S_i$ & Firm & $V^S(a_i)$ & $f^S_i$    \\ \midrule
$a_1$ & 0.5528 & $C_3$ & $a_6$ & 0.6462 & $C_4$ & $a_{11}$ & 0.3488 & $C_1$ & $a_{16}$ & 0.6429 & $C_4$ \\
$a_2$ & 0.6299 & $C_4$ & $a_7$ & 0.6441 & $C_4$ & $a_{12}$ & 0.3624 & $C_1$ & $a_{17}$ & 0.5407 & $C_3$ \\
$a_3$ & 0.4516 & $C_2$ & $a_8$ & 0.3858 & $C_1$ & $a_{13}$ & 0.364 & $C_1$ & $a_{18}$ & 0.6443 & $C_4$ \\
$a_4$ & 0.5983 & $C_3$ & $a_9$ & 0.5407 & $C_3$ & $a_{14}$ & 0.1358 & $C_1$ & $a_{19}$ & 0.574 & $C_3$ \\
$a_5$ & 0.4405 & $C_1$ & $a_{10}$ & 0.4516 & $C_2$ & $a_{15}$ & 0.6985 & $C_4$ & $a_{20}$ & 0.2568 & $C_1$ \\
\bottomrule
\end{tabular}
\end{table}

Comparing Tables \ref{table:global_v2} and \ref{table:T_V2}, we observe that the sorting result for the alternatives remain consistent before and after the transformation.

Additionally, a comparison of Figs. \ref{fig:numericla_v1} and \ref{fig:numericla_v2}, reveals that the marginal value functions of $g_1$ derived by both approaches are non-monotonic and exhibit similar shapes, first increasing and then decreasing. This is consistent with practical financial considerations: a higher cash to total assets typically signifies good liquidity and a strong ability to repay debt, positively impacting the financial state of the firm. However, an excessively high ratio may indicate inefficient use of cash for investment or growth opportunities, which can negatively affect the firm's financial state.
The marginal value functions of $g_2$ obtained by both approaches are monotonically decreasing, though they differ somewhat in shape. As the long term debt and stockholder's equity to fixed assets increases, it might indicate that the firm is relying more on long term capital to finance its fixed assets. This increased reliance can elevate financial risk and negatively impact the firm's financial state.
For the marginal value functions of $g_3$, Approach 1 results in a monotonically decreasing function, while Approach 2 yields a non-monotonic one. An increasing total liabilities to total assets may lead to higher interest payments and greater financial strain, further diminishing the firm's ability to manage cash and potentially causing financial instability. However, in Approach 2, beyond a certain point, the firm may benefit from leverage through enhanced returns, tax advantages or improved capital efficiency, resulting in a recovery and an improved financial state.
Furthermore, there are differences in the sorting result for some non-reference alternatives between Approaches 1 and 2. This discrepancy is expected,  as the two approaches consider different priorities regarding model complexity and discriminative power.

\subsection{Further discussions}\label{sec:5.2}
In this subsection, we provide further discussions using the example presented in Section \ref{sec:5.1}. First, we explore the impact of  the number of subintervals for each criterion $s_j$ on the proposed approaches. Subsequently, we compare the proposed approaches with the non-monotonic criteria modeling method introduced by \cite{Kadzinski20ijar}.

\subsubsection{The impact of $s_j$ on the proposed approaches}

To investigate the impact of the number of subintervals for each criterion on the proposed approaches, we set $s_j\in\{5,6,7\}$, $j=1,2,3$. For each fixed $s_j$ value, Algorithm \ref{alg:1} is implemented under different priorities for model complexity and discriminative power.

For Approach 1, the transformed marginal value functions corresponding to  different $s_j$ values are depicted in Figs. \ref{fig:Approach1_5}~-~\ref{fig:Approach1_7}. Let $F^1_{s_j}$ be the sorting result for all firms obtained by Approach 1 when the number of subintervals for each criterion is $s_j$. The sorting result under different $s_j$ values are presented in Table \ref{table:Approach1_sj}.

By comparing Figs. \ref{fig:numericla_v1}, \ref{fig:Approach1_5}~-~\ref{fig:Approach1_7}, it is evident that for Approach 1, the shape of marginal value functions remains largely consistent across different $s_j$ values. Furthermore, as observed in Tables \ref{table:T_V} and \ref{table:Approach1_sj}, when $s_j=5$, the sorting result for the alternative $a_{18}$ differs from those for $s_j=4,6,7$, while the sorting results for the remaining alternatives remain consistent across all $s_j$ values. These findings suggest that Approach 1 is relatively stable, as the value of
$s_j$ has minimal impact on both the shape of the marginal value functions and the sorting result.

\begin{figure}[H]
\centering
\includegraphics[scale=0.3]{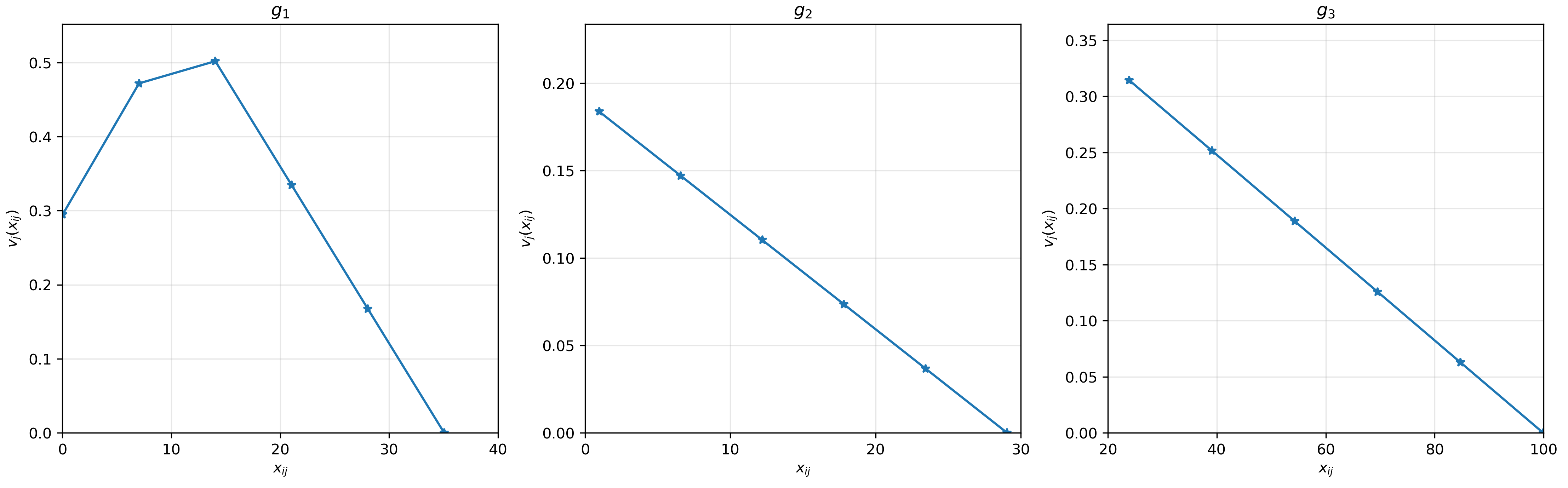}
\caption{The transformed marginal value functions obtained by Approach 1 ($s_j=5$) }
\label{fig:Approach1_5}
\end{figure}

\begin{table}[H]
\centering
\setlength{\abovecaptionskip}{0pt}
\setlength{\belowcaptionskip}{10pt}
\caption{The sorting result obtained by Approach 1 under different $s_j$ values}
\label{table:Approach1_sj}
\begin{tabular}{cc}
\toprule
$s_j$ & $F^1_{s_j}$    \\ \midrule
5 & $\{C_3, C_4, C_2, C_4, C_2,  C_4,C_4,C_1,C_3,C_2,  C_1,C_1,C_1,C_1,C_1, C_1,C_3,C_2,C_4,C_1\}$ \\
6 & $\{C_3, C_4, C_2, C_4, C_2,  C_4,C_4,C_1,C_3,C_2,  C_1,C_1,C_1,C_1,C_1, C_1,C_3,C_3,C_4,C_1\}$ \\
7 & $\{C_3, C_4, C_2, C_4, C_2,  C_4,C_4,C_1,C_3,C_2,  C_1,C_1,C_1,C_1,C_1, C_1,C_3,C_3,C_4,C_1\}$ \\
\bottomrule
\end{tabular}
\end{table}

\begin{figure}[H]
\centering
\includegraphics[scale=0.3]{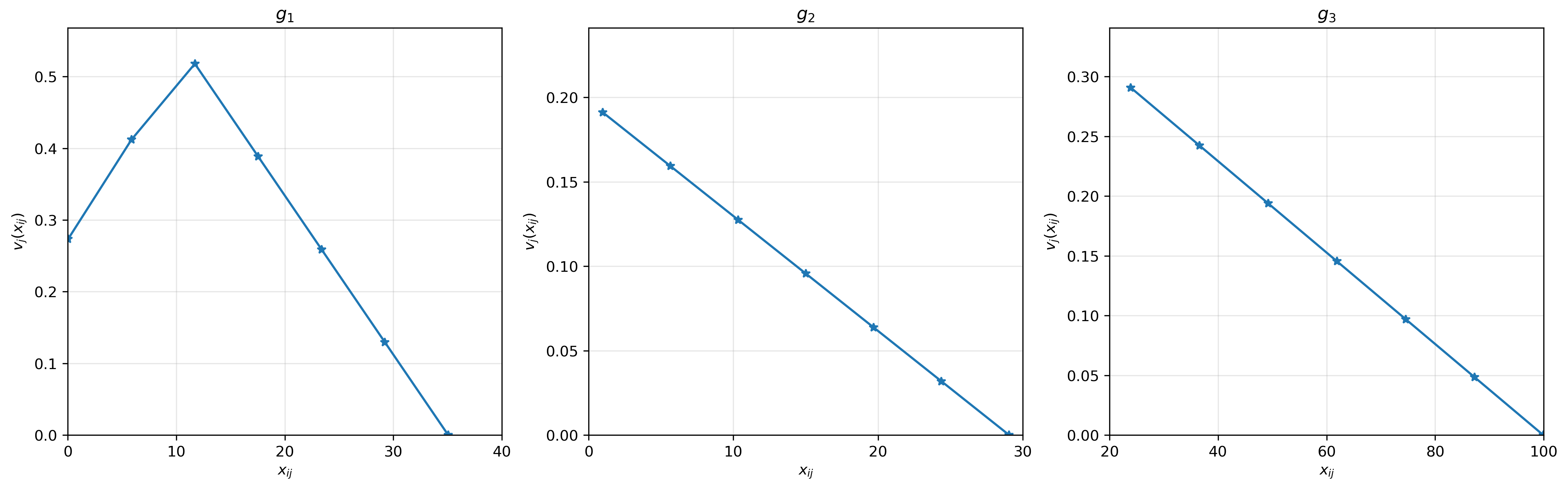}
\caption{The transformed marginal value functions obtained by Approach 1 ($s_j=6$) }
\label{fig:Approach1_6}
\end{figure}

\begin{figure}[H]
\centering
\includegraphics[scale=0.3]{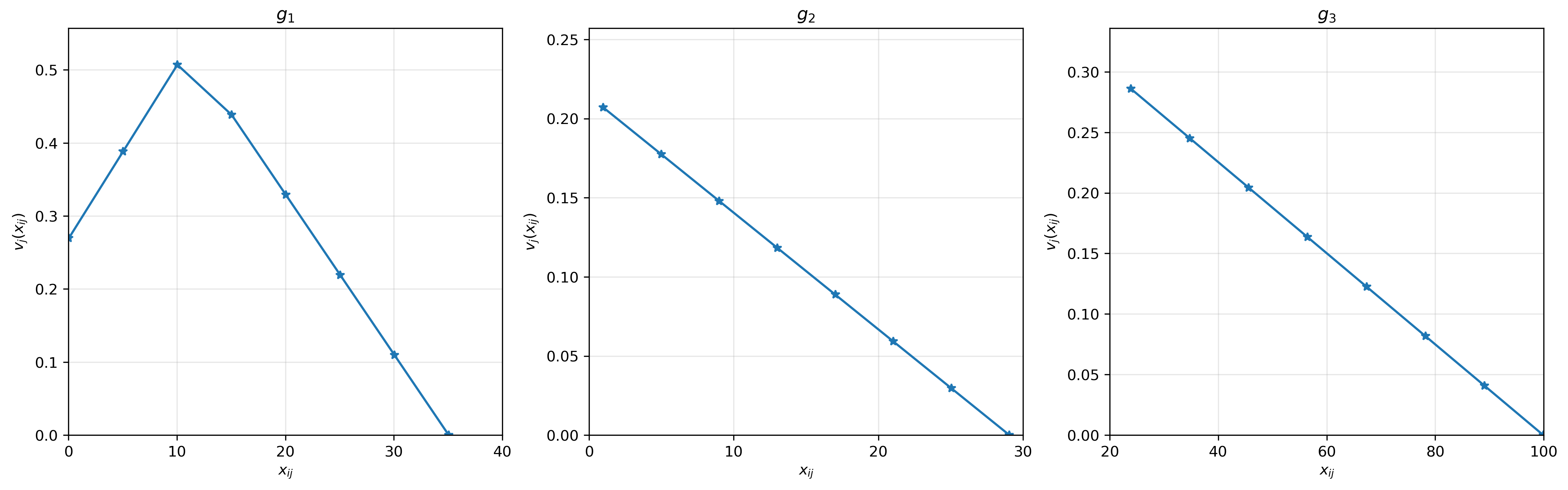}
\caption{The transformed marginal value functions obtained by Approach 1 ($s_j=7$) }
\label{fig:Approach1_7}
\end{figure}

Similarly, for Approach 2, the transformed marginal value functions corresponding to different $s_j$ values are shown in Figs. \ref{fig:Approach2_5}~-~\ref{fig:Approach2_7}, respectively, and the sorting result under different $s_j$ values are presented in Table \ref{table:Approach2_sj}.

\begin{figure}[H]
\centering
\includegraphics[scale=0.3]{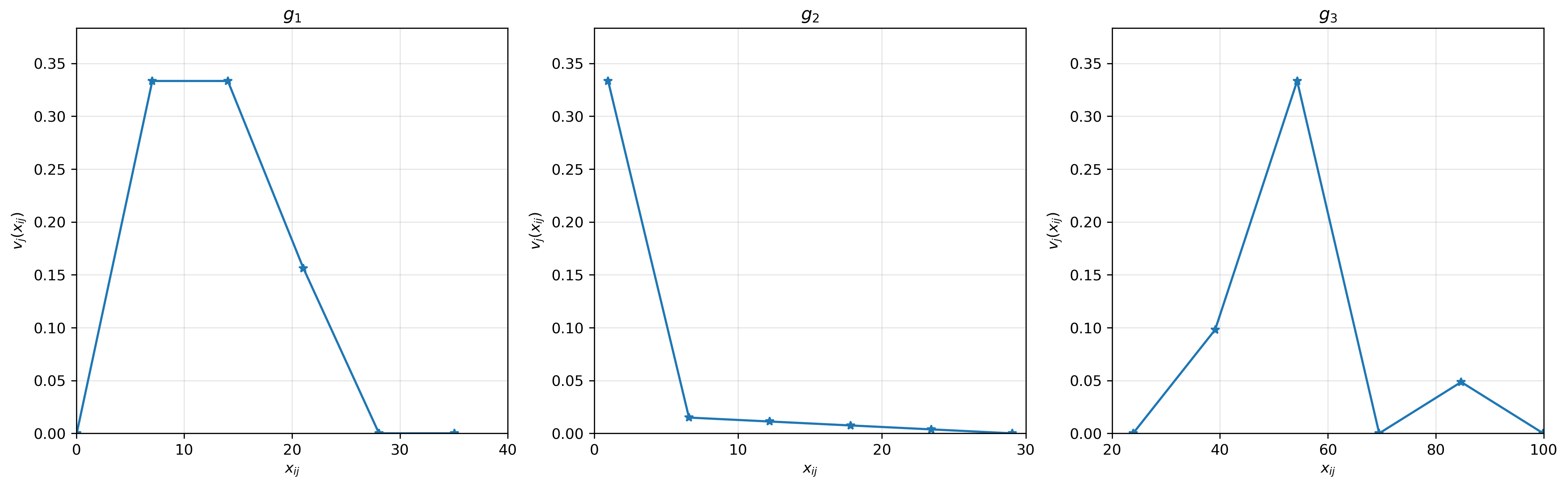}
\caption{The transformed marginal value functions obtained by Approach 2 ($s_j=5$)}
\label{fig:Approach2_5}
\end{figure}

\begin{table}[htbp]
\centering
\setlength{\abovecaptionskip}{0pt}
\setlength{\belowcaptionskip}{10pt}
\caption{The sorting result obtained by Approach 2 under different $s_j$ values}
\label{table:Approach2_sj}
\begin{tabular}{cc}
\toprule
$s_j$ & $F^2_{s_j}$    \\ \midrule
5 & $\{C_3, C_4, C_2, C_4, C_1,  C_4,C_1,C_2,C_3,C_2,  C_1,C_1,C_1,C_1,C_1, C_1,C_3,C_2,C_4,C_1\}$ \\
6 & $\{C_3, C_4, C_2, C_3, C_2,  C_4,C_3,C_1,C_3,C_2,  C_1,C_1,C_1,C_1,C_1, C_1,C_3,C_2,C_4,C_1\}$ \\
7 & $\{C_3, C_4, C_2, C_4, C_1,  C_4,C_2,C_2,C_3,C_2,  C_1,C_1,C_1,C_1,C_1, C_1,C_3,C_2,C_4,C_1\}$ \\
\bottomrule
\end{tabular}
\end{table}

\begin{figure}[H]
\centering
\includegraphics[scale=0.3]{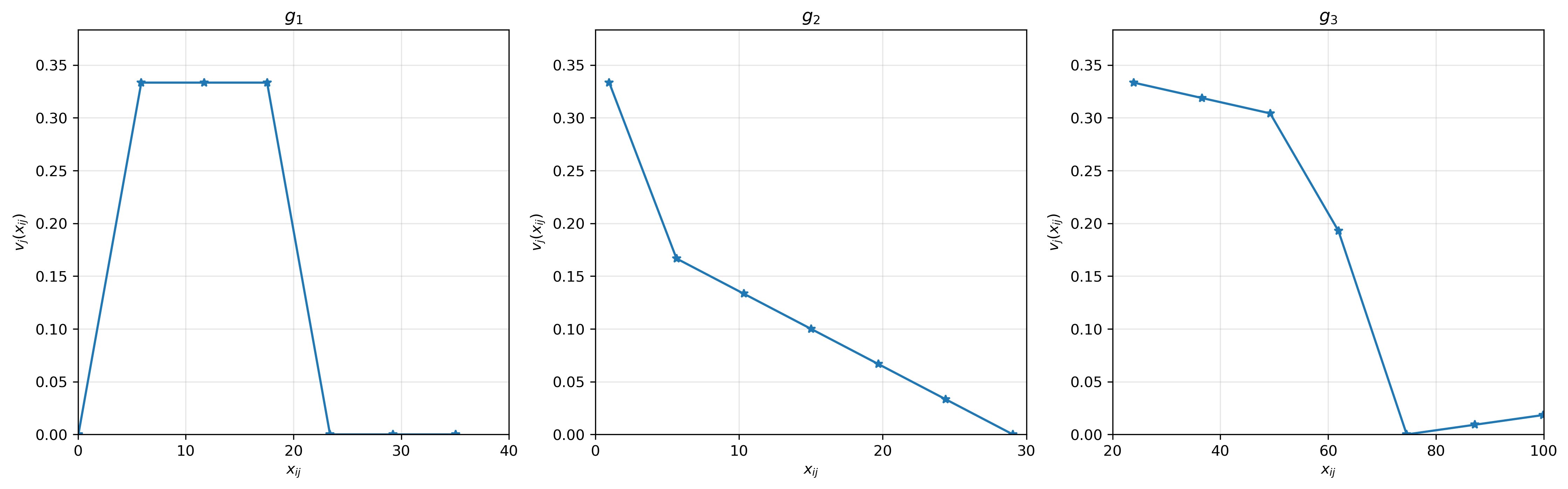}
\caption{The transformed marginal value functions obtained by Approach 2 ($s_j=6$)}
\label{fig:Approach2_6}
\end{figure}

\begin{figure}[H]
\centering
\includegraphics[scale=0.3]{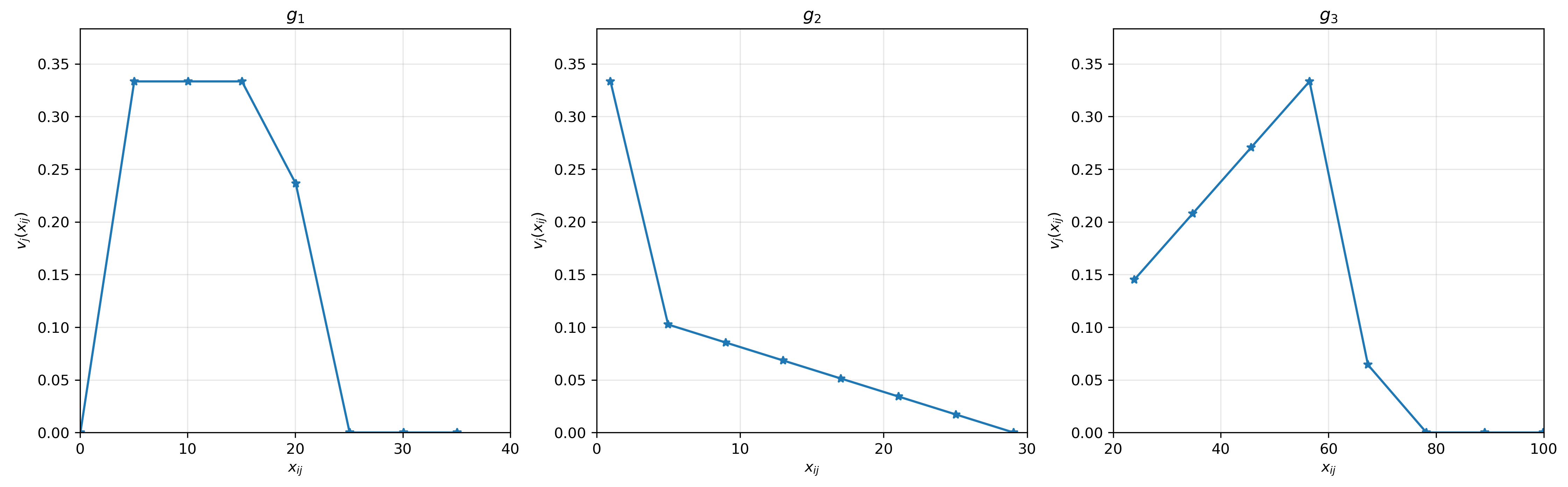}
\caption{The transformed marginal value functions obtained by Approach 2 ($s_j=7$)}
\label{fig:Approach2_7}
\end{figure}

By comparing Figs. \ref{fig:numericla_v2}, \ref{fig:Approach2_5}~-~\ref{fig:Approach2_7}, it is clear that the marginal value functions for the criteria $g_1$ and $g_2$ exhibit similar shape across different $s_j$ values, while the marginal value functions for the criterion $g_3$ show distinct shapes. Specifically, the marginal value function for the criterion $g_1$ initially increases and then decreases, whereas the marginal value function for the criterion $g_2$ consistently decreases across all $s_j$ values. For the criterion $g_3$, the shape of marginal value functions are similar when $s_j=4$ and $s_j=6$, and when $s_j=5$ and $s_j=7$. Moreover, the sorting result for the alternatives $a_4$, $a_5$, $a_7$, $a_8$, $a_{15}$, $a_{16}$, $a_{18}$ and $a_{19}$ differs under different $s_j$ values. These observations indicate that the value of $s_j$ significantly impacts Approach 2.

From the above analysis, it is evident that $s_j$ has a smaller impact on Approach 1 and a greater impact on Approach 2. In practical MCS problems, two methods can be considered for determining the value of $s_j$: the decision maker can either directly assign a value to $s_j$ based on their knowledge and experience, or the value that yields the best performance across different values can be selected as the final $s_j$ value.

\subsubsection{Comparisons with the non-monotonic criteria modeling method in \cite{Kadzinski20ijar}}

To emphasize the characteristics of the proposed approaches, we compare them with the non-monotonic criteria modeling method  introduced by \cite{Kadzinski20ijar}. This method seeks to  control the complexity of the inferred marginal value functions and identifies the performance levels associated with the maximum and minimum marginal values taken for each criterion by introducing binary variables.

We implement \cite{Kadzinski20ijar}'s method using the data presented in Section \ref{sec:5.1}. The inferred marginal value functions are displayed in Fig. \ref{fig:IJAR_results}, and the category thresholds are obtained as $b=(0, 0.0980, 0.0990,0.132, 1.001)^{\rm T}$. Additionally, the global values and sorting result for all firms are shown in Table \ref{table:IJAR_V}.
\begin{figure}[htbp]
\centering
\includegraphics[scale=0.3]{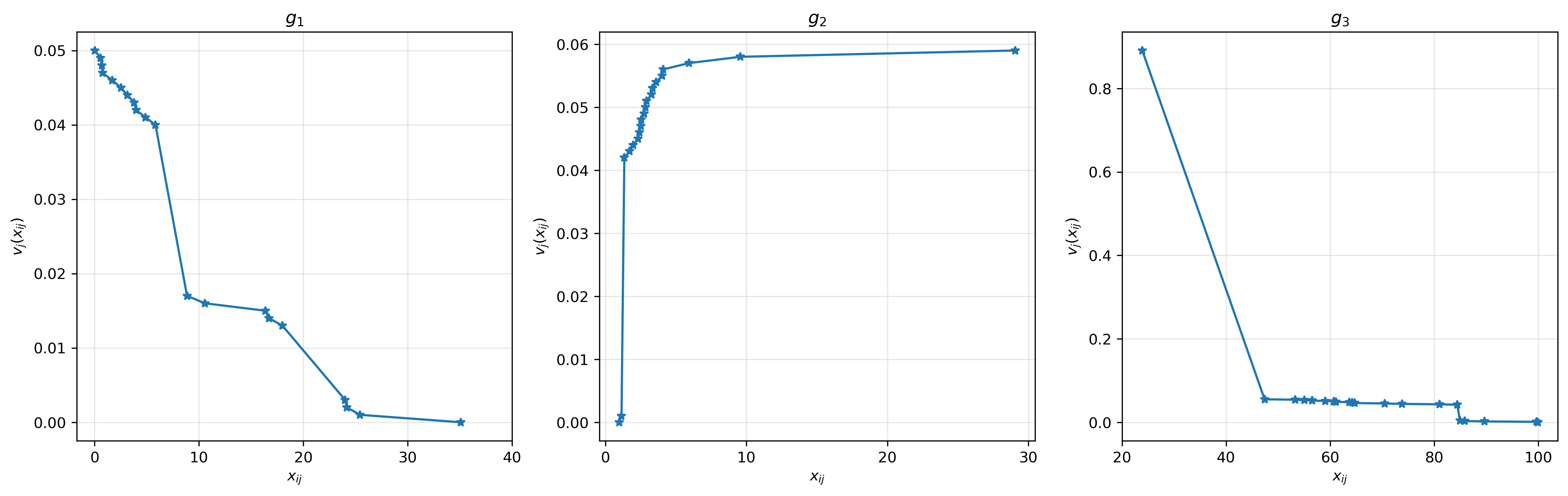}
\caption{The marginal value functions obtained by  \cite{Kadzinski20ijar}'s method}
\label{fig:IJAR_results}
\end{figure}

\begin{table}[htbp]
\centering
\setlength{\abovecaptionskip}{0pt}
\setlength{\belowcaptionskip}{10pt}
\caption{The global value and sorting result obtained by \cite{Kadzinski20ijar}'s method}
\label{table:IJAR_V}
\begin{tabular}{cccccccccccc}
\toprule
Firm & $V(a_i)$ & $f_i$  & Firm & $V(a_i)$ & $f_i$ & Firm & $V(a_i)$ & $f_i$ & Firm & $V(a_i)$ & $f_i$    \\ \midrule
$a_1$ & 0.1390 & $C_4$ & $a_6$ & 0.1130 & $C_3$ & $a_{11}$ & 0.1380 & $C_4$ & $a_{16}$ & 0.0500 & $C_1$ \\
$a_2$ & 0.1320 & $C_4$ & $a_7$ & 0.9910 & $C_4$ & $a_{12}$ & 0.0970 & $C_1$ & $a_{17}$ & 0.1310 & $C_3$ \\
$a_3$ & 0.0980 & $C_2$ & $a_8$ & 0.1050 & $C_3$ & $a_{13}$ & 0.1360 & $C_4$ & $a_{18}$ & 0.0730 & $C_1$ \\
$a_4$ & 0.1450 & $C_4$ & $a_9$ & 0.0990 & $C_3$ & $a_{14}$ & 0.1070 & $C_3$ & $a_{19}$ & 0.1210 & $C_3$ \\
$a_5$ & 0.1380 & $C_4$ & $a_{10}$ & 0.0980 & $C_2$ & $a_{15}$ & 0.0490 & $C_1$ & $a_{20}$ & 0.1060 & $C_3$ \\
\bottomrule
\end{tabular}
\end{table}

Comparing Figs. \ref{fig:numericla_v1}, \ref{fig:numericla_v2} and \ref{fig:IJAR_results} reveals significant differences in the marginal value functions for the criteria $g_1$ and $g_2$ between \cite{Kadzinski20ijar}'s method and  the proposed approaches. Interestingly, the marginal value function for the criterion $g_3$ inferred using \cite{Kadzinski20ijar}'s method is similar to that obtained by Approach 1. Additionally, the criterion weights are substantially  different, with \cite{Kadzinski20ijar}'s method assigning weights of 0.05, 0.059, and 0.891, indicating very low weights for $g_1$ and $g_2$. These variations also lead to noticeable differences in the sorting result for non-reference alternatives, as shown in Tables \ref{table:T_V}, \ref{table:T_V2} and \ref{table:IJAR_V}.

The observed differences are primarily due to the following factors. First, the proposed approaches use piecewise linear functions to approximate the true marginal value function, whereas \cite{Kadzinski20ijar} directly calculates marginal values across all performance levels. Secondly, the proposed approaches model non-monotonic criteria using transformation functions, while \cite{Kadzinski20ijar} employs binary variables for this purpose. Lastly, the proposed approaches considers both model complexity and discriminative power, utilizing  lexicographic optimization to derive a representative sorting model, whereas \cite{Kadzinski20ijar} focuses on minimizing the number of monotonic changes to achieve this goal. These  differences contribute to the divergent results observed  between the proposed approaches and \cite{Kadzinski20ijar}'s method.

In addition, we assess the robustness of these methods. Applying the robustness analysis method detailed in Section \ref{sec:4.6}, the possible assignments for all non-reference alternatives under both methods were found to be $C^P_i=\{C_1,C_2,C_3,C_4\}$,  $\forall a_i\in \{a_1,a_4,a_5,a_6,a_7,a_8,a_{11},a_{13},a_{14},a_{15},a_{16},a_{18},a_{19},a_{20}\}$. Consequently, according to Eq. \eqref{eq:APA}, all these methods yield a value of 0 on the $APA$ metric.

\section{Experimental analysis}\label{sec:6}
In this section, we first compare the two proposed approaches (Approach 1 and Approach 2) with some non-monotonic criteria modeling approaches and analyze the impact of certain parameters on these approaches by simulation experiments. Subsequently, we assess the robustness of the proposed approaches.

\subsection{Comparisons with some non-monotonic criteria modeling approaches}\label{sec:6.1}

We first introduce an accuracy metric to evaluate the performance of these approaches, which quantifies how many non-reference alternatives are correctly assigned.
\begin{definition}
Let $F=(f_1,f_2,\ldots,f_{n_2})^{\rm T}$ and $\overline{F}=(\overline{f_1},\overline{f_2},\ldots,\overline{f_{n_2}})^{\rm T}$ be
the real sorting result and the sorting result obtained by a sorting model for $n_2$ non-reference alternatives, respectively, where $f_i$ and $\overline{f_i}$ are the real and inferred category assignment of the $i$-th non-reference alternative, respectively. Let $y_i$ be a binary variable such that
\begin{equation}\label{eq:accu1}
y_i=\left\{
\begin{aligned}
&1, \ \text{if} \ f_i = \overline{f_i},\\
&0, \ \text{otherwise},
\end{aligned}
\right.
\end{equation}
then the accuracy metric is computed by
\begin{equation}\label{eq:accu2}
accuracy=\frac{\sum\nolimits_{i=1}^{n_2} y_i}{n_2}.
\end{equation}
\end{definition}

Evidently, a higher accuracy indicates a better performance of the sorting model.

Subsequently, we present an algorithm (see Algorithm \ref{alg:2} in the supplemental file) for generating simulation datasets to be used in the subsequent experiments, ensuring that the assignment example preference information provided by the decision maker is consistent.

Following this, we perform simulation experiments to compare the proposed approaches with some non-monotonic criteria modeling approaches, aiming to demonstrate the effectiveness of the proposed lexicographic optimization-based approaches. We also analyze the influence of various parameters, namely the number of alternatives, criteria, categories and the proportion of reference alternatives among all alternatives, on these approaches.

The descriptions of the non-monotonic criteria modeling approaches compared in this subsection are outlined below.

$\bullet$ \textbf{UTADIS-NONM}: Although the proposed approaches include lexicographic optimization models, it is worth highlighting that a sorting model can also be derived by exclusively solving the model \eqref{m:con_check}. This can be considered as a variant of the UTADIS method, referred to as UTADIS-NONM in this paper.

$\bullet$ \textbf{LFP-NONM}: This approach employs the proposed transformation functions to model non-monotonic criteria and utilizes linear fractional programming (LFP) to derive the representative sorting model \citep{Ghaderi17ejor}, which can be formulated as
\begin{equation}\tag{M-10}\label{m:LFP}
\begin{aligned}
&\min\ \frac{\sum\nolimits_{j=1}^m\sum\nolimits_{l=2}^{s_j}\gamma_{lj}}{\varepsilon}\\
&\begin{aligned}
\rm{s.t.} &\ E^{A^{R}}, \ E^{Bound},\ E^{Sort}\\
&\gamma_{lj} = \left|\frac{v_j(\beta_j^l)-v_j(\beta_j^{l-1})}{\beta_j^l-\beta_j^{l-1}} - \frac{v_j(\beta_j^{l+1})-v_j(\beta_j^l)}{\beta_j^{l+1}-\beta_j^l}\right|, l=2,\ldots,s_j, j\in M.
\end{aligned}
\end{aligned}
\end{equation}

Subsequently, we present an algorithm (see Algorithm \ref{alg:3} in the supplemental file) designed to compare Approaches 1 and 2 with UTADIS-NONM and LFP-NONM. On this basis, we conduct  four distinct simulation experiments evaluate the performance of these four approaches under varying parameter settings.

\textbf{Simulation experiment \RNum{1}\ :}
In this experiment, we examine  the impact of the number of alternatives on the performance of the four approaches. We consider different values for the number of alternatives, specifically $n\in\{200,400,600\}$, while keeping $m=6$, $q=4$, and $r=0.8$ constant. Additionally, we assume that all criteria share the same number of subintervals, denoted as $s_j\in \{1,2,3,4\}$, $j\in M$. For each fixed value of $n$ and $s_j$, we execute Algorithm \ref{alg:2} 10 times, generating 10 datasets. Subsequently, for each dataset, we run Algorithm \ref{alg:3} 100 times and compute the average values of the accuracy metric. Finally, we calculate the mean and standard deviation of the accuracy metric across the ten datasets for each of the four approaches. The summarized results are presented in Table \ref{table:simu1}. It is worth noting that Approach 1 and LFP-NONM are not applicable when $s_j=1$, as they do not allow for any slope change in this case. Therefore, the results for Approach 1 and LFP-NONM when $s_j=1$ are represented as ``-'' in the following experiments.
\begin{table}[htbp]
\centering
\setlength{\abovecaptionskip}{0pt}
\setlength{\belowcaptionskip}{10pt}
\caption{The mean and standard deviation of the accuracy metric for the four approaches with different values of $n$ in Simulation experiment \RNum{1}}
\label{table:simu1}
\begin{tabular}{cccccc}
\toprule
$n$ & Approach & $s_j=1$ & $s_j=2$ & $s_j=3$ & $s_j=4$    \\ \midrule
\multirow{3}*{200} & Approach 1 & - & $0.9433\pm0.0081$ & $0.9250\pm0.0105$ & $0.9021\pm0.0112$\\
& Approach 2 & \bm{$0.9649\pm0.0073$} & $0.9435\pm0.0083$ & \bm{$0.9293\pm0.0089$} & \bm{$0.9036\pm0.0114$} \\
& LFP-NONM & - & \bm{$0.9445\pm0.0081$} & $0.9247\pm0.0090$ & $0.9017\pm0.0117$ \\
& UTADIS-NONM  & $0.9584\pm0.0029$ & $0.9252\pm0.0044$ & $0.9047\pm0.0086$ & $0.8614\pm0.0067$ \\

\multirow{3}*{400} & Approach 1 & - & $0.9734\pm0.0040$ & $0.9629\pm0.0059$ & $0.9473\pm 0.0037$ \\
& Approach 2 & \bm{$0.9841\pm0.0034$} & $0.9734\pm0.0039$ & \bm{$0.9634\pm0.0067$} & \bm{$0.9483\pm0.0051$} \\
& LFP-NONM & - & \bm{$0.9736\pm0.0039$} & $0.9631\pm 0.0060$ & $0.9479\pm 0.0042$\\
& UTADIS-NONM  & $0.9808\pm0.0025$ & $0.9650\pm0.0039$ & $0.9512\pm0.0047$ & $0.9314\pm0.0062$\\

\multirow{3}*{600} & Approach 1 & - & $0.9812\pm0.0027$ & $0.9735\pm0.0035$ & $0.9670\pm0.0020$ \\
& Approach 2 & \bm{$0.9894\pm0.0015$} & \bm{$0.9816\pm0.0031$} & \bm{$0.9738\pm0.0033$} & \bm{$0.9674\pm0.0021$} \\
& LFP-NONM & - & $0.9813\pm 0.0031$ & \bm{$0.9738\pm 0.0033$} & $0.9675\pm 0.0017$ \\
& UTADIS-NONM  & $0.9870\pm0.0013$ & $0.9744\pm0.0017$ & $0.9669\pm0.0036$ & $0.9527\pm0.0053$\\
\bottomrule
\end{tabular}
\end{table}

\textbf{Simulation experiment \RNum{2}\ :}
In this experiment, we examine how the number of criteria impacts the performance of the four approaches. To do so, we set $n$ to 200 and vary $m$ in \{6, 7, 8\}, keeping $q$ at 4, $r$ at 0.8, and $s_j$ in $\{1, 2, 3, 4\}$, $j=1,2,\ldots,m$. For each combination of $m$ and $s_j$, we execute Algorithm \ref{alg:2} 10 times to generate initial datasets. Afterwards, for each dataset, we run Algorithms \ref{alg:3} 100 times and calculate the average accuracy metric. Finally, we calculate the mean and standard deviation of the accuracy metric across the ten datasets for all four approaches. The results are tabulated in Table \ref{table:simu2}.

\begin{table}[htbp]
\centering
\setlength{\abovecaptionskip}{0pt}
\setlength{\belowcaptionskip}{10pt}
\caption{The mean and standard deviation of the accuracy metric for the four approaches with different values of $m$ in Simulation experiment \RNum{2}}
\label{table:simu2}
\begin{tabular}{cccccc}
\toprule
$m$ & Approach & $s_j=1$ & $s_j=2$ & $s_j=3$ & $s_j=4$    \\ \midrule
\multirow{3}*{6} & Approach 1 & - & $0.9433\pm0.0081$ & $0.9250\pm0.0105$ & $0.9021\pm0.0112$\\
& Approach 2 & \bm{$0.9649\pm0.0073$} & $0.9435\pm0.0083$ & \bm{$0.9293\pm0.0089$} & \bm{$0.9036\pm0.0114$} \\
& LFP-NONM & - & \bm{$0.9445\pm0.0081$} & $0.9247\pm0.0090$ & $0.9017\pm0.0117$ \\
& UTADIS-NONM  & $0.9584\pm0.0029$ & $0.9252\pm0.0044$ & $0.9047\pm0.0086$ & $0.8614\pm0.0067$ \\

\multirow{3}*{7} & Approach 1 & - & $0.9396\pm0.0087$ & $0.9121\pm0.0108$ & $0.8897\pm0.0119$ \\
& Approach 2 & \bm{$0.9614\pm0.0067$} & $0.9385\pm0.0073$ & \bm{$0.9139\pm0.0094$} & \bm{$0.8916\pm0.0125$} \\
& LFP-NONM & - & \bm{$0.9400\pm0.0083$} & $0.9132\pm0.0113$ & $0.8902\pm0.0112$ \\
& UTADIS-NONM  & $0.9525\pm0.0051$ & $0.9163\pm0.0077$ & $0.8806\pm0.0063$ & $0.8437\pm0.0058$\\

\multirow{3}*{8} & Approach 1 & - & $0.9294\pm0.0065$ & $0.9076\pm0.0110$ & $0.8707\pm0.0149$ \\
& Approach 2 & \bm{$0.9576\pm0.0072$} & \bm{$0.9314\pm0.0055$} & \bm{$0.9125\pm0.0107$} & \bm{$0.8776\pm0.0108$} \\
& LFP-NONM & - & $0.9303\pm 0.0057$ & $0.9086\pm 0.0105$ & $0.8716\pm 0.0143$ \\
& UTADIS-NONM  & $0.9439\pm0.0046$ & $0.9130\pm0.0109$ & $0.8713\pm0.0078$ & $0.8354\pm0.0092$\\
\bottomrule
\end{tabular}
\end{table}

\textbf{Simulation experiment \RNum{3}\ :}\
In this experiment, we explore the impact of varying the number of categories on the performance of different approaches.
To do so, let $n=200$, $m=6$, $q\in \{2,3,4\}$, $r=0.8$, $s_j\in \{1,2,3,4\}$, $j\in M$. For each combination of $q$ and $s_j$, we execute Algorithm \ref{alg:2} 10 times to generate 10 datasets. Subsequently, for each dataset, we run Algorithms \ref{alg:3} 100 times and calculate the average accuracy metric. Finally, we calculate the mean and standard deviation of the accuracy metric across the ten datasets for each of the four approaches. The results are displayed in Table \ref{table:simu3}.
\begin{table}[htbp]
\centering
\setlength{\abovecaptionskip}{0pt}
\setlength{\belowcaptionskip}{10pt}
\caption{The mean and standard deviation of the accuracy metric for the four approaches with different values of $q$ in Simulation experiment \RNum{3}}
\label{table:simu3}
\begin{tabular}{cccccc}
\toprule
$q$ & Approach & $s_j=1$ & $s_j=2$ & $s_j=3$ & $s_j=4$    \\ \midrule
\multirow{3}*{2} & Approach 1 & - & $0.9528\pm0.0066$ & \bm{$0.9272\pm0.0063$} & $0.9078\pm0.0154$ \\
& Approach 2 & \bm{$0.9767\pm0.0037$} & \bm{$0.9551\pm0.0069$} & $0.9260\pm0.0083$ & \bm{$0.9108\pm0.0123$} \\
& LFP-NONM & - & $0.9528\pm0.0066$ & \bm{$0.9272\pm0.0063$} & $0.9078\pm0.0154$\\
& UTADIS-NONM & $0.9708\pm0.0032$ & $0.9367\pm0.0062$ & $0.9096\pm0.0052$ & $0.8838\pm0.0092$\\

\multirow{3}*{3} & Approach 1 & - & $0.9503\pm0.0079$ & $0.9266\pm0.0068$ & $0.9061\pm0.0137$ \\
& Approach 2 & \bm{$0.9740\pm0.0047$} & \bm{$0.9517\pm0.0070$} & \bm{$0.9321\pm0.0077$} & \bm{$0.9107\pm0.0125$} \\
& LFP-NONM & - & $0.9501\pm 0.0079$ & $0.9267\pm0.0067$ & $0.9061\pm0.0137$\\
& UTADIS-NONM & $0.9644\pm0.0036$ & $0.9294\pm0.0049$ & $0.9038\pm0.0087$ & $0.8688\pm0.0093$\\

\multirow{3}*{4} & Approach 1 & - & $0.9433\pm0.0081$ & $0.9250\pm0.0105$ & $0.9021\pm0.0112$\\
& Approach 2 & \bm{$0.9649\pm0.0073$} & $0.9435\pm0.0083$ & \bm{$0.9293\pm0.0089$} & \bm{$0.9036\pm0.0114$} \\
& LFP-NONM & - & \bm{$0.9445\pm0.0081$} & $0.9247\pm0.0090$ & $0.9017\pm0.0117$ \\
& UTADIS-NONM  & $0.9584\pm0.0029$ & $0.9252\pm0.0044$ & $0.9047\pm0.0086$ & $0.8614\pm0.0067$ \\

\bottomrule
\end{tabular}
\end{table}

\textbf{Simulation experiment \RNum{4}:}
In this experiment, we analyze the influence of the proportion of reference alternatives among all alternatives on the performance of the four approaches. To this end, we vary $r$ in \{0.2, 0.5, 0.8\}, and set $n=200$, $m=6$, $q=4$, $s_j\in \{1,2,3,4\}$, $j\in M$. In particular, for each fixed value of $s_j$, $j\in M$, we utilize the same ten datasets generated in the above simulations when $n=200$, $m=6$, $q=4$. Following this, for each dataset, Algorithms \ref{alg:3} is run 100 times and the average accuracy metric is computed. On this basis, we calculate the mean and standard deviation of the accuracy metric across the ten datasets for each of the four approaches. The results are displayed in Table \ref{table:simu4}.
\begin{table}[htbp]
\centering
\setlength{\abovecaptionskip}{0pt}
\setlength{\belowcaptionskip}{10pt}
\caption{The mean and standard deviation of the accuracy metric for the four approaches with different values of $r$ in Simulation experiment \RNum{4}}
\label{table:simu4}
\begin{tabular}{cccccc}
\toprule
$r$ & Approach & $s_j=1$ & $s_j=2$ & $s_j=3$ & $s_j=4$    \\ \midrule
\multirow{3}*{0.2} & Approach 1 & - & $0.7886\pm0.0138$ & $0.7379\pm0.0202$ & $0.7010\pm0.0271$  \\
& Approach 2 & \bm{$0.8716\pm0.0089$} & \bm{$0.8066\pm0.0164$} & \bm{$0.7621\pm0.0093$} & \bm{$0.7208\pm0.0151$} \\
& LFP-NONM & - & $0.7866\pm0.0121$ & $0.7372\pm0.0214$ & $0.6986\pm0.0264$\\
& UTADIS-NONM & $0.8274\pm0.0120$ & $0.7316\pm0.0111$ & $0.6507\pm0.0157$ & $0.6195\pm0.0187$ \\

\multirow{3}*{0.5} & Approach 1 & - & $0.9115\pm0.0043$ & $0.8830\pm0.0071$ & $0.8505\pm0.0105$ \\
& Approach 2 &  \bm{$0.9434\pm0.0054$} & \bm{$0.9138\pm0.0053$} & \bm{$0.8870\pm0.0056$} & \bm{$0.8549\pm0.0079$}\\
& LFP-NONM & - & $0.9124\pm0.0047$ & $0.8826\pm0.0072$ & $0.8500\pm0.0091$ \\
& UTADIS-NONM & $0.9297\pm0.0059$ & $0.8814\pm0.0062$ & $0.8375\pm0.0040$ & $0.7889\pm0.00100$\\

\multirow{3}*{0.8} & Approach 1 & - & $0.9433\pm0.0081$ & $0.9250\pm0.0105$ & $0.9021\pm0.0112$\\
& Approach 2 & \bm{$0.9649\pm0.0073$} & $0.9435\pm0.0083$ & \bm{$0.9293\pm0.0089$} & \bm{$0.9036\pm0.0114$} \\
& LFP-NONM & - & \bm{$0.9445\pm0.0081$} & $0.9247\pm0.0090$ & $0.9017\pm0.0117$ \\
& UTADIS-NONM  & $0.9584\pm0.0029$ & $0.9252\pm0.0044$ & $0.9047\pm0.0086$ & $0.8614\pm0.0067$ \\

\bottomrule
\end{tabular}
\end{table}

From Tables \ref{table:simu1}~-~\ref{table:simu4}, we draw the following observations:

(1) Approaches 1~-~2 and LFP-NONM consistently outperform UTADIS-NONM in all scenarios examined in the simulation experiments. Furthermore, as the number of subintervals for each criterion increases, Approaches 1~-~2 and LFP-NONM exhibit superior performance compared to UTADIS-NONM. This underscores the potential of considering both model complexity and discriminative power to derive the representative sorting model for MCS problems with non-monotonic criteria, ultimately leading to improved performance.

(2) Among the three approaches, i.e., Approaches 1~-~2 and LFP-NONM, Approach 2 demonstrates the best performance in most cases, indicating that prioritizing model discriminative power enhances the performance. While LFP-NONM excels in some certain specific scenarios, its performance is generally inferior to Approach 2 and comparable to Approach 1. The proposed lexicographic optimization-based approaches offer flexibility by allowing different priorities to be set, making them adaptable to the specific needs of various MCS problems. This adaptability makes the proposed approaches more versatile and applicable to a broader range of real-world MCS problems.

(3) The accuracy of all four approaches improves with an increasing number of alternatives or a higher proportion of reference alternatives among all alternatives. These improvements primarily arise from the greater availability of assignment example preference information, which enhances model training as the number of alternatives or the proportion of reference alternatives grows. However, as the number of criteria and categories increases, the performance of all four approaches deteriorates. This decline is likely due to the increased complexity of datasets generated with more criteria and categories, which consequently reduces the model's ability to fit these datasets effectively.

(4) Across  all simulation experiments, a slight decrease in accuracy is observed for all approaches as the number of subintervals within each criterion increases. This decrease can be attributed to the growing complexity of the marginal value functions with more subintervals, which leads to a reduced generalization ability of the models.

To further confirm the superiority of the proposed approaches, we perform one-tailed paired $t$-tests to assess the differences between Approach 1 and UTADIS-NONM (Comparison 1),  Approach 2 and UTADIS-NONM (Comparison 2), as well as between LFP-NONM and UTADIS-NONM (Comparison 3) from a statistical perspective. For each $t$-test, the null hypothesis is denoted as $H_0$: $u_{\rm{method1}}\le u_{\rm{method2}}$, and the alternative hypothesis as $H_1$: $u_{\rm{method1}} > u_{\rm{method2}}$, where $u_{\rm{method1}}$ and $u_{\rm{method2}}$ represent the means of a specific measure for method 1 and method 2, respectively. Given a significance level $\alpha=0.05$, we conduct $t$-tests under each parameter setting based on Simulation experiments \RNum{1} to \RNum{4}, and the results are presented in Table \ref{table:ttest1} of the supplemental file.

Observing Table \ref{table:ttest1}, it is apparent that, under all parameter settings, the null hypothesis $H_0$ is rejected, which indicates that there are statistically significant differences between Approach 1, Approach 2 and LFP-NONM compared to UTADIS-NONM.

\subsection{Comparisons with the non-monotonic criteria modeling method in \cite{Kadzinski20ijar}}\label{sec:6.2}

To further validate the effectiveness of the proposed approaches, we conduct simulation experiments to compare Approach 1, Approach 2 with \cite{Kadzinski20ijar}'s non-monotonic criteria modeling method. Due to the significant increase in model complexity associated with \cite{Kadzinski20ijar}'s method, which involves numerous binary variables, we limit the number of alternatives to 50 for a fair comparison.

We then examine the impact of varying the number of criteria, categories, and the proportion of reference alternatives among all alternatives on the performance of the three approaches. The detailed parameter settings for each simulation experiment are shown in Table \ref{table:parameter_s}, and the experimental setting is consistent with that of Simulation experiments \RNum{1}~-~\RNum{5}. The mean and standard deviation of the accuracy metric for the three approaches under different values of $m$, $q$, $r$ in Simulation experiment \RNum{5}~-~\RNum{7} are displayed in Tables \ref{table:simu5}~-~\ref{table:simu7}. Moreover, the results of $t$-tests in Simulation experiments \RNum{5}~-~\RNum{7} are provided in Table \ref{table:ttest2} of the supplemental file.

\begin{table}[htbp]
\centering
\setlength{\abovecaptionskip}{0pt}
\setlength{\belowcaptionskip}{10pt}
\caption{The detailed parameter settings of each simulation experiment}
\label{table:parameter_s}
\begin{tabular}{lccccc}
\toprule
Simulation experiment & $n$ & $m$  & $q$ & $r$ & $s_j$    \\ \midrule
Simulation experiment \RNum{5} & 50 & \{6,7,8\} & 4 & 0.8 & \{1,2,3,4\}  \\
Simulation experiment \RNum{6} & 50 & 6 & \{2,3,4\} & 0.8 & \{1,2,3,4\}  \\
Simulation experiment \RNum{7} & 50 & 6 & 4 & \{0.2,0.5,0.8\} & \{1,2,3,4\} \\
\bottomrule
\end{tabular}
\end{table}

\begin{table}[htbp]
\centering
\setlength{\abovecaptionskip}{0pt}
\setlength{\belowcaptionskip}{10pt}
\caption{The mean and standard deviation of the accuracy metric for the three approaches with different values of $m$ in Simulation experiment \RNum{5}}
\label{table:simu5}
\resizebox{\textwidth}{!}{
\begin{tabular}{cccccc}
\toprule
$m$ & Approach & $s_j=1$ & $s_j=2$ & $s_j=3$ & $s_j=4$    \\ \midrule
\multirow{3}*{6} & Approach 1 & - & $0.7851\pm0.0188$ & $0.7279\pm0.0155$ & \bm{$0.7176\pm0.0457$}\\
& Approach 2 & \bm{$0.8791\pm0.0186$} & \bm{$0.8061\pm0.0175$} & \bm{$0.7455\pm0.0322$} & $0.7153\pm0.0459$ \\
& \cite{Kadzinski20ijar}'s method & $0.4752\pm 0.0702$ & $0.4614\pm0.0552$ & $0.4727\pm0.0492$ & $0.4712\pm0.0804$ \\

\multirow{3}*{7} & Approach 1 & - & $0.7618\pm0.0254$ & $0.7334\pm0.0428$ & $0.6583\pm0.0646$ \\
& Approach 2 & \bm{$0.8704\pm0.0366$} & \bm{$0.7919\pm0.0229$} & \bm{$0.7393\pm0.0292$} & \bm{$0.6653\pm0.0528$} \\
& \cite{Kadzinski20ijar}'s method & $0.5513\pm0.0981$ & $0.4906\pm0.0796$ & $0.5004\pm0.0509$ & $0.4656\pm0.0806$ \\

\multirow{3}*{8} & Approach 1 & - & $0.7497\pm0.0327$ & $0.7157\pm0.0560$ & \bm{$0.6557\pm0.0722$} \\
& Approach 2 & \bm{$0.8614\pm0.0198$} & \bm{$0.7621\pm0.0420$} & \bm{$0.7216\pm0.0382$} & $0.6518\pm0.0282$ \\
& \cite{Kadzinski20ijar}'s method & $0.5358\pm0.0481$ & $0.4718\pm 0.0496$ & $0.5528\pm 0.0816$ & $0.5046\pm 0.0593$ \\
\bottomrule
\end{tabular}}
\vspace{2em}
\centering
\setlength{\abovecaptionskip}{0pt}
\setlength{\belowcaptionskip}{10pt}
\caption{The mean and standard deviation of the accuracy metric for the three approaches with different values of $q$ in Simulation experiment \RNum{6}}
\label{table:simu6}
\resizebox{\textwidth}{!}{
\begin{tabular}{cccccc}
\toprule
$q$ & Approach & $s_j=1$ & $s_j=2$ & $s_j=3$ & $s_j=4$    \\ \midrule
\multirow{3}*{2} & Approach 1 & - & $0.8304\pm0.0346$ & \bm{$0.7740\pm0.0496$} & $0.7491\pm0.0493$\\
& Approach 2 & \bm{$0.9042\pm0.0115$} & \bm{$0.8339\pm0.0265$} & $0.7679\pm0.0465$ & \bm{$0.7634\pm0.0392$} \\
& \cite{Kadzinski20ijar}'s method & $0.6580\pm 0.0420$ & $0.6473\pm0.0512$ & $0.6220\pm0.0727$ & $0.6202\pm0.0468$ \\

\multirow{3}*{3} & Approach 1 & - & $0.8167\pm0.0245$ & $0.7495\pm0.0559$ & $0.7211\pm0.0415$ \\
& Approach 2 & \bm{$0.8790\pm0.0139$} & \bm{$0.8219\pm0.0178$} & \bm{$0.7911\pm0.0365$} & \bm{$0.7392\pm0.0444$} \\
& \cite{Kadzinski20ijar}'s method & $0.5512\pm0.0586$ & $0.5255\pm0.0562$ & $0.5216\pm0.0543$ & $0.5042\pm0.0740$ \\

\multirow{3}*{4} & Approach 1 & - & $0.7851\pm0.0188$ & $0.7279\pm0.0155$ & \bm{$0.7176\pm0.0457$}\\
& Approach 2 & \bm{$0.8791\pm0.0186$} & \bm{$0.8061\pm0.0175$} & \bm{$0.7455\pm0.0322$} & $0.7153\pm0.0459$ \\
& \cite{Kadzinski20ijar}'s method & $0.4752\pm 0.0702$ & $0.4614\pm0.0552$ & $0.4727\pm0.0492$ & $0.4712\pm0.0804$ \\

\bottomrule
\end{tabular}}
\vspace{2em}
\centering
\setlength{\abovecaptionskip}{0pt}
\setlength{\belowcaptionskip}{10pt}
\caption{The mean and standard deviation of the accuracy metric for the three approaches with different values of $r$ in Simulation experiment \RNum{7}}
\label{table:simu7}
\resizebox{\textwidth}{!}{
\begin{tabular}{cccccc}
\toprule
$r$ & Approach & $s_j=1$ & $s_j=2$ & $s_j=3$ & $s_j=4$    \\ \midrule
\multirow{3}*{0.2} & Approach 1 & - & \bm{$0.5545\pm0.0797$} & $0.5207\pm0.0492$ & $0.4816\pm0.0655$\\
& Approach 2 & \bm{$0.6229\pm0.0446$} & $0.5513\pm0.0331$ & \bm{$0.5320\pm0.0430$} & \bm{$0.5061\pm0.0564$} \\
& \cite{Kadzinski20ijar}'s method & $0.3742\pm 0.0393$ & $0.3661\pm0.0301$ & $0.3779\pm0.0367$ & $0.3599\pm0.0573$ \\

\multirow{3}*{0.5} & Approach 1 & - & $0.6908\pm0.0184$ & \bm{$0.6525\pm0.0298$} & $0.6352\pm0.0910$ \\
& Approach 2 & \bm{$0.7976\pm0.0169$} & \bm{$0.6930\pm0.0377$} & $0.6464\pm0.0390$ & \bm{$0.6416\pm0.0455$} \\
& \cite{Kadzinski20ijar}'s method & $0.4276\pm0.0874$ & $0.3696\pm0.0714$ & $0.4374\pm0.0679$ & $0.4440\pm0.0721$ \\

\multirow{3}*{0.8} & Approach 1 & - & $0.7851\pm0.0188$ & $0.7279\pm0.0155$ & \bm{$0.7176\pm0.0457$}\\
& Approach 2 & \bm{$0.8791\pm0.0186$} & \bm{$0.8061\pm0.0175$} & \bm{$0.7455\pm0.0322$} & $0.7153\pm0.0459$ \\
& \cite{Kadzinski20ijar}'s method & $0.4752\pm 0.0702$ & $0.4614\pm0.0552$ & $0.4727\pm0.0492$ & $0.4712\pm0.0804$ \\

\bottomrule
\end{tabular}}
\end{table}

From Tables \ref{table:simu5}~-~\ref{table:simu7} and Table \ref{table:ttest2}, it is evident that across all parameter settings, Approach 1 and Approach 2 consistently  outperform \cite{Kadzinski20ijar}'s method, with statistically significant differences observed. These results further demonstrate the effectiveness of the proposed approaches.

Additionally, it is noteworthy that the category thresholds are set as $b_0=0$, $b_q=1$, $b_h=1/q$, $h=1,\ldots,q-1$ in Algorithm \ref{alg:2}, and the reference alternatives are randomly chosen in Algorithm \ref{alg:3}, potentially causing category imbalance in Simulation experiments \RNum{1}~-~\RNum{7}. To address this issue, we conduct additional simulation experiments with a balanced distribution of alternatives across all categories (Simulation experiments \RNum{1}$'$~-~\RNum{7}$'$). The simulation procedure and outcomes are presented in \ref{appendix:c} of the supplemental file. The results from these additional experiments corroborate the findings from Simulation experiments \RNum{1}~-~\RNum{7}, providing further evidence for the effectiveness of the proposed approaches.

\subsection{Robustness analysis of the proposed approaches}
In this subsection, we explore the robustness of the proposed approaches by analyzing the effects of varying the number of criteria, categories, and the proportion of reference alternatives among all alternatives. The parameter settings used are consistent with those outlined in Table \ref{table:parameter_s}. For each parameter setting, we execute  Algorithm \ref{alg:2} 10 times, generating  10 datasets. For each dataset, Algorithm \ref{alg:4} is run 100 times, and the average value of the $APA$ metric is calculated. We then compute the mean and standard deviation of the $APA$ metric across the 10 datasets. The results are shown in Fig. \ref{fig:APA_imbalance}.

\begin{algorithm}\label{alg:4}
~\\
{\bf Input:} The decision matrix $X$, the sorting result for alternatives $F$, the number of subintervals for each criterion $s_j$, $j\in M$ and the proportion of reference alternatives among all alternatives $r$.\\
{\bf Output:} The $APA$ value of the proposed approaches.
\begin{enumerate}[\bf Step 1:]
  \item Randomly partition the set of alternatives $A=\{a_1,a_2,\ldots,a_n\}$ into two subsets: the set of reference alternatives $A^R$ and the set of non-reference alternatives $A^N$, where $|A^R|=[n\cdot r]$ and $|A^N|=n-|A^R|$ are the number of alternatives in $A^R$ and $A^N$, respectively.
  \item Determine the sorting result for reference alternatives in $A^R$ based on $F$, denoted as $F_1$.
  \item For each non-reference alternative $a_i\in A^N$ and each category $C_h\in C$, solve the model \eqref{m:robust}  to obtain the optimal objective function value, denoted as $\varepsilon^*_h$.
  \item Determine  the possible assignment for each non-reference alternative, i.e., $C^P_i=\{C_h| \varepsilon^*_h>0, h\in Q\}$, $\forall a_{i}\in A^N$.
  \item Calculate the $APA$ metric using Eq. \eqref{eq:APA}
  \item Output the $APA$ value.
\end{enumerate}
\end{algorithm}

Similar to Sections \ref{sec:6.1}~-~\ref{sec:6.2}, we also investigate the impact of various parameters on the $APA$ metric for the proposed approaches, assuming that alternatives are distributed in a balanced manner across all categories. The results of this analysis are displayed  in Fig. \ref{fig:APA_balance}.

%\begin{figure}[htbp]
%\centering
%\includegraphics[scale=0.4]{APA_balance.png}
%\caption{The mean and  standard deviation of the $APA$ metric for the proposed approaches under different parameter settings under the situation that alternatives are balanced distributed across all categories}
%\label{fig:APA_balance}
%\end{figure}

From Figs. \ref{fig:APA_imbalance}~-~\ref{fig:APA_balance}, we can derive the following observations. On the one hand, as the number of subintervals for each criterion ($s_j$) and the number of criteria ($m$) increase, the robustness of the proposed approaches  tends to decrease. This is because the flexibility of the compatible marginal value functions increases with these parameters, resulting in a larger set of possible assignments for non-reference alternatives, which leads to a lower $APA$ value. On the other hand, the proposed approaches demonstrate greater robustness as the number of categories ($q$) and the proportion of reference alternatives among all alternatives ($r$) increase.  This suggests that higher values of $q$ and $r$ make the set of compatible marginal value functions and category thresholds more constrained, thereby enhancing the robustness of the proposed approaches.

\begin{figure}[H]
\centering
\includegraphics[scale=0.4]{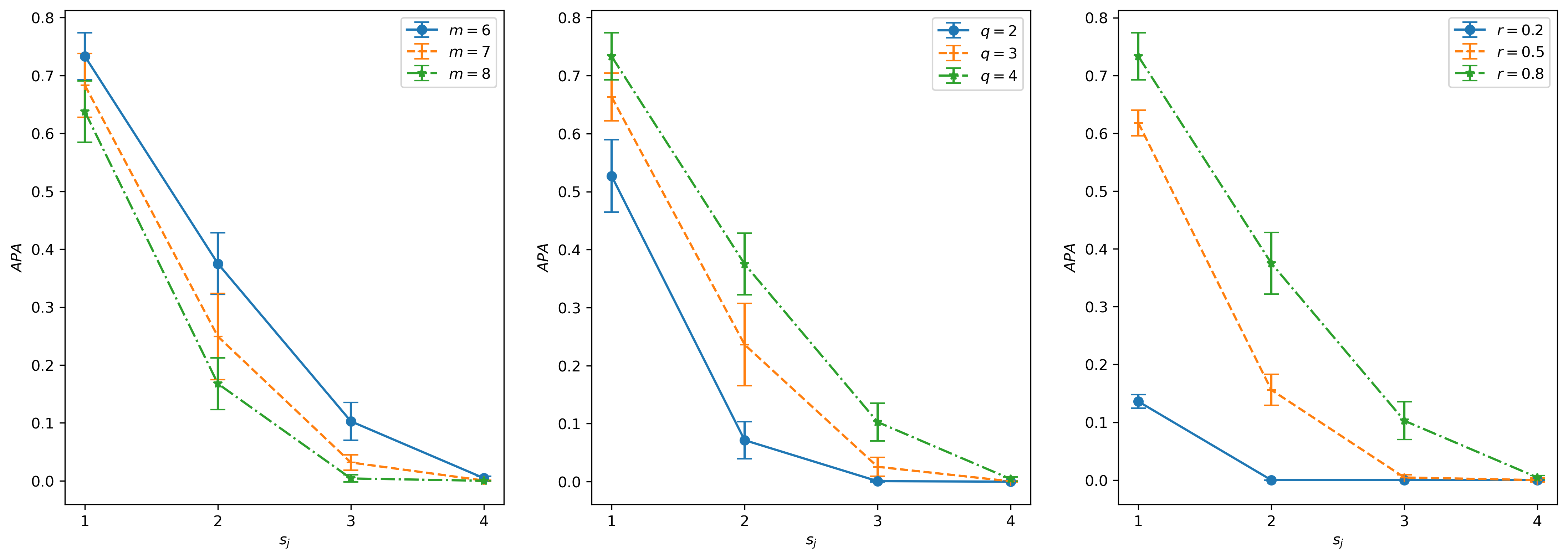}
\caption{The mean and  standard deviation of the $APA$ metric for the proposed approaches under different parameter settings}\label{fig:APA_imbalance}
\vspace{0.5em}
\includegraphics[scale=0.4]{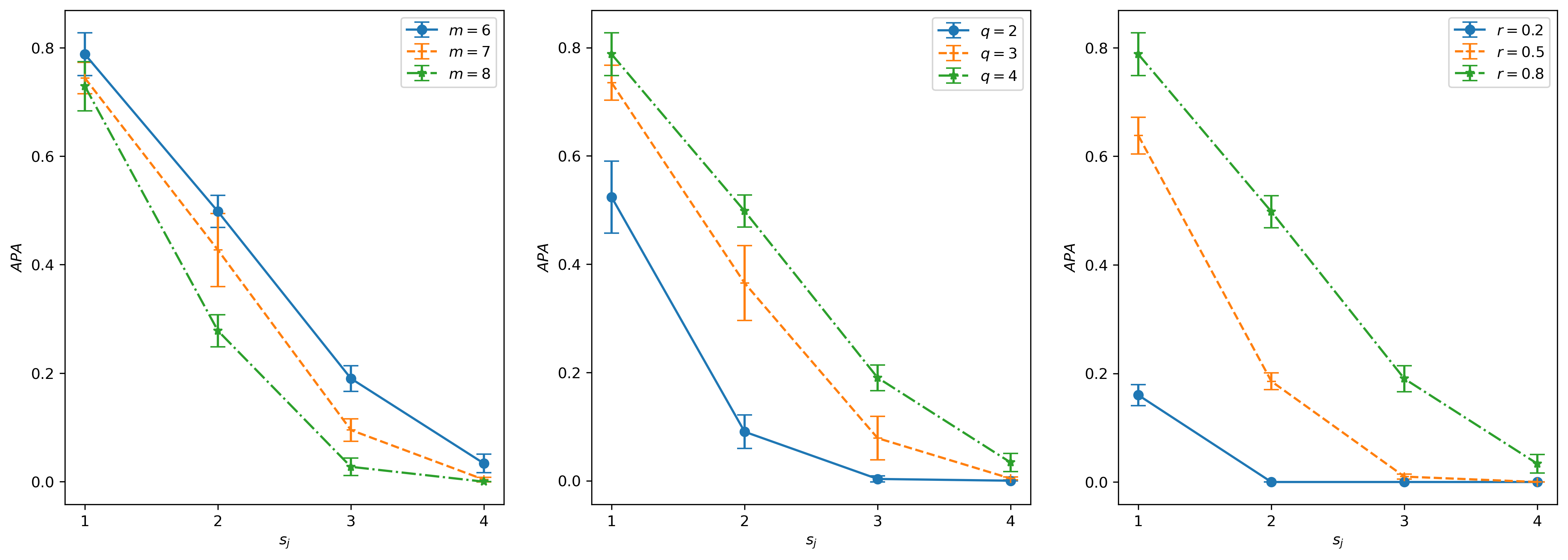}
\caption{The mean and  standard deviation of the $APA$ metric for the proposed approaches under different parameter settings when alternatives are balanced distributed across all categories}\label{fig:APA_balance}
\end{figure}

\section{Conclusion}\label{sec:7}
In this paper, we present lexicographic optimization-based approaches to learning a representative model for MCS problems with non-monotonic criteria from the preference disaggregation perspective. To elaborate, we commence by introducing transformation functions designed to map the marginal values of each criterion and the category thresholds into a UTA-like functional space. This guarantees that the sorting result for alternatives before and after transformation remain unchanged. Subsequently, some constraint sets are developed to model the non-monotonic criteria in MCS problems.
Moreover, we introduce a consistency check model and a preference adjustment model to derive consistent assignment example preference information. We further propose lexicographic optimization-based approaches to learning a representative sorting model. The proposed approaches simultaneously consider both model complexity and discriminative power. Following this, we conduct an extensive series of experiments to illustrate the effectiveness of our proposed approaches. The experimental results underscore the superiority of our approaches when compared to some non-monotonic criteria modeling methods. Ultimately, we investigate the influence of certain parameters on the robustness of the proposed approaches.

Additionally, we identify several potential avenues for future research:

(1) While this paper employs the threshold-based value-driven sorting procedure as the foundational sorting model, it would be promising to explore novel approaches that address non-monotonic criteria using alternative MCS methods, such as outranking-based sorting methods \citep{Kadzinski16ins,Dias18omega}.

(2) The proposed lexicographic optimization-based approaches primarily focus on simple assignment example preference information. However, decision makers may provide other forms of preference information, such as desired category cardinalities and assignment-based pairwise comparisons \citep{Kadzinski15ejor,Li23jors}. Future research could explore the integration of these types of preference information into our proposed approaches.

%(3) Exploring the robustness of the proposed approaches through experimental analysis also presents a compelling avenue for future research \citep{Kadzinski17caor}.

(3) Incorporating machine learning techniques into our approaches is a necessary step to effectively handle MCS problems involving large-scale datasets \citep{Martyn23ejor}. This integration could enhance the scalability and performance of our apporoaches in real-world applications.

\ack{This work was partly supported by the National Natural Science Foundation of China under Grant Nos. 72371049 and 71971039, the Natural Science Foundation of Liaoning Province under Grant No. 2024-MSBA-26, the Funds for Humanities and Social Sciences of Ministry of Education of China under Grant No. 23YJC630219 and the Fundamental Research Funds for the Central Universities of China under Grant No. DUT23RW406.}

\bibliographystyle{model5-names}
\bibliography{reference}

\begin{thebibliography}{56}
\expandafter\ifx\csname natexlab\endcsname\relax\def\natexlab#1{#1}\fi
\providecommand{\bibinfo}[2]{#2}
\ifx\xfnm\relax \def\xfnm[#1]{\unskip,\space#1}\fi
%Type = Article
\bibitem[{Almeida-Dias et~al.(2010)Almeida-Dias, Figueira \&
  Roy}]{Almeida10ejor}
\bibinfo{author}{Almeida-Dias, J.}, \bibinfo{author}{Figueira, J.~R.}, \&
  \bibinfo{author}{Roy, B.} (\bibinfo{year}{2010}).
\newblock \bibinfo{title}{{Electre Tri-C}: A multiple criteria sorting method
  based on characteristic reference actions}.
\newblock {\it \bibinfo{journal}{European Journal of Operational Research}\/},
  {\it \bibinfo{volume}{204}\/}, \bibinfo{pages}{565--580}.
%Type = Article
\bibitem[{Arcidiacono et~al.(2024)Arcidiacono, Corrente \&
  Greco}]{Arcidiacono23omega}
\bibinfo{author}{Arcidiacono, S.~G.}, \bibinfo{author}{Corrente, S.}, \&
  \bibinfo{author}{Greco, S.} (\bibinfo{year}{2024}).
\newblock \bibinfo{title}{Inducing a probability distribution in {Stochastic
  Multicriteria Acceptability Analysis}}.
\newblock {\it \bibinfo{journal}{Omega}\/},  {\it \bibinfo{volume}{123}\/},
  \bibinfo{pages}{102969}.
%Type = Article
\bibitem[{Bagherzadeh et~al.(2022)Bagherzadeh, Ghaderi \&
  Fernandez}]{Bagherzadeh22ejor}
\bibinfo{author}{Bagherzadeh, M.}, \bibinfo{author}{Ghaderi, M.}, \&
  \bibinfo{author}{Fernandez, A.-S.} (\bibinfo{year}{2022}).
\newblock \bibinfo{title}{Coopetition for innovation - the more, the better? an
  empirical study based on preference disaggregation analysis}.
\newblock {\it \bibinfo{journal}{European Journal of Operational Research}\/},
  {\it \bibinfo{volume}{297}\/}, \bibinfo{pages}{695--708}.
%Type = Article
\bibitem[{Belahc\'{e}ne et~al.(2023{\natexlab{a}})Belahc\'{e}ne, Mousseau,
  Ouerdane, Pirlot \& Sobrie}]{Belahcene234or}
\bibinfo{author}{Belahc\'{e}ne, K.}, \bibinfo{author}{Mousseau, V.},
  \bibinfo{author}{Ouerdane, W.}, \bibinfo{author}{Pirlot, M.}, \&
  \bibinfo{author}{Sobrie, O.} (\bibinfo{year}{2023}{\natexlab{a}}).
\newblock \bibinfo{title}{Multiple criteria sorting models and methods{---Part
  I}: survey of the literature}.
\newblock {\it \bibinfo{journal}{4OR}\/},  {\it \bibinfo{volume}{21}\/},
  \bibinfo{pages}{1--46}.
%Type = Article
\bibitem[{Belahc\'{e}ne et~al.(2023{\natexlab{b}})Belahc\'{e}ne, Mousseau,
  Ouerdane, Pirlot \& Sobrie}]{Belahcene234or1}
\bibinfo{author}{Belahc\'{e}ne, K.}, \bibinfo{author}{Mousseau, V.},
  \bibinfo{author}{Ouerdane, W.}, \bibinfo{author}{Pirlot, M.}, \&
  \bibinfo{author}{Sobrie, O.} (\bibinfo{year}{2023}{\natexlab{b}}).
\newblock \bibinfo{title}{Multiple criteria sorting models and methods. {Part
  II}: theoretical results and general issues}.
\newblock {\it \bibinfo{journal}{4OR}\/},  {\it \bibinfo{volume}{21}\/},
  \bibinfo{pages}{181--204}.
%Type = Article
\bibitem[{Belahc\'{e}ne et~al.(2023{\natexlab{c}})Belahc\'{e}ne, Mousseau,
  Ouerdane, Pirlot \& Sobrie}]{Belahcene23caor}
\bibinfo{author}{Belahc\'{e}ne, K.}, \bibinfo{author}{Mousseau, V.},
  \bibinfo{author}{Ouerdane, W.}, \bibinfo{author}{Pirlot, M.}, \&
  \bibinfo{author}{Sobrie, O.} (\bibinfo{year}{2023}{\natexlab{c}}).
\newblock \bibinfo{title}{Ranking with multiple reference points: Efficient
  {SAT}-based learning procedures}.
\newblock {\it \bibinfo{journal}{Computers \& Operations Research}\/},  {\it
  \bibinfo{volume}{150}\/}, \bibinfo{pages}{106054}.
%Type = Article
\bibitem[{Chen et~al.(2007)Chen, Hipel \& Kilgour}]{Chen07smca}
\bibinfo{author}{Chen, Y.}, \bibinfo{author}{Hipel, K.~W.}, \&
  \bibinfo{author}{Kilgour, D.~M.} (\bibinfo{year}{2007}).
\newblock \bibinfo{title}{Multiple-criteria sorting using case-based distance
  models with an application in water resources management}.
\newblock {\it \bibinfo{journal}{IEEE Transactions on Systems, Man, and
  Cybernetics-Part A: Systems and Humans}\/},  {\it \bibinfo{volume}{37}\/},
  \bibinfo{pages}{680--691}.
%Type = Article
\bibitem[{Chen et~al.(2008)Chen, Li, Kilgour \& Hipel}]{Chen08caor}
\bibinfo{author}{Chen, Y.}, \bibinfo{author}{Li, K.~W.},
  \bibinfo{author}{Kilgour, D.~M.}, \& \bibinfo{author}{Hipel, K.~W.}
  (\bibinfo{year}{2008}).
\newblock \bibinfo{title}{A case-based distance model for multiple criteria
  {ABC} analysis}.
\newblock {\it \bibinfo{journal}{Computers \& Operations Research}\/},  {\it
  \bibinfo{volume}{35}\/}, \bibinfo{pages}{776--796}.
%Type = Article
\bibitem[{Cinelli et~al.(2022)Cinelli, Kadzi\'{n}ski, Miebs, Gonzalez \&
  S{\l}owi\'{n}ski}]{Cinelli22ejor}
\bibinfo{author}{Cinelli, M.}, \bibinfo{author}{Kadzi\'{n}ski, M.},
  \bibinfo{author}{Miebs, G.}, \bibinfo{author}{Gonzalez, M.}, \&
  \bibinfo{author}{S{\l}owi\'{n}ski, R.} (\bibinfo{year}{2022}).
\newblock \bibinfo{title}{Recommending multiple criteria decision analysis
  methods with a new taxonomy-based decision support system}.
\newblock {\it \bibinfo{journal}{European Journal of Operational Research}\/},
  {\it \bibinfo{volume}{302}\/}, \bibinfo{pages}{633--651}.
%Type = Article
\bibitem[{Corrente et~al.(2013)Corrente, Greco, Kadzi\'{n}ski \&
  S{\l}owi\'{n}ski}]{Corrente13ml}
\bibinfo{author}{Corrente, S.}, \bibinfo{author}{Greco, S.},
  \bibinfo{author}{Kadzi\'{n}ski, M.}, \& \bibinfo{author}{S{\l}owi\'{n}ski,
  R.} (\bibinfo{year}{2013}).
\newblock \bibinfo{title}{Robust ordinal regression in preference learning and
  ranking}.
\newblock {\it \bibinfo{journal}{Machine Learning}\/},  {\it
  \bibinfo{volume}{93}\/}, \bibinfo{pages}{381--422}.
%Type = Article
\bibitem[{Corrente \& Tasiou(2023)}]{Corrente23eswa}
\bibinfo{author}{Corrente, S.}, \& \bibinfo{author}{Tasiou, M.}
  (\bibinfo{year}{2023}).
\newblock \bibinfo{title}{A robust {TOPSIS} method for decision making problems
  with hierarchical and non-monotonic criteria}.
\newblock {\it \bibinfo{journal}{Expert Systems with Applications}\/},  {\it
  \bibinfo{volume}{214}\/}, \bibinfo{pages}{119045}.
%Type = Article
\bibitem[{Dembczy\'{n}ski et~al.(2009)Dembczy\'{n}ski, Greco \&
  S{\l}owi\'{n}ski}]{Dembczynski09ejor}
\bibinfo{author}{Dembczy\'{n}ski, K.}, \bibinfo{author}{Greco, S.}, \&
  \bibinfo{author}{S{\l}owi\'{n}ski, R.} (\bibinfo{year}{2009}).
\newblock \bibinfo{title}{Rough set approach to multiple criteria
  classification with imprecise evaluations and assignments}.
\newblock {\it \bibinfo{journal}{European Journal of Operational Research}\/},
  {\it \bibinfo{volume}{198}\/}, \bibinfo{pages}{626--636}.
%Type = Incollection
\bibitem[{Despotis \& Zopounidis(1995)}]{Despotis95ama}
\bibinfo{author}{Despotis, D.~K.}, \& \bibinfo{author}{Zopounidis, C.}
  (\bibinfo{year}{1995}).
\newblock \bibinfo{title}{Building additive utilities in the presence of
  non-monotonic preferences}.
\newblock In {\it \bibinfo{booktitle}{Advances in Multicriteria Analysis}\/}
  (pp. \bibinfo{pages}{101--114}).
\newblock \bibinfo{address}{Boston, MA}: \bibinfo{publisher}{Springer}.
%Type = Incollection
\bibitem[{Devaud et~al.(1980)Devaud, Groussaud \&
  Jacquet-Lagreze}]{Devaud80UTADIS}
\bibinfo{author}{Devaud, J.~M.}, \bibinfo{author}{Groussaud, G.}, \&
  \bibinfo{author}{Jacquet-Lagreze, E.} (\bibinfo{year}{1980}).
\newblock \bibinfo{title}{{UTADIS}: Une m\'{e}thode de construction de
  fonctions d$^\prime$utilit\'{e} additives rendant compte de jugements
  globaux}.
\newblock In {\it \bibinfo{booktitle}{European Working Group on Multicriteria
  Decision Aid, Bochum}\/}.
\newblock volume~\bibinfo{volume}{94}.
%Type = Article
\bibitem[{Dias et~al.(2018)Dias, Antunes, Dantas, de~Castro \&
  Zamboni}]{Dias18omega}
\bibinfo{author}{Dias, L.~C.}, \bibinfo{author}{Antunes, C.~H.},
  \bibinfo{author}{Dantas, G.}, \bibinfo{author}{de~Castro, N.}, \&
  \bibinfo{author}{Zamboni, L.} (\bibinfo{year}{2018}).
\newblock \bibinfo{title}{A multi-criteria approach to sort and rank policies
  based on {D}elphi qualitative assessments and {ELECTRE TRI}: The case of
  smart grids in {Brazil}}.
\newblock {\it \bibinfo{journal}{Omega}\/},  {\it \bibinfo{volume}{76}\/},
  \bibinfo{pages}{100--111}.
%Type = Article
\bibitem[{Doumpos(2012)}]{Doumpos12orsp}
\bibinfo{author}{Doumpos, M.} (\bibinfo{year}{2012}).
\newblock \bibinfo{title}{Learning non-monotonic additive value functions for
  multicriteria decision making}.
\newblock {\it \bibinfo{journal}{OR Spectrum}\/},  {\it
  \bibinfo{volume}{34}\/}, \bibinfo{pages}{89--106}.
%Type = Article
\bibitem[{Doumpos et~al.(2001)Doumpos, Zanakis \& Zopounidis}]{Doumpo01ds}
\bibinfo{author}{Doumpos, M.}, \bibinfo{author}{Zanakis, S.~H.}, \&
  \bibinfo{author}{Zopounidis, C.} (\bibinfo{year}{2001}).
\newblock \bibinfo{title}{Multicriteria preference disaggregation for
  classification problems with an application to global investing risk}.
\newblock {\it \bibinfo{journal}{Decision Sciences}\/},  {\it
  \bibinfo{volume}{32}\/}, \bibinfo{pages}{333--386}.
%Type = Book
\bibitem[{Doumpos \& Zopounidis(2002)}]{Doumpo02springer}
\bibinfo{author}{Doumpos, M.}, \& \bibinfo{author}{Zopounidis, C.}
  (\bibinfo{year}{2002}).
\newblock {\it \bibinfo{title}{{Multicriteria Decision Aid Classification
  Methods}}\/}.
\newblock \bibinfo{publisher}{Springer New York, NY}.
%Type = Article
\bibitem[{Doumpos et~al.(2014)Doumpos, Zopounidis \&
  Galariotis}]{Doumpos14ejor}
\bibinfo{author}{Doumpos, M.}, \bibinfo{author}{Zopounidis, C.}, \&
  \bibinfo{author}{Galariotis, E.} (\bibinfo{year}{2014}).
\newblock \bibinfo{title}{Inferring robust decision models in multicriteria
  classification problems: An experimental analysis}.
\newblock {\it \bibinfo{journal}{European Journal of Operational Research}\/},
  {\it \bibinfo{volume}{236}\/}, \bibinfo{pages}{601--611}.
%Type = Inproceedings
\bibitem[{Eckhardt \& Kliegr(2012)}]{Eckhardt12pl}
\bibinfo{author}{Eckhardt, A.}, \& \bibinfo{author}{Kliegr, T.}
  (\bibinfo{year}{2012}).
\newblock \bibinfo{title}{Preprocessing algorithm for handling non-monotone
  attributes in the {UTA} method}.
\newblock In {\it \bibinfo{booktitle}{{Preference Learning: Problems and
  Applications in AI}}\/} (pp. \bibinfo{pages}{28--32}).
\newblock \bibinfo{address}{Montpellier, France}.
%Type = Article
\bibitem[{Ghaderi \& Kadzi\'{n}ski(2021)}]{Ghaderi21omega}
\bibinfo{author}{Ghaderi, M.}, \& \bibinfo{author}{Kadzi\'{n}ski, M.}
  (\bibinfo{year}{2021}).
\newblock \bibinfo{title}{Incorporating uncovered structural patterns in value
  functions construction}.
\newblock {\it \bibinfo{journal}{Omega}\/},  {\it \bibinfo{volume}{99}\/},
  \bibinfo{pages}{102203}.
%Type = Article
\bibitem[{Ghaderi et~al.(2015)Ghaderi, Ruiz \& Agell}]{Ghaderi15prl}
\bibinfo{author}{Ghaderi, M.}, \bibinfo{author}{Ruiz, F.}, \&
  \bibinfo{author}{Agell, N.} (\bibinfo{year}{2015}).
\newblock \bibinfo{title}{Understanding the impact of brand colour on brand
  image: A preference disaggregation approach}.
\newblock {\it \bibinfo{journal}{Pattern Recognition Letters}\/},  {\it
  \bibinfo{volume}{67}\/}, \bibinfo{pages}{11--18}.
%Type = Article
\bibitem[{Ghaderi et~al.(2017)Ghaderi, Ruiz \& Agell}]{Ghaderi17ejor}
\bibinfo{author}{Ghaderi, M.}, \bibinfo{author}{Ruiz, F.}, \&
  \bibinfo{author}{Agell, N.} (\bibinfo{year}{2017}).
\newblock \bibinfo{title}{A linear programming approach for learning
  non-monotonic additive value functions in multiple criteria decision aiding}.
\newblock {\it \bibinfo{journal}{European Journal of Operational Research}\/},
  {\it \bibinfo{volume}{259}\/}, \bibinfo{pages}{1073--1084}.
%Type = Article
\bibitem[{Greco et~al.(2011)Greco, Kadzi\'{n}ski \&
  S{\L}owi\'{n}ski}]{Greco11caor}
\bibinfo{author}{Greco, S.}, \bibinfo{author}{Kadzi\'{n}ski, M.}, \&
  \bibinfo{author}{S{\L}owi\'{n}ski, R.} (\bibinfo{year}{2011}).
\newblock \bibinfo{title}{Selection of a representative value function in
  robust multiple criteria sorting}.
\newblock {\it \bibinfo{journal}{Computers \& Operations Research}\/},  {\it
  \bibinfo{volume}{38}\/}, \bibinfo{pages}{1620--1637}.
%Type = Article
\bibitem[{Greco et~al.(2010)Greco, Mousseau \& S{\l}owi\'{n}ski}]{Greco10ejor}
\bibinfo{author}{Greco, S.}, \bibinfo{author}{Mousseau, V.}, \&
  \bibinfo{author}{S{\l}owi\'{n}ski, R.} (\bibinfo{year}{2010}).
\newblock \bibinfo{title}{Multiple criteria sorting with a set of additive
  value functions}.
\newblock {\it \bibinfo{journal}{European Journal of Operational Research}\/},
  {\it \bibinfo{volume}{207}\/}, \bibinfo{pages}{1455--1470}.
%Type = Article
\bibitem[{Guo et~al.(2019)Guo, Liao \& Liu}]{Guo19eswa}
\bibinfo{author}{Guo, M.}, \bibinfo{author}{Liao, X.}, \& \bibinfo{author}{Liu,
  J.} (\bibinfo{year}{2019}).
\newblock \bibinfo{title}{A progressive sorting approach for multiple criteria
  decision aiding in the presence of non-monotonic preferences}.
\newblock {\it \bibinfo{journal}{Expert Systems with Applications}\/},  {\it
  \bibinfo{volume}{123}\/}, \bibinfo{pages}{1--17}.
%Type = Article
\bibitem[{Guo et~al.(2020)Guo, Liao, Liu \& Zhang}]{Guo20omega}
\bibinfo{author}{Guo, M.}, \bibinfo{author}{Liao, X.}, \bibinfo{author}{Liu,
  J.}, \& \bibinfo{author}{Zhang, Q.} (\bibinfo{year}{2020}).
\newblock \bibinfo{title}{Consumer preference analysis: A data-driven multiple
  criteria approach integrating online information}.
\newblock {\it \bibinfo{journal}{Omega}\/},  {\it \bibinfo{volume}{96}\/},
  \bibinfo{pages}{102074}.
%Type = Article
\bibitem[{Jacquet-Lagr\`{e}ze \& Siskos(2001)}]{Jacquet01ejor}
\bibinfo{author}{Jacquet-Lagr\`{e}ze, E.}, \& \bibinfo{author}{Siskos, Y.}
  (\bibinfo{year}{2001}).
\newblock \bibinfo{title}{Preference disaggregation: 20 years of {MCDA}
  experience}.
\newblock {\it \bibinfo{journal}{European Journal of Operational Research}\/},
  {\it \bibinfo{volume}{130}\/}, \bibinfo{pages}{233--245}.
%Type = Article
\bibitem[{Kadzi\'{n}ski \& Ciomek(2016)}]{Kadzinski16ins}
\bibinfo{author}{Kadzi\'{n}ski, M.}, \& \bibinfo{author}{Ciomek, K.}
  (\bibinfo{year}{2016}).
\newblock \bibinfo{title}{Integrated framework for preference modeling and
  robustness analysis for outranking-based multiple criteria sorting with
  {ELECTRE} and {PROMETHEE}}.
\newblock {\it \bibinfo{journal}{Information Sciences}\/},  {\it
  \bibinfo{volume}{352-353}\/}, \bibinfo{pages}{167--187}.
%Type = Article
\bibitem[{Kadzi\'{n}ski \& Ciomek(2021)}]{Kadzinski21ejor}
\bibinfo{author}{Kadzi\'{n}ski, M.}, \& \bibinfo{author}{Ciomek, K.}
  (\bibinfo{year}{2021}).
\newblock \bibinfo{title}{Active learning strategies for interactive
  elicitation of assignment examples for threshold-based multiple criteria
  sorting}.
\newblock {\it \bibinfo{journal}{European Journal of Operational Research}\/},
  {\it \bibinfo{volume}{293}\/}, \bibinfo{pages}{658--680}.
%Type = Article
\bibitem[{Kadzi\'{n}ski et~al.(2015)Kadzi\'{n}ski, Ciomek \&
  S{\l}owi\'{n}ski}]{Kadzinski15ejor}
\bibinfo{author}{Kadzi\'{n}ski, M.}, \bibinfo{author}{Ciomek, K.}, \&
  \bibinfo{author}{S{\l}owi\'{n}ski, R.} (\bibinfo{year}{2015}).
\newblock \bibinfo{title}{Modeling assignment-based pairwise comparisons within
  integrated framework for value-driven multiple criteria sorting}.
\newblock {\it \bibinfo{journal}{European Journal of Operational Research}\/},
  {\it \bibinfo{volume}{241}\/}, \bibinfo{pages}{830--841}.
%Type = Article
\bibitem[{Kadzi\'{n}ski et~al.(2017)Kadzi\'{n}ski, Ghaderi, W\k{a}sikowski \&
  Agell}]{Kadzinski17caor}
\bibinfo{author}{Kadzi\'{n}ski, M.}, \bibinfo{author}{Ghaderi, M.},
  \bibinfo{author}{W\k{a}sikowski, J.}, \& \bibinfo{author}{Agell, N.}
  (\bibinfo{year}{2017}).
\newblock \bibinfo{title}{Expressiveness and robustness measures for the
  evaluation of an additive value function in multiple criteria preference
  disaggregation methods: An experimental analysis}.
\newblock {\it \bibinfo{journal}{Computers \& Operations Research}\/},  {\it
  \bibinfo{volume}{87}\/}, \bibinfo{pages}{146--164}.
%Type = Article
\bibitem[{Kadzi\'{n}ski et~al.(2020)Kadzi\'{n}ski, Martyn, Cinelli,
  S{\l}owi\'{n}ski, Corrente \& Greco}]{Kadzinski20ijar}
\bibinfo{author}{Kadzi\'{n}ski, M.}, \bibinfo{author}{Martyn, K.},
  \bibinfo{author}{Cinelli, M.}, \bibinfo{author}{S{\l}owi\'{n}ski, R.},
  \bibinfo{author}{Corrente, S.}, \& \bibinfo{author}{Greco, S.}
  (\bibinfo{year}{2020}).
\newblock \bibinfo{title}{Preference disaggregation for multiple criteria
  sorting with partial monotonicity constraints: Application to exposure
  management of nanomaterials}.
\newblock {\it \bibinfo{journal}{International Journal of Approximate
  Reasoning}\/},  {\it \bibinfo{volume}{117}\/}, \bibinfo{pages}{60--80}.
%Type = Article
\bibitem[{Kadzi\'{n}ski et~al.(2021)Kadzi\'{n}ski, Martyn, Cinelli,
  S{\l}owi\'{n}ski, Corrente \& Greco}]{Kadzinski21kbs}
\bibinfo{author}{Kadzi\'{n}ski, M.}, \bibinfo{author}{Martyn, K.},
  \bibinfo{author}{Cinelli, M.}, \bibinfo{author}{S{\l}owi\'{n}ski, R.},
  \bibinfo{author}{Corrente, S.}, \& \bibinfo{author}{Greco, S.}
  (\bibinfo{year}{2021}).
\newblock \bibinfo{title}{Preference disaggregation method for value-based
  multi-decision sorting problems with a real-world application in
  nanotechnology}.
\newblock {\it \bibinfo{journal}{Knowledge-Based Systems}\/},  {\it
  \bibinfo{volume}{218}\/}, \bibinfo{pages}{106879}.
%Type = Article
\bibitem[{Kadzi\'{n}ski et~al.(2016)Kadzi\'{n}ski, S{\l}owi\'{n}ski \&
  Greco}]{Kadzinski16ins1}
\bibinfo{author}{Kadzi\'{n}ski, M.}, \bibinfo{author}{S{\l}owi\'{n}ski, R.}, \&
  \bibinfo{author}{Greco, S.} (\bibinfo{year}{2016}).
\newblock \bibinfo{title}{Robustness analysis for decision under uncertainty
  with rule-based preference model}.
\newblock {\it \bibinfo{journal}{Information Sciences}\/},  {\it
  \bibinfo{volume}{328}\/}, \bibinfo{pages}{321--339}.
%Type = Article
\bibitem[{Kadzi\'{n}ski \& Szczepa\'{n}ski(2021)}]{Kadzinski21asoc}
\bibinfo{author}{Kadzi\'{n}ski, M.}, \& \bibinfo{author}{Szczepa\'{n}ski, A.}
  (\bibinfo{year}{2021}).
\newblock \bibinfo{title}{Learning the parameters of an outranking-based
  sorting model with characteristic class profiles from large sets of
  assignment examples}.
\newblock {\it \bibinfo{journal}{Applied Soft Computing}\/},  {\it
  \bibinfo{volume}{116}\/}, \bibinfo{pages}{108312}.
%Type = Article
\bibitem[{Kadzi\'{n}ski \& Tervonen(2013)}]{Kadzinski13dss}
\bibinfo{author}{Kadzi\'{n}ski, M.}, \& \bibinfo{author}{Tervonen, T.}
  (\bibinfo{year}{2013}).
\newblock \bibinfo{title}{Stochastic ordinal regression for multiple criteria
  sorting problems}.
\newblock {\it \bibinfo{journal}{Decision Support Systems}\/},  {\it
  \bibinfo{volume}{55}\/}, \bibinfo{pages}{55--66}.
%Type = Incollection
\bibitem[{Kliegr(2009)}]{Kliegr09pl}
\bibinfo{author}{Kliegr, T.} (\bibinfo{year}{2009}).
\newblock \bibinfo{title}{{UTA-NM}: Explaining stated preferences with additive
  non-monotonic utility functions}.
\newblock In {\it \bibinfo{booktitle}{{Proceedings of the Preference Learning
  ECML/PKDD-2009 workshop}}\/} (pp. \bibinfo{pages}{56--68}).
%Type = Article
\bibitem[{Li \& Zhang(2024)}]{Li23tcss}
\bibinfo{author}{Li, Z.}, \& \bibinfo{author}{Zhang, Z.}
  (\bibinfo{year}{2024}).
\newblock \bibinfo{title}{Threshold-based value-driven method to support
  consensus reaching in multicriteria group sorting problems: A minimum
  adjustment perspective}.
\newblock {\it \bibinfo{journal}{IEEE Transactions on Computational Social
  Systems}\/},  {\it \bibinfo{volume}{11}\/}, \bibinfo{pages}{1230--1243}.
%Type = Article
\bibitem[{Li et~al.(2024)Li, Zhang \& Yu}]{Li23jors}
\bibinfo{author}{Li, Z.}, \bibinfo{author}{Zhang, Z.}, \& \bibinfo{author}{Yu,
  W.} (\bibinfo{year}{2024}).
\newblock \bibinfo{title}{Consensus reaching for ordinal classification-based
  group decision making with heterogeneous preference information}.
\newblock {\it \bibinfo{journal}{Journal of the Operational Research
  Society}\/},  {\it \bibinfo{volume}{75}\/}, \bibinfo{pages}{224--245}.
%Type = Article
\bibitem[{de~Lima~Silva \& de~Almeida~Filho(2020)}]{de20caie}
\bibinfo{author}{de~Lima~Silva, D.~F.}, \& \bibinfo{author}{de~Almeida~Filho,
  A.~T.} (\bibinfo{year}{2020}).
\newblock \bibinfo{title}{Sorting with {TOPSIS} through boundary and
  characteristic profiles}.
\newblock {\it \bibinfo{journal}{Computers \& Industrial Engineering}\/},  {\it
  \bibinfo{volume}{141}\/}, \bibinfo{pages}{106328}.
%Type = Article
\bibitem[{Liu et~al.(2020{\natexlab{a}})Liu, Kadzi\'{n}ski, Liao \&
  Mao}]{Liu20joc}
\bibinfo{author}{Liu, J.}, \bibinfo{author}{Kadzi\'{n}ski, M.},
  \bibinfo{author}{Liao, X.}, \& \bibinfo{author}{Mao, X.}
  (\bibinfo{year}{2020}{\natexlab{a}}).
\newblock \bibinfo{title}{Data-driven preference learning methods for
  value-driven multiple criteria sorting with interacting criteria}.
\newblock {\it \bibinfo{journal}{INFORMS Journal on Computing}\/},  {\it
  \bibinfo{volume}{33}\/}, \bibinfo{pages}{586--606}.
%Type = Article
\bibitem[{Liu et~al.(2020{\natexlab{b}})Liu, Kadzi\'{n}ski, Liao, Mao \&
  Wang}]{Liu20ejor}
\bibinfo{author}{Liu, J.}, \bibinfo{author}{Kadzi\'{n}ski, M.},
  \bibinfo{author}{Liao, X.}, \bibinfo{author}{Mao, X.}, \&
  \bibinfo{author}{Wang, Y.} (\bibinfo{year}{2020}{\natexlab{b}}).
\newblock \bibinfo{title}{A preference learning framework for multiple criteria
  sorting with diverse additive value models and valued assignment examples}.
\newblock {\it \bibinfo{journal}{European Journal of Operational Research}\/},
  {\it \bibinfo{volume}{286}\/}, \bibinfo{pages}{963--985}.
%Type = Article
\bibitem[{Liu et~al.(2019)Liu, Liao, Kadzi\'{n}ski \&
  S{\l}owi\'{n}ski}]{Liu19ejor}
\bibinfo{author}{Liu, J.}, \bibinfo{author}{Liao, X.},
  \bibinfo{author}{Kadzi\'{n}ski, M.}, \& \bibinfo{author}{S{\l}owi\'{n}ski,
  R.} (\bibinfo{year}{2019}).
\newblock \bibinfo{title}{Preference disaggregation within the regularization
  framework for sorting problems with multiple potentially non-monotonic
  criteria}.
\newblock {\it \bibinfo{journal}{European Journal of Operational Research}\/},
  {\it \bibinfo{volume}{276}\/}, \bibinfo{pages}{1071--1089}.
%Type = Article
\bibitem[{Liu et~al.(2015)Liu, Liao \& Yang}]{Liu15ejor}
\bibinfo{author}{Liu, J.}, \bibinfo{author}{Liao, X.}, \&
  \bibinfo{author}{Yang, J.-b.} (\bibinfo{year}{2015}).
\newblock \bibinfo{title}{A group decision-making approach based on evidential
  reasoning for multiple criteria sorting problem with uncertainty}.
\newblock {\it \bibinfo{journal}{European Journal of Operational Research}\/},
  {\it \bibinfo{volume}{246}\/}, \bibinfo{pages}{858--873}.
%Type = Article
\bibitem[{Liu et~al.(2016)Liu, Liao, Zhao \& Yang}]{Liu16omega}
\bibinfo{author}{Liu, J.}, \bibinfo{author}{Liao, X.}, \bibinfo{author}{Zhao,
  W.}, \& \bibinfo{author}{Yang, N.} (\bibinfo{year}{2016}).
\newblock \bibinfo{title}{A classification approach based on the outranking
  model for multiple criteria {ABC} analysis}.
\newblock {\it \bibinfo{journal}{Omega}\/},  {\it \bibinfo{volume}{61}\/},
  \bibinfo{pages}{19--34}.
%Type = Article
\bibitem[{Martyn \& Kadzi\'{n}ski(2023)}]{Martyn23ejor}
\bibinfo{author}{Martyn, K.}, \& \bibinfo{author}{Kadzi\'{n}ski, M.}
  (\bibinfo{year}{2023}).
\newblock \bibinfo{title}{Deep preference learning for multiple criteria
  decision analysis}.
\newblock {\it \bibinfo{journal}{European Journal of Operational Research}\/},
  {\it \bibinfo{volume}{305}\/}, \bibinfo{pages}{781--805}.
%Type = Article
\bibitem[{Olteanu \& Meyer(2022)}]{Olteanu22caor}
\bibinfo{author}{Olteanu, A.-L.}, \& \bibinfo{author}{Meyer, P.}
  (\bibinfo{year}{2022}).
\newblock \bibinfo{title}{Inferring a hierarchical majority-rule sorting
  model}.
\newblock {\it \bibinfo{journal}{Computers \& Operations Research}\/},  {\it
  \bibinfo{volume}{146}\/}, \bibinfo{pages}{105888}.
%Type = Article
\bibitem[{Pelissari \& Duarte(2022)}]{Pelissari22eswa}
\bibinfo{author}{Pelissari, R.}, \& \bibinfo{author}{Duarte, L.~T.}
  (\bibinfo{year}{2022}).
\newblock \bibinfo{title}{{SMAA-Choquet-FlowSort}: A novel
  user-preference-driven {Choquet} classifier applied to supplier evaluation}.
\newblock {\it \bibinfo{journal}{Expert Systems with Applications}\/},  {\it
  \bibinfo{volume}{207}\/}, \bibinfo{pages}{117898}.
%Type = Article
\bibitem[{Qin et~al.(2021)Qin, Zeng \& Zhou}]{Qin21ins}
\bibinfo{author}{Qin, J.}, \bibinfo{author}{Zeng, Y.}, \&
  \bibinfo{author}{Zhou, Y.} (\bibinfo{year}{2021}).
\newblock \bibinfo{title}{Context-dependent {DEASort}: A multiple criteria
  sorting method for ecological risk assessment problems}.
\newblock {\it \bibinfo{journal}{Information Sciences}\/},  {\it
  \bibinfo{volume}{572}\/}, \bibinfo{pages}{88--108}.
%Type = Article
\bibitem[{Rezaei(2018)}]{Rezaei18eswa}
\bibinfo{author}{Rezaei, J.} (\bibinfo{year}{2018}).
\newblock \bibinfo{title}{Piecewise linear value functions for multi-criteria
  decision-making}.
\newblock {\it \bibinfo{journal}{Expert Systems with Applications}\/},  {\it
  \bibinfo{volume}{98}\/}, \bibinfo{pages}{43--56}.
%Type = Article
\bibitem[{Ru et~al.(2023)Ru, Liu, Kadzi\'{n}ski \& Liao}]{Ru23ejor}
\bibinfo{author}{Ru, Z.}, \bibinfo{author}{Liu, J.},
  \bibinfo{author}{Kadzi\'{n}ski, M.}, \& \bibinfo{author}{Liao, X.}
  (\bibinfo{year}{2023}).
\newblock \bibinfo{title}{Probabilistic ordinal regression methods for multiple
  criteria sorting admitting certain and uncertain preferences}.
\newblock {\it \bibinfo{journal}{European Journal of Operational Research}\/},
  {\it \bibinfo{volume}{311}\/}, \bibinfo{pages}{596--616}.
%Type = Article
\bibitem[{Sobrie et~al.(2018)Sobrie, Gillis, Mousseau \& Pirlot}]{Sobrie18ejor}
\bibinfo{author}{Sobrie, O.}, \bibinfo{author}{Gillis, N.},
  \bibinfo{author}{Mousseau, V.}, \& \bibinfo{author}{Pirlot, M.}
  (\bibinfo{year}{2018}).
\newblock \bibinfo{title}{{UTA-poly and UTA-splines}: Additive value functions
  with polynomial marginals}.
\newblock {\it \bibinfo{journal}{European Journal of Operational Research}\/},
  {\it \bibinfo{volume}{264}\/}, \bibinfo{pages}{405--418}.
%Type = Article
\bibitem[{Tomczyk \& Kadzi\'{n}ski(2019)}]{Tomczyk19caor}
\bibinfo{author}{Tomczyk, M.~K.}, \& \bibinfo{author}{Kadzi\'{n}ski, M.}
  (\bibinfo{year}{2019}).
\newblock \bibinfo{title}{Emosor: Evolutionary multiple objective optimization
  guided by interactive stochastic ordinal regression}.
\newblock {\it \bibinfo{journal}{Computers \& Operations Research}\/},  {\it
  \bibinfo{volume}{108}\/}, \bibinfo{pages}{134--154}.
%Type = Article
\bibitem[{Zhang \& Li(2023)}]{Zhang22anor}
\bibinfo{author}{Zhang, Z.}, \& \bibinfo{author}{Li, Z.}
  (\bibinfo{year}{2023}).
\newblock \bibinfo{title}{Consensus-based {TOPSIS-Sort-B} for multi-criteria
  sorting in the context of group decision-making}.
\newblock {\it \bibinfo{journal}{Annals of Operations Research}\/},  {\it
  \bibinfo{volume}{325}\/}, \bibinfo{pages}{911--938}.
%Type = Article
\bibitem[{Zopounidis \& Doumpos(2001)}]{Zopounidis01ejor}
\bibinfo{author}{Zopounidis, C.}, \& \bibinfo{author}{Doumpos, M.}
  (\bibinfo{year}{2001}).
\newblock \bibinfo{title}{A preference disaggregation decision support system
  for financial classification problems}.
\newblock {\it \bibinfo{journal}{European Journal of Operational Research}\/},
  {\it \bibinfo{volume}{130}\/}, \bibinfo{pages}{402--413}.

\end{thebibliography}

\newpage

{\Large \noindent \textbf{Supplemental File}}

\appendix
\setcounter{table}{0}
\pagenumbering{arabic}
\setcounter{page}{1}

\renewcommand{\thealgorithm}{\Alph{section}\arabic{algorithm}}
\setcounter{algorithm}{0}

\section{Algorithms for generating simulation datasets and comparison}

\begin{algorithm}\label{alg:2}
~\\
{\bf Input:} The number of alternatives $n$, the number of criteria $m$, the number of categories $q$, and the number of subintervals for each criterion $s_j$, $j\in M$.\\
{\bf Output:} The decision matrix $X=(x_{ij})_{n\times m}$ and the sorting result for alternatives $F=(f_1,f_2,\ldots,f_n)^{\rm T}$.
\begin{enumerate}[\bf Step 1:]
  \item Generate the decision matrix $X=(x_{ij})_{n\times m}$ by randomly generating values for each element $x_{ij}$ within the interval [0, 100], following a uniform distribution, $i\in N$, $j\in M$.
  \item Using the decision matrix $X$, calculate the minimum and maximum performance levels of each criterion, i.e., $\beta_j^-$ and $\beta_j^+$, $j\in M$.
  \item Based on $\beta_j^-$, $\beta_j^+$ and the number of subintervals $s_j$, calculate all the breakpoints $\beta_j^l$, $l=1,\ldots,s_j+1$, $j\in M$.
  \item Randomly generate the marginal value for each breakpoint $v_j(\beta_j^l)$ from the interval [0,1] based on a uniform distribution, $l=1,\ldots,s_j+1, j\in M$.
  \item Use Eq. \eqref{eq:fv} to calculate the transformed marginal value for each breakpoint $v_j(\beta_j^l)$, $l=1,\ldots,s_j+1, j\in M$. Furthermore, set the category thresholds as follows: $b_0=0$, $b_q=1$, $b_h=1/q$, $h=1,\ldots,q-1$.
  \item Employ Eqs. \eqref{eq:global_v}~-~\eqref{eq:marginal_v0} to calculate the global value for each alternative. Afterward, derive the sorting result for alternatives, denoted as $F=(f_1, f_2,\ldots,f_n)^{\rm T}$, using the threshold-based value-driven sorting procedure.
   \item Output $X$ and $F$.
\end{enumerate}
\end{algorithm}

\begin{algorithm}\label{alg:3}
~\\
{\bf Input:} The decision matrix $X$, the sorting result for alternatives $F$, the number of subintervals for each criterion $s_j$, $j\in M$ and the proportion of reference alternatives among all alternatives $r$.\\
{\bf Output:} The accuracy metric of the four approaches, denoted by $accuracy_1$, $accuracy_2$, $accuracy_3$ and $accuracy_4$ , respectively.
\begin{enumerate}[\bf Step 1:]
  \item Randomly partition the set of alternatives $A=\{a_1,a_2,\ldots,a_n\}$ into two subsets: the set of reference alternatives $A^R$ and the set of non-reference alternatives $A^N$, where $|A^R|=[n\cdot r]$ and $|A^N|=n-|A^R|$ are the number of alternatives in $A^R$ and $A^N$, respectively.
  \item Determine the sorting result for alternatives in $A^R$ and $A^N$ based on $F$. These results are denoted as $F_1$ for $A^R$ and $F_2$ for $A^N$.
  \item Employ Approach 1 to derive the inferred sorting result for non-reference alternatives as $\overline{F_2^1}$, and then move on to Step 7.
  \item Use Approach 2 to obtain the inferred sorting result for non-reference alternatives as $\overline{F_2^2}$, and then move on to Step 7.
  \item Utilize LFP-NONM to derive the inferred sorting result for non-reference alternatives as $\overline{F_2^3}$, i.e., solve the model \eqref{m:LFP}, and then move on to Step 7.
  \item Apply UTADIS-NONM to derive the inferred sorting result for non-reference alternatives as $\overline{F_2^4}$, i.e., directly utilize the model \eqref{m:con_check} to derive the representative sorting model and implement Step 5 in Algorithm \ref{alg:1}. Then, continue to the next step.
  \item Based on $\overline{F_2^1}$, $\overline{F_2^2}$, $\overline{F_2^3}$,  $\overline{F_2^4}$ and $F_2$, calculate the accuracy metrics $accuracy_1$, $accuracy_2$, $accuracy_3$ and $accuracy_4$ by Eq. \eqref{eq:accu2}.
  \item Output $accuracy_1$, $accuracy_2$, $accuracy_3$ and $accuracy_4$.
\end{enumerate}
\end{algorithm}

\newpage
\section{The results of $t$-tests in Simulation experiments \RNum{1}~-~\RNum{4}}

\begin{table}[htbp]
\centering
\setlength{\abovecaptionskip}{0pt}
\setlength{\belowcaptionskip}{10pt}
\caption{The results of $t$-tests in Simulation experiments \RNum{1}~-~\RNum{4}}
\label{table:ttest1}
\resizebox{0.73\textwidth}{!}{
\begin{threeparttable}
\begin{tabular}{cccccc}
\toprule
Simulation experiment & Parameter & $s_j$ & Comparison 1 & Comparison 2 & Comparison 3  \\ \midrule
\multirow{12}*{\RNum{1}} & \multirow{4}*{$n=200$} & $s_j=1$ & - & $0.0097^*$ & - \\
&  &$s_j=2$ & $0.0000^*$ & $0.0000^*$ & $0.0000^*$ \\
&  &$s_j=3$ & $0.0000^*$ & $0.0000^*$ & $0.0000^*$ \\
&  &$s_j=4$ & $0.0000^*$ & $0.0000^*$ & $0.0000^*$ \\

& \multirow{4}*{$n=400$} & $s_j=1$ & - & $0.0008^*$ & - \\
&  &$s_j=2$ & $0.0000^*$ & $0.0000^*$ & $0.0000^*$ \\
&  &$s_j=3$ & $0.0000^*$ & $0.0000^*$ & $0.0000^*$ \\
&  &$s_j=4$ & $0.0000^*$ & $0.0000^*$ & $0.0000^*$ \\

& \multirow{4}*{$n=600$} & $s_j=1$ & - & $0.0004^*$ & - \\
&  &$s_j=2$ & $0.0000^*$ & $0.0001^*$ & $0.0001^*$ \\
&  &$s_j=3$ & $0.0004^*$ & $0.0003^*$ & $0.0003^*$ \\
&  &$s_j=4$ & $0.0000^*$ & $0.0000^*$ & $0.0000^*$ \\ \midrule

\multirow{12}*{\RNum{2}} & \multirow{4}*{$m=6$} & $s_j=1$ & - & $0.0097^*$ & - \\
&  &$s_j=2$ & $0.0000^*$ & $0.0000^*$ & $0.0000^*$ \\
&  &$s_j=3$ & $0.0000^*$ & $0.0000^*$ & $0.0000^*$ \\
&  &$s_j=4$ & $0.0000^*$ & $0.0000^*$ & $0.0000^*$ \\

& \multirow{4}*{$m=7$} & $s_j=1$ & - & $0.0011^*$  & - \\
&  &$s_j=2$ & $0.0000^*$ & $0.0000^*$ & $0.0000^*$ \\
&  &$s_j=3$ & $0.0000^*$ & $0.0000^*$ & $0.0000^*$ \\
&  &$s_j=4$ & $0.0000^*$ & $0.0000^*$ & $0.0000^*$ \\

& \multirow{4}*{$m=8$} & $s_j=1$ & - & $0.0000^*$ & - \\
&  &$s_j=2$ & $0.0000^*$ & $0.0000^*$ & $0.0000^*$ \\
&  &$s_j=3$ & $0.0000^*$ & $0.0000^*$ & $0.0000^*$ \\
&  &$s_j=4$ & $0.0000^*$ & $0.0000^*$ & $0.0000^*$ \\ \midrule

\multirow{12}*{\RNum{3}} & \multirow{4}*{$q=2$} & $s_j=1$ & - & $0.0000^*$ & - \\
&  &$s_j=2$ & $0.0000^*$ & $0.0000^*$ & $0.0000^*$ \\
&  &$s_j=3$ & $0.0000^*$ & $0.0000^*$ & $0.0000^*$ \\
&  &$s_j=4$ & $0.0002^*$ & $0.0000^*$ & $0.0002^*$ \\

& \multirow{4}*{$q=3$} & $s_j=1$ & - & $0.0000^*$ & - \\
&  &$s_j=2$ & $0.0000^*$ & $0.0000^*$ & $0.0000^*$ \\
&  &$s_j=3$ & $0.0000^*$ & $0.0000^*$ & $0.0000^*$ \\
&  &$s_j=4$ & $0.0000^*$ & $0.0000^*$ & $0.0000^*$ \\

& \multirow{4}*{$q=4$} & $s_j=1$ & - & $0.0097^*$ & - \\
&  &$s_j=2$ & $0.0000^*$ & $0.0000^*$ & $0.0000^*$ \\
&  &$s_j=3$ & $0.0000^*$ & $0.0000^*$ & $0.0000^*$ \\
&  &$s_j=4$ & $0.0000^*$ & $0.0000^*$ & $0.0000^*$ \\  \midrule

\multirow{12}*{\RNum{4}} & \multirow{4}*{$r=0.2$} & $s_j=1$ & - & $0.0000^*$ & - \\
&  &$s_j=2$ & $0.0000^*$ & $0.0000^*$ & $0.0000^*$ \\
&  &$s_j=3$ & $0.0000^*$ & $0.0000^*$ & $0.0000^*$ \\
&  &$s_j=4$ & $0.0000^*$ & $0.0000^*$ & $0.0000^*$ \\

& \multirow{4}*{$r=0.5$} & $s_j=1$ & - & $0.0003^*$ & - \\
&  &$s_j=2$ & $0.0000^*$ & $0.0000^*$ & $0.0000^*$ \\
&  &$s_j=3$ & $0.0000^*$ & $0.0000^*$ & $0.0000^*$ \\
&  &$s_j=4$ & $0.0000^*$ & $0.0000^*$ & $0.0000^*$ \\

& \multirow{4}*{$r=0.8$} & $s_j=1$ & - & $0.0097^*$ & - \\
&  &$s_j=2$ & $0.0000^*$ & $0.0000^*$ & $0.0000^*$ \\
&  &$s_j=3$ & $0.0000^*$ & $0.0000^*$ & $0.0000^*$ \\
&  &$s_j=4$ & $0.0000^*$ & $0.0000^*$ & $0.0000^*$ \\

\bottomrule
\end{tabular}
\begin{tablenotes}
\footnotesize
\item[*] Significance at the 5 \% level.
\end{tablenotes}
\end{threeparttable}}
\end{table}

\begin{small}
\begin{table}[htbp]
\centering
\setlength{\abovecaptionskip}{0pt}
\setlength{\belowcaptionskip}{10pt}
\caption{The results of $t$-tests in Simulation experiments \RNum{5}~-~\RNum{7}}
\label{table:ttest2}
\resizebox{0.73\textwidth}{!}{
\begin{threeparttable}
\begin{tabular}{ccccc}
\toprule
Simulation experiment & Parameter & $s_j$ & Comparison 4 & Comparison 5  \\ \midrule
\multirow{12}*{\RNum{5}} & \multirow{4}*{$m=6$} & $s_j=1$ & - & $0.0000^*$  \\
&  &$s_j=2$ & $0.0000^*$ & $0.0000^*$  \\
&  &$s_j=3$ & $0.0000^*$ & $0.0000^*$  \\
&  &$s_j=4$ & $0.0000^*$ & $0.0000^*$  \\

& \multirow{4}*{$m=7$} & $s_j=1$ & - & $0.0000^*$  \\
&  &$s_j=2$ & $0.0000^*$ & $0.0000^*$  \\
&  &$s_j=3$ & $0.0000^*$ & $0.0000^*$  \\
&  &$s_j=4$ & $0.0000^*$ & $0.0000^*$  \\

& \multirow{4}*{$m=8$} & $s_j=1$ & - & $0.0000^*$  \\
&  &$s_j=2$ & $0.0000^*$ & $0.0000^*$  \\
&  &$s_j=3$ & $0.0000^*$ & $0.0000^*$  \\
&  &$s_j=4$ & $0.0000^*$ & $0.0000^*$  \\  \midrule

\multirow{12}*{\RNum{6}} & \multirow{4}*{$q=2$} & $s_j=1$ & - & $0.0000^*$  \\
&  &$s_j=2$ & $0.0000^*$ & $0.0000^*$  \\
&  &$s_j=3$ & $0.0000^*$ & $0.0000^*$  \\
&  &$s_j=4$ & $0.0000^*$ & $0.0000^*$  \\

& \multirow{4}*{$q=3$} & $s_j=1$ & - & $0.0000^*$   \\
&  &$s_j=2$ & $0.0000^*$ & $0.0000^*$  \\
&  &$s_j=3$ & $0.0000^*$ & $0.0000^*$  \\
&  &$s_j=4$ & $0.0000^*$ & $0.0000^*$  \\

& \multirow{4}*{$q=4$} & $s_j=1$ & - & $0.0000^*$ \\
&  &$s_j=2$ & $0.0000^*$ & $0.0000^*$  \\
&  &$s_j=3$ & $0.0000^*$ & $0.0000^*$  \\
&  &$s_j=4$ & $0.0000^*$ & $0.0000^*$  \\  \midrule

\multirow{12}*{\RNum{7}} & \multirow{4}*{$r=0.2$} & $s_j=1$ & - & $0.0000^*$  \\
&  &$s_j=2$ & $0.0000^*$ & $0.0000^*$  \\
&  &$s_j=3$ & $0.0000^*$ & $0.0000^*$  \\
&  &$s_j=4$ & $0.0000^*$ & $0.0000^*$  \\

& \multirow{4}*{$r=0.5$} & $s_j=1$ & - & $0.0000^*$ \\
&  &$s_j=2$ & $0.0000^*$ & $0.0000^*$  \\
&  &$s_j=3$ & $0.0000^*$ & $0.0000^*$  \\
&  &$s_j=4$ & $0.0000^*$ & $0.0000^*$  \\

& \multirow{4}*{$r=0.8$} & $s_j=1$ & - & $0.0000^*$  \\
&  &$s_j=2$ & $0.0000^*$ & $0.0000^*$  \\
&  &$s_j=3$ & $0.0000^*$ & $0.0000^*$  \\
&  &$s_j=4$ & $0.0000^*$ & $0.0000^*$  \\

\bottomrule
\end{tabular}
\begin{tablenotes}
\footnotesize
\item \emph{Note:}  Comparison 4 evaluates the differences between Approach 1 and \cite{Kadzinski20ijar}'s method, and Comparison 5 evaluates the differences between Approach 2 and \cite{Kadzinski20ijar}'s method.
\item[*] Significance at the 5 \% level.
\end{tablenotes}
\end{threeparttable}}
\end{table}
\end{small}

\newpage

\section{Simulation experiments under the situation that alternatives are balanced distributed across all categories }\label{appendix:c}
We modify Steps 5 and 6 of Algorithm \ref{alg:2} to Steps 5$'$ and 6$'$, and form a new algorithm (Algorithm A1$'$).

Step 5$'$: Use Eq. \eqref{eq:fv} to calculate the transformed marginal value for each breakpoint $v_j(\beta_j^l)$, $l=1,\ldots,s_j+1, j\in M$.

Step 6$'$: Employ Eqs. \eqref{eq:global_v}~-~\eqref{eq:marginal_v0} to calculate the global value for each alternative. Afterwards, determine the category thresholds with the objective of achieving an approximately equal number of alternatives across different categories. Finally, employ the threshold-based value-driven sorting procedure to derive the sorting result for alternatives, denoted as $F=(f_1, f_2,\ldots,f_n)^{\rm T}$.

Moreover, we update Step 2 of Algorithm \ref{alg:3} to Step 2$'$, and form Algorithm A2$'$.

Step 2$'$: Calculate the number of reference alternatives to be selected, i.e., $|A^R|=[n\cdot r]$. Subsequently, choose $|A^R|$ reference alternatives with the aim of achieving a balanced distribution of alternatives across all categories. Additionally, denote the remaining $|A^N|=n-|A^R|$ alternatives as non-reference alternatives.

On this basis, we replicate Simulation experiments \RNum{1}~-~\RNum{7} using Algorithms A1$'$ and A2$'$ (referred to as Simulation experiments \RNum{1}$'$~-~\RNum{7}$'$). The simulation results are reported in Tables \ref{table:r_simu1}~-~\ref{table:r_simu7}. Analogously, we also conduct $t$-tests for Simulation experiments \RNum{1}$'$~-~\RNum{4}$'$, Simulation experiments \RNum{5}$'$~-~\RNum{7}$'$, respectively. The results of these $t$-tests are reported in Tables \ref{table:r_ttest1}~-~\ref{table:r_ttest2}. Specifically,  Comparison 1$'$, Comparison 2$'$ and Comparison 3$'$ evaluate the differences between the three approaches (Approach 1, Approach2, and LFP-NONM) and UTADIS-NONM in Simulation experiments \RNum{1}$'$~-~\RNum{4}$'$, Comparison 4$'$ and Comparison 5$'$ evaluate the differences between Approach 1, Approach 2 and \cite{Kadzinski20ijar}'s method in Simulation experiments \RNum{5}$'$~-~\RNum{7}$'$.

\begin{table}[htbp]
\centering
\setlength{\abovecaptionskip}{0pt}
\setlength{\belowcaptionskip}{10pt}
\caption{The mean and standard deviation of the accuracy metric for the four approaches with different values of $n$ in Simulation experiment \RNum{1}$'$}
\label{table:r_simu1}
\begin{tabular}{cccccc}
\toprule
$n$ & Approach & $s_j=1$ & $s_j=2$ & $s_j=3$ & $s_j=4$    \\ \midrule
\multirow{3}*{200} & Approach 1 & - & $0.9450\pm0.0047$ & $0.9198\pm0.0097$ & $0.9059\pm0.0081$\\
& Approach 2 & \bm{$0.9667\pm0.0060$} & \bm{$0.9463\pm0.0044$} & \bm{$0.9206\pm0.0097$} & \bm{$0.9071\pm0.0089$} \\
& LFP-NONM & - & $0.9451\pm0.0046$ & $0.9198\pm0.0079$ & $0.9059\pm0.0081$ \\
& UTADIS-NONM  & $0.9591\pm0.0037$ & $0.9297\pm0.0080$ & $0.9037\pm0.0085$ & $0.8767\pm0.0193$ \\

\multirow{3}*{400} & Approach 1 & - & \bm{$0.9731\pm0.0025$} & \bm{$0.9588\pm0.0029$} & $0.9513\pm 0.0048$ \\
& Approach 2 & \bm{$0.9843\pm0.0024$} & $0.9730\pm0.0023$ & $0.9587\pm0.0030$ & \bm{$0.9519\pm0.0051$} \\
& LFP-NONM & - & \bm{$0.9731\pm0.0025$} & \bm{$0.9588\pm 0.0029$} & $0.9513\pm 0.0048$\\
& UTADIS-NONM  & $0.9805\pm0.0029$ & $0.9632\pm0.0042$ & $0.9505\pm0.0048$ & $0.9351\pm0.0068$\\

\multirow{3}*{600} & Approach 1 & - & \bm{$0.9812\pm0.0023$} & $0.9729\pm0.0020$ & $0.9666\pm0.0026$ \\
& Approach 2 & \bm{$0.9893\pm0.0012$} & \bm{$0.9812\pm0.0023$} & \bm{$0.9732\pm0.0019$} & \bm{$0.9670\pm0.0029$} \\
& LFP-NONM & - & \bm{$0.9812\pm 0.0023$} & $0.9729\pm 0.0020$ & $0.9666\pm 0.0026$ \\
& UTADIS-NONM  & $0.9871\pm0.0009$ & $0.9760\pm0.0019$ & $0.9675\pm0.0039$ & $0.9565\pm0.0051$\\
\bottomrule
\end{tabular}
\end{table}

\begin{table}[htbp]
\centering
\setlength{\abovecaptionskip}{0pt}
\setlength{\belowcaptionskip}{10pt}
\caption{The mean and standard deviation of the accuracy metric for the four approaches with different values of $m$ in Simulation experiment \RNum{2}$'$}
\label{table:r_simu2}
\begin{tabular}{cccccc}
\toprule
$m$ & Approach & $s_j=1$ & $s_j=2$ & $s_j=3$ & $s_j=4$    \\ \midrule
\multirow{3}*{6} & Approach 1 & - & $0.9450\pm0.0047$ & $0.9198\pm0.0097$ & $0.9059\pm0.0081$\\
& Approach 2 & \bm{$0.9667\pm0.0060$} & \bm{$0.9463\pm0.0044$} & \bm{$0.9206\pm0.0097$} & \bm{$0.9071\pm0.0089$} \\
& LFP-NONM & - & $0.9451\pm0.0046$ & $0.9198\pm0.0079$ & $0.9059\pm0.0081$ \\
& UTADIS-NONM  & $0.9591\pm0.0037$ & $0.9297\pm0.0080$ & $0.9037\pm0.0085$ & $0.8767\pm0.0193$ \\

\multirow{3}*{7} & Approach 1 & - & $0.9385\pm0.0083$ & $0.9121\pm0.0060$ & \bm{$0.8867\pm0.0089$} \\
& Approach 2 & \bm{$0.9649\pm0.0038$} & \bm{$0.9395\pm0.0083$} & \bm{$0.9124\pm0.0055$} & $0.8861\pm0.0078$ \\
& LFP-NONM & - & $0.9385\pm0.0083$ & $0.9121\pm0.0060$ & \bm{$0.8867\pm0.0089$} \\
& UTADIS-NONM  & $0.9525\pm0.0035$ & $0.9214\pm0.0094$ & $0.8925\pm0.0086$ & $0.8478\pm0.0131$\\

\multirow{3}*{8} & Approach 1 & - & $0.9346\pm0.0057$ & $0.9077\pm0.0074$ & $0.8688\pm0.0098$ \\
& Approach 2 & \bm{$0.9599\pm0.0038$} & \bm{$0.9346\pm0.0055$} & \bm{$0.9094\pm0.0072$} & \bm{$0.8694\pm0.0096$} \\
& LFP-NONM & - & $0.9346\pm 0.0057$ & $0.9077\pm 0.0074$ & $0.8688\pm 0.0098$ \\
& UTADIS-NONM  & $0.9463\pm0.0051$ & $0.9180\pm0.0096$ & $0.8813\pm0.0172$ & $0.8258\pm0.0134$\\
\bottomrule
\end{tabular}
\end{table}

\begin{table}[htbp]
\centering
\setlength{\abovecaptionskip}{0pt}
\setlength{\belowcaptionskip}{10pt}
\caption{The mean and standard deviation of the accuracy metric for the four approaches with different values of $q$ in Simulation experiment \RNum{3}$'$}
\label{table:r_simu3}
\begin{tabular}{cccccc}
\toprule
$q$ & Approach & $s_j=1$ & $s_j=2$ & $s_j=3$ & $s_j=4$    \\ \midrule
\multirow{3}*{2} & Approach 1 & - & $0.9524\pm0.0068$ & $0.9313\pm0.0109$ & $0.9112\pm0.0100$ \\
& Approach 2 & \bm{$0.9764\pm0.0040$} & \bm{$0.9550\pm0.0088$} & \bm{$0.9334\pm0.0086$} & \bm{$0.9118\pm0.0122$} \\
& LFP-NONM & - & $0.9524\pm0.0068$ & $0.9313\pm0.0109$ & $0.9112\pm0.0100$\\
& UTADIS-NONM & $0.9711\pm0.0026$ & $0.9407\pm0.0067$ & $0.9123\pm0.0075$ & $0.8889\pm0.0112$\\

\multirow{3}*{3} & Approach 1 & - & $0.9485\pm0.0052$ & $0.9275\pm0.0067$ & $0.9010\pm0.0097$ \\
& Approach 2 & \bm{$0.9704\pm0.0058$} & \bm{$0.9492\pm0.0044$} & \bm{$0.9286\pm0.0062$} & \bm{$0.9034\pm0.0087$} \\
& LFP-NONM & - & $0.9485\pm 0.0052$ & $0.9275\pm0.0067$ & $0.9010\pm0.0097$\\
& UTADIS-NONM & $0.9637\pm0.0043$ & $0.9327\pm0.0068$ & $0.9077\pm0.0105$ & $0.8757\pm0.0087$\\

\multirow{3}*{4} & Approach 1 & - & $0.9450\pm0.0047$ & $0.9198\pm0.0097$ & $0.9059\pm0.0081$\\
& Approach 2 & \bm{$0.9667\pm0.0060$} & \bm{$0.9463\pm0.0044$} & \bm{$0.9206\pm0.0097$} & \bm{$0.9071\pm0.0089$} \\
& LFP-NONM & - & $0.9451\pm0.0046$ & $0.9198\pm0.0079$ & $0.9059\pm0.0081$ \\
& UTADIS-NONM  & $0.9591\pm0.0037$ & $0.9297\pm0.0080$ & $0.9037\pm0.0085$ & $0.8767\pm0.0193$ \\

\bottomrule
\end{tabular}
\end{table}

\begin{table}[htbp]
\centering
\setlength{\abovecaptionskip}{0pt}
\setlength{\belowcaptionskip}{10pt}
\caption{The mean and standard deviation of the accuracy metric for the four approaches with different values of $r$ in Simulation experiment \RNum{4}$'$}
\label{table:r_simu4}
\begin{tabular}{cccccc}
\toprule
$r$ & Approach & $s_j=1$ & $s_j=2$ & $s_j=3$ & $s_j=4$    \\ \midrule
\multirow{3}*{0.2} & Approach 1 & - & $0.7874\pm0.0072$ & $0.7039\pm0.0079$ & $0.6158\pm0.0125$  \\
& Approach 2 & \bm{$0.8731\pm0.0052$} & \bm{$0.7956\pm0.0051$} & \bm{$0.7189\pm0.0126$} & \bm{$0.6496\pm0.0122$} \\
& LFP-NONM & - & $0.7876\pm0.0072$ & $0.7040\pm0.0078$ & $0.6157\pm0.0126$\\
& UTADIS-NONM & $0.8342\pm0.0064$ & $0.7140\pm0.0119$ & $0.6043\pm0.0173$ & $0.4991\pm0.0146$ \\

\multirow{3}*{0.5} & Approach 1 & - & \bm{$0.9089\pm0.0059$} & $0.8804\pm0.0082$ & $0.8422\pm0.0044$ \\
& Approach 2 &  \bm{$0.9491\pm0.0059$} & $0.9075\pm0.0082$ & \bm{$0.8821\pm0.0080$} & \bm{$0.8480\pm0.0070$}\\
& LFP-NONM & - & \bm{$0.9089\pm0.0059$} & $0.8804\pm0.0082$ & $0.8422\pm0.0043$ \\
& UTADIS-NONM & $0.9341\pm0.0057$ & $0.8866\pm0.0075$ & $0.8499\pm0.0130$ & $0.7901\pm0.0155$\\

\multirow{3}*{0.8} & Approach 1 & - & $0.9450\pm0.0047$ & $0.9198\pm0.0097$ & $0.9059\pm0.0081$\\
& Approach 2 & \bm{$0.9667\pm0.0060$} & \bm{$0.9463\pm0.0044$} & \bm{$0.9206\pm0.0097$} & \bm{$0.9071\pm0.0089$} \\
& LFP-NONM & - & $0.9451\pm0.0046$ & $0.9198\pm0.0079$ & $0.9059\pm0.0081$ \\
& UTADIS-NONM  & $0.9591\pm0.0037$ & $0.9297\pm0.0080$ & $0.9037\pm0.0085$ & $0.8767\pm0.0193$ \\
\bottomrule
\end{tabular}
\end{table}

\begin{table}[htbp]
\centering
\setlength{\abovecaptionskip}{0pt}
\setlength{\belowcaptionskip}{10pt}
\caption{The mean and standard deviation of the accuracy metric for the three approaches with different values of $m$ in Simulation experiment \RNum{5}$'$}
\label{table:r_simu5}
\resizebox{\textwidth}{!}{
\begin{tabular}{cccccc}
\toprule
$m$ & Approach & $s_j=1$ & $s_j=2$ & $s_j=3$ & $s_j=4$    \\ \midrule
\multirow{3}*{6} & Approach 1 & - & $0.7996\pm0.0332$ & $0.7127\pm0.0265$ & $0.6280\pm0.0322$\\
& Approach 2 & \bm{$0.8871\pm0.0202$} & \bm{$0.8047\pm0.0240$} & \bm{$0.7277\pm0.0323$} & \bm{$0.6667\pm0.0466$} \\
& \cite{Kadzinski20ijar}'s method & $0.4678\pm 0.0555$ & $0.4099\pm0.0327$ & $0.3529\pm0.0320$ & $0.3406\pm0.0313$ \\

\multirow{3}*{7} & Approach 1 & - & $0.7764\pm0.0324$ & $0.6772\pm0.0364$ & $0.5737\pm0.0656$ \\
& Approach 2 & \bm{$0.8734\pm0.0312$} & \bm{$0.7837\pm0.0265$} & \bm{$0.6917\pm0.0449$} & \bm{$0.5920\pm0.0457$} \\
& \cite{Kadzinski20ijar}'s method & $0.4624\pm0.0320$ & $0.3929\pm0.0529$ & $0.3466\pm0.0388$ & $0.3426\pm0.0544$ \\

\multirow{3}*{8} & Approach 1 & - & $0.7450\pm0.0442$ & $0.6498\pm0.0302$ & $0.5647\pm0.0328$ \\
& Approach 2 & \bm{$0.8484\pm0.0169$} & \bm{$0.7549\pm0.0315$} & \bm{$0.6583\pm0.0571$} & \bm{$0.5693\pm0.0335$} \\
& \cite{Kadzinski20ijar}'s method & $0.4700\pm0.0371$ & $0.3725\pm 0.0372$ & $0.3661\pm 0.0328$ & $0.3405\pm 0.0111$ \\

\bottomrule
\end{tabular}}
\end{table}

\begin{table}[htbp]
\centering
\setlength{\abovecaptionskip}{0pt}
\setlength{\belowcaptionskip}{10pt}
\caption{The mean and standard deviation of the accuracy metric for the three approaches with different values of $q$ in Simulation experiment \RNum{6}$'$}
\label{table:r_simu6}
\resizebox{\textwidth}{!}{
\begin{tabular}{cccccc}
\toprule
$q$ & Approach & $s_j=1$ & $s_j=2$ & $s_j=3$ & $s_j=4$    \\ \midrule
\multirow{3}*{2} & Approach 1 & - & $0.8599\pm0.0205$ & $0.7762\pm0.0356$ & $0.7524\pm0.0170$\\
& Approach 2 & \bm{$0.9140\pm0.0257$} & \bm{$0.8699\pm0.0190$} & \bm{$0.8124\pm0.0308$} & \bm{$0.7878\pm0.0227$} \\
& \cite{Kadzinski20ijar}'s method & $0.6758\pm 0.0655$ & $0.5878\pm0.0634$ & $0.5859\pm0.0381$ & $0.5744\pm0.0434$ \\

\multirow{3}*{3} & Approach 1 & - & \bm{$0.7858\pm0.0290$} & $0.7465\pm0.0451$ & $0.6650\pm0.0384$ \\
& Approach 2 & \bm{$0.9008\pm0.0246$} & $0.7855\pm0.0200$ & \bm{$0.7585\pm0.0350$} & \bm{$0.6735\pm0.0315$} \\
& \cite{Kadzinski20ijar}'s method & $0.5432\pm0.0699$ & $0.4626\pm0.0354$ & $0.4739\pm0.0429$ & $0.4318\pm0.0346$ \\

\multirow{3}*{4} & Approach 1 & - & $0.7996\pm0.0332$ & $0.7127\pm0.0265$ & $0.6280\pm0.0322$\\
& Approach 2 & \bm{$0.8871\pm0.0202$} & \bm{$0.8047\pm0.0240$} & \bm{$0.7277\pm0.0323$} & \bm{$0.6667\pm0.0466$} \\
& \cite{Kadzinski20ijar}'s method & $0.4678\pm 0.0555$ & $0.4099\pm0.0327$ & $0.3529\pm0.0320$ & $0.3406\pm0.0313$ \\

\bottomrule
\end{tabular}}
\end{table}

\begin{table}[htbp]
\centering
\setlength{\abovecaptionskip}{0pt}
\setlength{\belowcaptionskip}{10pt}
\caption{The mean and standard deviation of the accuracy metric for the three approaches with different values of $r$ in Simulation experiment \RNum{7}$'$}
\label{table:r_simu7}
\resizebox{\textwidth}{!}{
\begin{tabular}{cccccc}
\toprule
$r$ & Approach & $s_j=1$ & $s_j=2$ & $s_j=3$ & $s_j=4$    \\ \midrule
\multirow{3}*{0.2} & Approach 1 & - & $0.3703\pm0.0364$ & $0.3282\pm0.0336$ & $0.3205\pm0.0238$\\
& Approach 2 & \bm{$0.5993\pm0.0237$} & \bm{$0.4352\pm0.0271$} & \bm{$0.3789\pm0.0265$} & \bm{$0.3552\pm0.0249$} \\
& \cite{Kadzinski20ijar}'s method & $0.2926\pm 0.0430$ & $0.2807\pm0.0112$ & $0.2724\pm0.0158$ & $0.2804\pm0.0230$ \\

\multirow{3}*{0.5} & Approach 1 & - & $0.6683\pm0.0296$ & $0.5477\pm0.0340$ & $0.4820\pm0.0371$ \\
& Approach 2 & \bm{$0.8174\pm0.0213$} & \bm{$0.6731\pm0.0282$} & \bm{$0.5791\pm0.0346$} & \bm{$0.5186\pm0.0336$} \\
& \cite{Kadzinski20ijar}'s method & $0.3765\pm0.0396$ & $0.3231\pm0.0279$ & $0.3060\pm0.0176$ & $0.3052\pm0.0217$ \\

\multirow{3}*{0.8} & Approach 1 & - & $0.7996\pm0.0332$ & $0.7127\pm0.0265$ & $0.6280\pm0.0322$\\
& Approach 2 & \bm{$0.8871\pm0.0202$} & \bm{$0.8047\pm0.0240$} & \bm{$0.7277\pm0.0323$} & \bm{$0.6667\pm0.0466$} \\
& \cite{Kadzinski20ijar}'s method & $0.4678\pm 0.0555$ & $0.4099\pm0.0327$ & $0.3529\pm0.0320$ & $0.3406\pm0.0313$ \\

\bottomrule
\end{tabular}}
\end{table}

\begin{small}
\begin{table}[htbp]
\centering
\setlength{\abovecaptionskip}{0pt}
\setlength{\belowcaptionskip}{10pt}
\caption{The results of $t$-tests in Simulation experiments \RNum{1}$'$~-~\RNum{4}$'$}
\label{table:r_ttest1}
\begin{threeparttable}
\begin{tabular}{cccccc}
\toprule
Simulation experiment & Parameter & $s_j$ & Comparison 1$'$ & Comparison 2$'$ & Comparison 3$'$  \\ \midrule
\multirow{12}*{\RNum{1}$'$} & \multirow{4}*{$n=200$} & $s_j=1$ & - & $0.0000^*$ & - \\
&  &$s_j=2$ & $0.0000^*$ & $0.0000^*$ & $0.0000^*$ \\
&  &$s_j=3$ & $0.0006^*$ & $0.0006^*$ & $0.0006^*$ \\
&  &$s_j=4$ & $0.0002^*$ & $0.0001^*$ & $0.0002^*$ \\

& \multirow{4}*{$n=400$} & $s_j=1$ & - & $0.0009^*$ & - \\
&  &$s_j=2$ & $0.0001^*$ & $0.0001^*$ & $0.0001^*$ \\
&  &$s_j=3$ & $0.0002^*$ & $0.0001^*$ & $0.0002^*$ \\
&  &$s_j=4$ & $0.0001^*$ & $0.0001^*$ & $0.0001^*$ \\

& \multirow{4}*{$n=600$} & $s_j=1$ & - & $0.0001^*$ & - \\
&  &$s_j=2$ & $0.0000^*$ & $0.0000^*$ & $0.0000^*$ \\
&  &$s_j=3$ & $0.0002^*$ & $0.0000^*$ & $0.0002^*$ \\
&  &$s_j=4$ & $0.0000^*$ & $0.0000^*$ & $0.0000^*$ \\ \midrule

\multirow{12}*{\RNum{2}$'$} & \multirow{4}*{$m=6$} & $s_j=1$ & - & $0.0000^*$ & - \\
&  &$s_j=2$ & $0.0000^*$ & $0.0000^*$ & $0.0000^*$ \\
&  &$s_j=3$ & $0.0006^*$ & $0.0006^*$ & $0.0006^*$ \\
&  &$s_j=4$ & $0.0002^*$ & $0.0001^*$ & $0.0002^*$ \\

& \multirow{4}*{$m=7$} & $s_j=1$ & - & $0.0000^*$  & - \\
&  &$s_j=2$ & $0.0003^*$ & $0.0003^*$ & $0.0003^*$ \\
&  &$s_j=3$ & $0.0001^*$ & $0.0000^*$ & $0.0001^*$ \\
&  &$s_j=4$ & $0.0000^*$ & $0.0000^*$ & $0.0000^*$ \\

& \multirow{4}*{$m=8$} & $s_j=1$ & - & $0.0000^*$ & - \\
&  &$s_j=2$ & $0.0000^*$ & $0.0000^*$ & $0.0000^*$ \\
&  &$s_j=3$ & $0.0004^*$ & $0.0003^*$ & $0.0004^*$ \\
&  &$s_j=4$ & $0.0000^*$ & $0.0000^*$ & $0.0000^*$ \\ \midrule

\multirow{12}*{\RNum{3}$'$} & \multirow{4}*{$q=2$} & $s_j=1$ & - & $0.0023^*$ & - \\
&  &$s_j=2$ & $0.0000^*$ & $0.0000^*$ & $0.0000^*$ \\
&  &$s_j=3$ & $0.0000^*$ & $0.0000^*$ & $0.0000^*$ \\
&  &$s_j=4$ & $0.0000^*$ & $0.0000^*$ & $0.0000^*$ \\

& \multirow{4}*{$q=3$} & $s_j=1$ & - & $0.0045^*$ & - \\
&  &$s_j=2$ & $0.0000^*$ & $0.0000^*$ & $0.0000^*$ \\
&  &$s_j=3$ & $0.0008^*$ & $0.0004^*$ & $0.0008^*$ \\
&  &$s_j=4$ & $0.0001^*$ & $0.0000^*$ & $0.0001^*$ \\

& \multirow{4}*{$q=4$} & $s_j=1$ & - & $0.0000^*$ & - \\
&  &$s_j=2$ & $0.0000^*$ & $0.0000^*$ & $0.0000^*$ \\
&  &$s_j=3$ & $0.0006^*$ & $0.0006^*$ & $0.0006^*$ \\
&  &$s_j=4$ & $0.0002^*$ & $0.0001^*$ & $0.0002^*$ \\ \midrule

\multirow{12}*{\RNum{4}$'$} & \multirow{4}*{$r=0.2$} & $s_j=1$ & - & $0.0000^*$ & - \\
&  &$s_j=2$ & $0.0000^*$ & $0.0000^*$ & $0.0000^*$ \\
&  &$s_j=3$ & $0.0000^*$ & $0.0000^*$ & $0.0000^*$ \\
&  &$s_j=4$ & $0.0000^*$ & $0.0000^*$ & $0.0000^*$ \\

& \multirow{4}*{$r=0.5$} & $s_j=1$ & - & $0.0000^*$ & - \\
&  &$s_j=2$ & $0.0000^*$ & $0.0000^*$ & $0.0000^*$ \\
&  &$s_j=3$ & $0.0000^*$ & $0.0000^*$ & $0.0000^*$ \\
&  &$s_j=4$ & $0.0000^*$ & $0.0000^*$ & $0.0000^*$ \\

& \multirow{4}*{$r=0.8$} & $s_j=1$ & - & $0.0000^*$ & - \\
&  &$s_j=2$ & $0.0000^*$ & $0.0000^*$ & $0.0000^*$ \\
&  &$s_j=3$ & $0.0006^*$ & $0.0006^*$ & $0.0006^*$ \\
&  &$s_j=4$ & $0.0002^*$ & $0.0001^*$ & $0.0002^*$ \\

\bottomrule
\end{tabular}
\begin{tablenotes}
\footnotesize
\item[*] Significance at the 5 \% level.
\end{tablenotes}
\end{threeparttable}
\end{table}
\end{small}

\begin{small}
\begin{table}[htbp]
\centering
\setlength{\abovecaptionskip}{0pt}
\setlength{\belowcaptionskip}{10pt}
\caption{The results of $t$-tests in Simulation experiments \RNum{5}$'$~-~\RNum{7}$'$}
\label{table:r_ttest2}
\begin{threeparttable}
\begin{tabular}{ccccc}
\toprule
Simulation experiment & Parameter & $s_j$ & Comparison 4$'$ & Comparison 5$'$  \\ \midrule
\multirow{12}*{\RNum{5}$'$} & \multirow{4}*{$m=6$} & $s_j=1$ & - & $0.0000^*$  \\
&  &$s_j=2$ & $0.0000^*$ & $0.0000^*$  \\
&  &$s_j=3$ & $0.0000^*$ & $0.0000^*$  \\
&  &$s_j=4$ & $0.0000^*$ & $0.0000^*$  \\

& \multirow{4}*{$m=7$} & $s_j=1$ & - & $0.0000^*$  \\
&  &$s_j=2$ & $0.0000^*$ & $0.0000^*$  \\
&  &$s_j=3$ & $0.0000^*$ & $0.0000^*$  \\
&  &$s_j=4$ & $0.0000^*$ & $0.0000^*$  \\

& \multirow{4}*{$m=8$} & $s_j=1$ & - & $0.0000^*$  \\
&  &$s_j=2$ & $0.0000^*$ & $0.0000^*$  \\
&  &$s_j=3$ & $0.0000^*$ & $0.0000^*$  \\
&  &$s_j=4$ & $0.0000^*$ & $0.0000^*$  \\  \midrule

\multirow{12}*{\RNum{6}$'$} & \multirow{4}*{$q=2$} & $s_j=1$ & - & $0.0000^*$  \\
&  &$s_j=2$ & $0.0000^*$ & $0.0000^*$  \\
&  &$s_j=3$ & $0.0000^*$ & $0.0000^*$  \\
&  &$s_j=4$ & $0.0000^*$ & $0.0000^*$  \\

& \multirow{4}*{$q=3$} & $s_j=1$ & - & $0.0000^*$   \\
&  &$s_j=2$ & $0.0000^*$ & $0.0000^*$  \\
&  &$s_j=3$ & $0.0000^*$ & $0.0000^*$  \\
&  &$s_j=4$ & $0.0000^*$ & $0.0000^*$  \\

& \multirow{4}*{$q=4$} & $s_j=1$ & - & $0.0000^*$ \\
&  &$s_j=2$ & $0.0000^*$ & $0.0000^*$  \\
&  &$s_j=3$ & $0.0000^*$ & $0.0000^*$  \\
&  &$s_j=4$ & $0.0000^*$ & $0.0000^*$  \\ \midrule

\multirow{12}*{\RNum{7}$'$} & \multirow{4}*{$r=0.2$} & $s_j=1$ & - & $0.0000^*$  \\
&  &$s_j=2$ & $0.0000^*$ & $0.0000^*$  \\
&  &$s_j=3$ & $0.0012^*$ & $0.0000^*$  \\
&  &$s_j=4$ & $0.0005^*$ & $0.0000^*$  \\

& \multirow{4}*{$r=0.5$} & $s_j=1$ & - & $0.0000^*$ \\
&  &$s_j=2$ & $0.0000^*$ & $0.0000^*$  \\
&  &$s_j=3$ & $0.0000^*$ & $0.0000^*$  \\
&  &$s_j=4$ & $0.0000^*$ & $0.0000^*$  \\

& \multirow{4}*{$r=0.8$} & $s_j=1$ & - & $0.0000^*$  \\
&  &$s_j=2$ & $0.0000^*$ & $0.0000^*$  \\
&  &$s_j=3$ & $0.0000^*$ & $0.0000^*$  \\
&  &$s_j=4$ & $0.0000^*$ & $0.0000^*$  \\

\bottomrule
\end{tabular}
\begin{tablenotes}
\footnotesize
\item[*] Significance at the 5 \% level.
\end{tablenotes}
\end{threeparttable}
\end{table}
\end{small}

\end{document}